%% file: neurips_2026.tex
\documentclass{article}

\usepackage[preprint]{neurips_2026}

\usepackage[most]{tcolorbox}
\newtcbox{\hessianbox}{on line,
  boxrule=0.35pt,
  colback=blue!4,
  colframe=blue!35!black,
  arc=1pt,
  left=2pt,
  right=2pt,
  top=1pt,
  bottom=1pt,
  boxsep=1pt}
  \newtcbox{\gradbox}{on line,
  boxrule=0.35pt,
  colback=red!3,
  colframe=red!40!black,
  arc=1pt,
  left=2pt,
  right=2pt,
  top=1pt,
  bottom=1pt,
  boxsep=1pt}
  \newtcbox{\twopointbox}{on line,
  boxrule=0.35pt,
  colback=violet!4,
  colframe=violet!35!black,
  arc=1pt,
  left=2pt,
  right=2pt,
  top=1pt,
  bottom=1pt,
  boxsep=1pt}
\definecolor{myblue}{RGB}{39,89,156}
\definecolor{myred}{RGB}{170,70,67}
\usepackage{pifont}
\usepackage{empheq}
\usepackage[ruled,vlined]{algorithm2e}
\usepackage{tikz,lipsum}
\usepackage{amsmath}
\usepackage{amsbsy}
\usepackage{amsthm}
\usepackage{amssymb}
\usepackage{upgreek}
\usepackage{wrapfig}
\usepackage[utf8]{inputenc} 
\usepackage[T1]{fontenc}    
\usepackage{hyperref}       
\usepackage{url}    
\usepackage{enumitem}        
\usepackage{booktabs}       
\usepackage{amsfonts}       
\usepackage{nicefrac}       
\usepackage{microtype}      
\usepackage{xcolor}   
\usepackage{subcaption}
\def \one {\mathbf{1}}
\newtheorem{definition}{Definition}
\newtheorem{theorem}{Theorem}
\newtheorem{lemma}{Lemma}
\newtheorem{remark}{Remark}

\newcommand{\bs}[1]{\boldsymbol{#1}}

\makeatletter
\newcommand{\thickhline}{%
    \noalign {\ifnum 0=`}\fi \hrule height 1pt
    \futurelet \reserved@a \@xhline
}
\makeatletter
\usepackage{graphicx}
\usepackage{multirow}
\newcounter{projroundfn}
\usepackage{colortbl}
\usepackage[table]{xcolor}
\usepackage{caption}

\definecolor{NavyBlue}{HTML}{000080}
\hypersetup{
    colorlinks=true,
    linkcolor=red,
    citecolor=NavyBlue,
    filecolor=magenta,      
    urlcolor=magenta,
}
\title{Probabilistic  Multi-Source Data Selection via Learnable Categorical Priors: Variance Reduction and Guaranteed Derandomization}

\author{%
  Shizhan Cai \quad
  Qixin Zhang \quad
  Siyuan Ma \\
  \textbf{Siyuan Liang \quad Yongcheng Jing \quad Dacheng Tao} \\[4pt]
  College of Computing and Data Science \\
  Nanyang Technological University \\[2pt]
  \texttt{shizhan001@e.ntu.edu.sg} \quad
  \texttt{qixin.zhang2026@gmail.com} \\
  \texttt{masi0004@e.ntu.edu.sg} \\
  \texttt{\{siyuan.liang,yongcheng.jing,dacheng.tao\}@ntu.edu.sg}
}

\begin{document}
\maketitle
\begin{abstract}
Bilevel data selection is a fundamental task with numerous applications in machine learning and statistics. 
A common strategy for tackling this problem is to relax the original bilevel combinatorial  optimization into a continuous one. However, existing methods, including 0-1 relaxation and probabilistic relaxation, still suffer from several limitations, such as unstable gradient estimation and infeasible rounding. 
To address these issues, this paper proposes a new probabilistic relaxation for the multi-source bilevel data selection problem. 
Our key idea is to introduce multiple learnable categorical priors to parameterize a probability distribution supported only on feasible data subsets, thus naturally yielding a lossless rounding scheme. 
We further show that our proposed relaxation admits an effective two-point gradient estimator, which can simultaneously avoid the inverse-Hessian computation required by existing 0-1 relaxation and the unbounded score term appearing in prior probabilistic relaxation. 
Furthermore, we develop a reducible representation of this estimator, which substantially reduces the number of lower-level retrainings required for gradient computation. 
Built upon this reducible two-point estimator, we then design a novel stochastic gradient algorithm, termed \texttt{ReSGD}, for multi-source bilevel data selection problem, and establish a rigorous approximation guarantee for the subset returned by \texttt{ReSGD} under weak submodularity. 
Finally, extensive experiments  are conducted to validate the effectiveness of our proposed \texttt{ReSGD}.
\end{abstract}
\input{introduction}

\input{Data_Selection}

\input{New_Relaxation}

\input{ReSGD}

\input{experiment}

\input{conclusion}

\clearpage
\bibliographystyle{plain}
\bibliography{references}

\clearpage
\appendix
\onecolumn

\section*{Appendix}

\subsection*{Contents}
\begin{description}
    
    \item[\ref{appendix:related_work}] 
    \textbf{Related Work} 
    \dotfill \pageref{appendix:related_work}
    \item[\ref{appendix:thm:bayesian_v_weak_submodularity}] 
    \textbf{Proof of Theorem~\ref{thm:bayesian_v_weak_submodularity} and the Equivalence of the V-Optimal Criterion} 
    \dotfill \pageref{appendix:thm:bayesian_v_weak_submodularity}

    \begin{description}
        \item[\ref{appendix:thm:bayesian_v_weak_submodularity}.1] 
        Proof of Theorem~\ref{thm:bayesian_v_weak_submodularity}
        \dotfill \pageref{appendix:thm:bayesian_v_weak_submodularity}

        \item[\ref{appendix:thm:bayesian_v_weak_submodularity}.2] 
        Equivalence of the V-Optimal Criterion
        \dotfill \pageref{appendix:thm:bayesian_v_weak_submodularity}
    \end{description}

    \item[\ref{appendix:ge}] 
    \textbf{Gradient Estimators for $\Phi(\bs{\pi})$} 
    \dotfill \pageref{appendix:pge}

    \begin{description}
        \item[\ref{appendix:pge}] 
        Policy Gradient Estimator 
        \dotfill \pageref{appendix:pge}

        \item[\ref{appendix:two-point}] 
        Two-point Gradient Estimator 
        \dotfill \pageref{appendix:two-point}

        \item[\ref{appendix:Reducible2}] 
        Reducible Two-point Representation 
        \dotfill \pageref{appendix:Reducible2}
    \end{description}

    \item[\ref{appendix:approximation}] 
    \textbf{Proof of Theorem~\ref{thm_SGD}} 
    \dotfill \pageref{appendix:approximation}

    \item[\ref{appendix:projection}] 
    \textbf{Projection Calculation} 
    \dotfill \pageref{appendix:projection}

    \item[\ref{appendix:RR}] 
    \textbf{Effective Reservoir Rounding} 
    \dotfill \pageref{appendix:RR}

    \item[\ref{appendix:ds_cifar100}] 
    \textbf{Data Summarization on Cifar-100 Dataset} 
    \dotfill \pageref{appendix:ds_cifar100}    

    \item[\ref{appendix:cs_ex_cl}] 
    \textbf{Coreset Selection for Existing Continual Learning Methods} 
    \dotfill \pageref{appendix:cs_ex_cl}

    \item [\ref{appendix:adversarial}] 
    \textbf{Experiment with Adversarial Initialization} 
    \dotfill \pageref{appendix:adversarial}

    \item[\ref{appendix:data_coruption}] 
    \textbf{Data Corruption} 
    \dotfill \pageref{appendix:data_coruption}

    \item[\ref{appendix:experimental details}] 
    \textbf{Experimental Details} 
    \dotfill \pageref{appendix:experimental details}

    \begin{description}
        \item[\ref{appendix:ds_details}] 
        Data Summarization Experiment Details
        \dotfill \pageref{appendix:ds_details}
        \item[\ref{appendix:cl_details}] 
        Continual Learning Experiment Details
        \dotfill \pageref{appendix:cl_details}
    \end{description}

    \item[\ref{appendix:limitation}] 
    \textbf{Limitations and Broader Impact} 
    \dotfill \pageref{appendix:limitation}   

\end{description}
\hspace{0pt}
\clearpage
\section{Related Work}\label{appendix:related_work}
\paragraph{Bilevel Continuous Optimization.}
Bilevel continuous optimization~(BCO) provides a natural framework for problems whose objective depends on the solution of another optimization problem, and has been widely used in hyperparameter optimization~\citep{franceschi2017forward,franceschi2018bilevel,shaban2019truncated}, meta-learning~\citep{finn2017model,rajeswaran2019meta}, and neural architecture search~\citep{elsken2019neural,liu2018progressive,he2020milenas}. Since its early formulation in mathematical programming~\citep{bracken1973mathematical}, BCO has attracted extensive attention in both optimization and machine learning~\citep{zhang2024introduction}. A major line of work focuses on hypergradient-based methods, which can be broadly categorized into approximate implicit differentiation~(AID) methods~\citep{ghadimi2018approximation,gould2016differentiating,grazzi2020iteration,ji2021bilevel,pedregosa2016hyperparameter} and iterative differentiation~(ITD) methods~\citep{franceschi2017forward,ji2021bilevel,maclaurin2015gradient}. AID methods compute the hypergradient by differentiating the optimality conditions of the lower-level problem, typically involving Hessian inverse-vector products or their approximations, while ITD methods differentiate through the trajectory of the lower-level optimization procedure. Another line of work adopts penalty-based reformulations, which transform the bilevel problem into a constrained or penalized single-level problem by replacing the lower-level solution requirement with its optimality conditions~\citep{chen2025near,lu2024first,sow2022convergence,liu2022bome}. 

\paragraph{Continual Learning.} Continual learning aims to train a model on a sequence of tasks while preserving the knowledge acquired from previously seen data. A central challenge in this setting is \emph{catastrophic forgetting}~\citep{french1999catastrophic,mccloskey1989catastrophic}, where the model quickly loses its performance on old tasks after being updated on new ones. Existing methods address this issue from different perspectives, including regularization-based methods~\citep{chaudhry2018riemannian,kirkpatrick2017overcoming,li2017learning,ritter2018online,yan2022learning,zenke2017continual}, parameter isolation based
methods~\citep{ostapenko2021continual,wang2022foster,yan2021dynamically}, and replay-based methods~\citep{borsos2020coresets,borsos2024data,buzzega2020dark,lopez2017gradient,tong2025coreset,yoon2022online,zhou2022probabilistic}. Among them, replay-based continual learning combats forgetting by maintaining a small memory buffer of representative samples from previous tasks and replaying them during future training. Since the memory budget is typically much smaller than the accumulated training data, the quality of the retained samples becomes critical to the effectiveness of rehearsal. This has motivated a series of coreset selection methods for continual learning~\citep{borsos2020coresets,borsos2024data,zhou2022probabilistic,hao2024bilevel,tong2025coreset}. For example, \citep{borsos2020coresets} extends coreset selection to deep neural networks by formulating it as a bilevel optimization problem. After that, \citep{zhou2022probabilistic} further introduces probabilistic masks to perform global subset selection. Next, \citep{hao2024bilevel} revisits bilevel coreset selection in continual learning and proposes a new 0-1 relaxation formulation. More recently, \citep{tong2025coreset} design a greedy method according to the reducible loss proposed by \citep{mindermann2022prioritized}.

\paragraph{Weakly Submodular Maximization.}
The notion of $\gamma$-weak submodularity was originally introduced by~\citep{das2018approximate}, who showed that the standard greedy algorithm achieves a $(1-e^{-\gamma})$-approximation for maximizing a monotone $\gamma$-weakly submodular function under a \emph{cardinality} constraint. Subsequently, \citep{gatmiry2019nonsubmodular} studied weakly submodular maximization beyond simple cardinality constraints and proved that the standard greedy algorithm achieves a $\frac{0.4\gamma^{2}}{\sqrt{r\gamma}+1}$-approximation for maximizing a monotone $\gamma$-weakly submodular function under a matroid constraint, where $r$ denotes the rank of the matroid. Unfortunately, this approximation guarantee explicitly depends on $r$ and therefore deteriorates as the matroid rank increases. Note that the partition constraints considered in this paper are a special case of matroid constraints. Later, \citep{thiery2022two} introduced the upper submodularity ratio $\beta$ and developed a distorted local search algorithm for monotone $(\gamma,\beta)$-weakly submodular maximization. This distorted local search typically requires repeated swap operations and will incur substantial computational overhead, making it difficult to apply to large-scale selection problems.

\input{appendix_weak-sub}
\input{appendix_gradient_computation}
\input{appendix_approximation}\label{appendix:sgd}
\input{appendix_Projection}

\input{appendix_reservoir}
\input{appendix_ds_cifar100}
\input{appendix_cs_ex_cl}
\input{appendix_adversarial}

\input{appendix_data_corruption}
\input{appendix_experiment_details}

\clearpage
\section{Limitations and Broader Impact}\label{appendix:limitation}

\paragraph{Limitations.} In this work, we introduce a new probabilistic relaxation for the multi-source bilevel data selection problem, which only assigns probability mass to feasible data subsets and thus naturally yields a lossless rounding scheme.
In addition, we show that this new probabilistic relaxation admits an effective two-point gradient estimator, which can simultaneously avoid the inverse-Hessian computation required by existing $0-1$ relaxation and the unbounded score term appearing in prior probabilistic relaxation.
Built upon this reducible two-point estimator, we further design a novel stochastic gradient algorithm, termed \texttt{ReSGD}, for the multi-source bilevel data selection problem, and establish a rigorous approximation guarantee for the subset returned by \texttt{ReSGD} under weak submodularity. A limitation of the current analysis is that \texttt{ReSGD} follows the standard random-output scheme in stochastic optimization: it needs to sample $\tau$ uniformly from $\{0,1,\ldots,T-1\}$ and returns a subset drawn according to the historical prior $\bs{\pi}^{(\tau)}$. 
Although theoretically convenient, in practice, this random-output scheme may be inconvenient, since it requires storing all $T$ historical iterates or recovering the sampled historical iterate after training. Fortunately, recent advances~\citep{attia2026fast,he2022revisit,liu2024revisiting} in last-iterate convergence of stochastic gradient methods provide a promising way to remove this random-sampling step, so that one can directly return the subset sampled from the final prior $\bs{\pi}^{(T)}$ while hopefully preserving a comparable approximation guarantee.
We leave this last-iterate variant of \texttt{ReSGD} as an important direction for future work.

\paragraph{Broader Impact.}This work mainly focuses on theoretically exploring the multi-source bilevel data selection problem.
Therefore, we do not foresee any direct negative societal impact induced by our proposed method.

\end{document}

%% file: introduction.tex
\section{Introduction}
Data selection aims to identify a small subset of training samples such that training solely on this subset can achieve performance comparable to that trained on the full dataset. As modern learning systems continue to grow in scale, this task has become increasingly crucial in a wide range of applications, including efficient training~\citep{chang2026spice,fancombatting,fan2026joint,huang2024robust,mindermann2022prioritized,mirzasoleiman2020coresets}, in-context learning~\citep{kumari2024end,kumari2024bumblebee,zhang2025tagcos} and continual learning~\citep{borsos2020coresets,borsos2024data,hao2024bilevel,tong2025coreset,zhou2022probabilistic}. For instance, in continual learning, an effective way to counter \emph{catastrophic forgetting}~\citep{french1999catastrophic,mccloskey1989catastrophic} is to maintain a small rehearsal memory that stores a subset of previously seen data and then replay them during training on new tasks~\citep{buzzega2020dark,chaudhry2019continual,lopez2017gradient,yoon2022online}. 

From an optimization viewpoint, many of data selection problems~\citep{borsos2020coresets,chenquadratic,killamsetty2021retrieve,lin2024data} can be formulated as \emph{bilevel combinatorial optimization}, where the upper level evaluates the selected subset according to the performance of the model induced by the corresponding lower-level training problem. However, the \emph{discrete} nature of data selection makes the underlying bilevel problem highly challenging. To circumvent this, recent works have sought to bring the original \emph{bilevel combinatorial optimization} into the world of continuous optimization, thereby leveraging the rich toolbox of gradient-based methods~\citep{bertsekas2015convex,lan2020first,nesterov2013introductory}. Existing approaches mainly follow two routes. The first is \emph{0-1 relaxation}~\citep{borsos2020coresets,hao2024bilevel}, which directly relaxes the binary selective mask into a continuous variable in $[0,1]$ and then applies implicit differentiation~\citep{lorraine2020optimizing,pedregosa2016hyperparameter} to solve the resulting bilevel continuous optimization. The second  is \emph{probabilistic relaxation}~\citep{zhou2022probabilistic}, which treats each binary mask as a Bernoulli random variable with learnable parameters and then optimizes these parameters via policy gradient descent~\cite{williams1992simple,sutton1999policy,kakade2001natural}.

Although these two aforementioned continuous-relaxation reformulations have made data selection tasks more tractable, their practical implementation still faces several fundamental limitations:

\vspace{-0.3em}
\begin{itemize}[leftmargin=*, itemsep=0.15em, topsep=0.2em, parsep=0pt, partopsep=0pt] 
	\item[\ding{202}]\textbf{Problematic Gradient Estimations.} In 0-1 relaxation, estimating the upper-level gradient typically requires approximating the \emph{inverse Hessian} of the lower-level objective, which is often \emph{computationally prohibitive}~\citep{borsos2020coresets,hao2024bilevel,zhang2024introduction}. In contrast, probabilistic relaxation avoids such second-order computation via policy-gradient estimators~\cite{sutton1999policy,williams1992simple}, but these estimators are \emph{unbounded} and often suffer from \emph{instability} and \emph{high variance} during training~\citep{greensmith2004variance,wu2018variance};
	\item[\ding{203}] \textbf{Limited to Single-Source Scenarios.} Both 0-1 and probabilistic relaxation primarily focus on subset selection over a \emph{single homogeneous dataset}. However, in many real-world applications, we often encounter \emph{multiple heterogeneous data sources} rather than a single-source dataset. A representative example arises in continual learning, where one often needs to retain a small data subset \emph{separately} from each past task for future replay~\citep{buzzega2020dark,lopez2017gradient,tong2025coreset}. Such settings cannot be trivially decomposed into multiple \emph{independent} single-source data selection problems,  since this treatment will ignore the \emph{cross-source correlations} among datasets and lead to \emph{redundant selections};
	\item[\ding{204}] \textbf{Heuristic and Infeasible Rounding.} In continuous relaxation, \emph{rounding} or \emph{derandomization} refers to the conversion of the obtained continuous solution into a discrete solution~\citep{ageev2004pipage,bu2024tackling}.  However, \emph{satisfactory} rounding schemes are still \emph{lacking for} both 0-1 and probabilistic relaxation. In 0-1 relaxation, current rounding methods are \emph{heuristic}, so there is \emph{no principled way} to quantify how the obtained subset relates to the original continuous solution. In probabilistic relaxation, sampling from the learned Bernoulli distributions may return a subset \emph{outside} the target constraint~\citep{calinescu2011maximizing,chekuri2010submodular,chekuri2014submodular};
	\item[\ding{205}] \textbf{Incomplete Convergence Analysis.} Existing theoretical results of both 0-1 and probabilistic relaxation mainly focus on convergence to a \emph{continuous stationary point}~\citep{zhou2022probabilistic,hao2024bilevel,chenquadratic,borsos2024data}. However, this is \emph{insufficient} for data selection, since our primary interest lies in the quality of the resulting subset rather than the intermediate continuous solution. To date, how to establish \emph{rigorous} theoretical guarantees for the obtained subset remains an \emph{uncharted territory} in bilevel data selection.
\end{itemize}\vspace{-0.3em}
In view of all this, a natural question arises: 
\begin{tcolorbox}
	Can we develop a \emph{provable} continuous-relaxation framework, together with \emph{guaranteed} rounding scheme and \emph{effective} gradient estimator, for \emph{multi-source} data selection problems?
\end{tcolorbox}In this paper, we will provide an affirmative answer to this question by presenting a novel probabilistic relaxation for multi-source data selection problems. Our key idea is to introduce multiple \emph{learnable categorical priors} to parameterize a probability distribution supported only on \emph{feasible} data subsets. Compared with prior relaxations, this constrained construction immediately yields a \emph{lossless} rounding scheme for multi-source data selection problem. Moreover, we derive an \emph{effective} two-point gradient estimator for our proposed probabilistic relaxation, which can simultaneously avoid the costly inverse-Hessian computation in \emph{0-1 relaxation}~\citep{borsos2020coresets,hao2024bilevel} and the unbounded score in prior \emph{probabilistic relaxation}~\citep{zhou2022probabilistic}, thereby achieving a significant \emph{variance-reduction} effect. Based on this two-point estimator, we  further develop an SGD algorithm, termed \texttt{ReSGD}, for the multi-source data selection problem, and establish a \emph{ rigorous} approximation guarantee for the subset returned by \texttt{ReSGD} under \emph{weak submodularity}~\citep{das2018approximate,thiery2022two}. In summary, we make the following contributions:\begin{list}{\labelitemi}{\leftmargin=0.5em \itemindent=-0.0em \itemsep=.1em}
	\vspace{-.05in}
    \item[$\bullet$]We present a novel probabilistic relaxation for \emph{multi-source} data selection problem, which only assigns probability mass to \emph{feasible subsets} and thus naturally induces a \emph{lossless} rounding scheme. 
    \item[$\bullet$]We prove that our proposed probabilistic relaxation admits an effect two-point gradient estimator, which avoids both the inverse-Hessian computation in \emph{0-1 relaxation} and the unbounded score in prior \emph{probabilistic relaxation}, thereby improving stability and reducing variance. Furthermore, we develop a reducible representation of this two-point gradient estimator under a Bayesian perspective, which can further reduce the gradient-computation cost from $\mathcal{O}(h+\sum_{i=1}^{h} n_i)$ retrainings to $\mathcal{O}(h)$. Here, $h$ denotes the number of data sources, and $\sum_{i=1}^{h} n_i$ is the total sample size.

    \item[$\bullet$]Based on the proposed reducible two-point representation, we then design a novel SGD algorithm, termed \texttt{ReSGD}. Furthermore, under $(\gamma,\beta)$-weak submodularity~\citep{thiery2022two}, we show that the subset returned by \texttt{ReSGD} achieves a $\Big(\frac{\gamma^{2}}{\beta+\beta(1-\gamma)+\gamma^{2}}\Big)$-approximation guarantee  with respect to the optimal data subset $\mathcal{S}^{*}$. To the best of our knowledge, this provides \emph{the first} rigorous theoretical guarantee for multi-source bilevel data selection problem.

    \item[$\bullet$]We demonstrate the practical efficacy of our \texttt{ReSGD} on data summarization and continual learning.
\end{list}

%% file: Data_Selection.tex
\section{Bilevel Data Selection: From Single Dataset to Multi-Source Scenarios}
This section first reviews the single-source bilevel data selection problem together with its two representative continuous relaxations, namely, \emph{0-1 relaxation}~\citep{borsos2020coresets,borsos2024data,hao2024bilevel} and \emph{probabilistic relaxation}~\citep{zhou2022probabilistic}. We then move on to the multi-source setting, which is the main focus of this paper.
\subsection{Single-Source Bilevel Data Selection: Background and Relaxations}\label{sec:single_source}Given a fixed dataset $\mathcal{D}\triangleq\{(x_i,y_i)\}_{i=1}^{n}$, the goal of single-source data selection is to identify a subset of at most $K$ samples from $\mathcal{D}$ such that training on this subset can achieve performance comparable to that trained on $\mathcal{D}$. Let $\ell(x_i,y_i;\bs{\theta})$ be the training loss on the $i$-th sample $(x_i,y_i)$, where $\bs{\theta}$ denotes the model parameters. Then, the single-source data selection problem can be formulated as the following \emph{bilevel combinatorial optimization}~\citep{borsos2020coresets,borsos2024data,dumouchelle2024neurbilo,hao2024bilevel,tong2025coreset,zhou2024learning,zhou2022probabilistic}:
\vspace{-0.5em}
\begin{equation}\label{eq:single_source_bilevel}
	\begin{aligned}
		\min_{\mathbf{m}\in\{0,1\}^{n},~\|\mathbf{m}\|_{0}\le K}~&\mathcal{L}(\bs{\theta}^{*}(\mathbf{m}))
		\triangleq \sum_{i=1}^{n}\ell(x_i,y_i;\bs{\theta}^{*}(\mathbf{m}))\\
		\text{s.t.}\quad
		\bs{\theta}^{*}(\mathbf{m}) \in &\arg\min_{\bs{\theta}}~\hat{\mathcal{L}}(\bs{\theta},\mathbf{m})
		\triangleq \sum_{i=1}^{n} m_i \cdot\ell(x_i,y_i;\bs{\theta}),
	\end{aligned}
\end{equation}\vspace{-0.2em}where $\mathbf{m}\in\{0,1\}^{n}$ denotes the selective mask vector with $m_i=1$ indicating that the $i$-th sample is selected and $m_i=0$ otherwise, $\bs{\theta}^{*}(\mathbf{m})$ is the optimal solution to the lower-level problem under the selective mask vector $\mathbf{m}$ and the symbol $\|\cdot\|_{0}$ represents the $\ell_0$ norm.

The bilevel problem in Eq.~\eqref{eq:single_source_bilevel} is generally known to be \textbf{NP}-hard~\citep{feige1998threshold,natarajan1995sparse}, implying that no polynomial-time algorithms can solve it optimally. To tackle this, recent works have relaxed Eq.~\eqref{eq:single_source_bilevel} into a continuous one, so that gradient-based methods can be used to \emph{approximately} solve it. To date, two representative relaxations have been developed for Eq.~\eqref{eq:single_source_bilevel}, namely, \emph{0-1} and \emph{probabilistic relaxation}.

\textbf{0-1 Relaxation.} The first line of work~\citep{borsos2020coresets,borsos2024data,hao2024bilevel}  directly relaxes each mask $m_i$ as a continuous variable over $[0,1]$. In this way, Eq.~\eqref{eq:single_source_bilevel} can be rewritten as a \emph{bilevel continuous optimization}, that is,\vspace{-0.5em}\begin{equation}\label{eq:01_relaxation}\begin{aligned}
\min_{\mathbf{m}\in[0,1]^{n},~\|\mathbf{m}\|_{1}\le K}~&\mathcal{L}(\bs{\theta}^{*}(\mathbf{m}))
		\triangleq \sum_{i=1}^{n}\ell(x_i,y_i;\bs{\theta}^{*}(\mathbf{m})),
    \\\text{s.t.}\quad\bs{\theta}^{*}(\mathbf{m}) \in &\arg\min_{\bs{\theta}}~\hat{\mathcal{L}}(\bs{\theta},\mathbf{m})\triangleq \sum_{i=1}^{n} m_i\cdot\ell(x_i,y_i;\bs{\theta}).
	\end{aligned}
\end{equation}\vspace{-0.1em}Compared with Eq.~\eqref{eq:single_source_bilevel}, Eq.~\eqref{eq:01_relaxation} enables the use of well-established gradient-based methods from \emph{bilevel continuous optimization}~\citep{chen2025near,hong2023two,ji2021bilevel} to search for a high-quality \emph{fractional} mask vector $\mathbf{m}$. In particular, under suitable conditions, the hypergradient w.r.t. each $m_i\in[0,1]$ can be written as:\begin{equation}\label{eq:hypergradient_01}
\frac{\partial\mathcal{L}(\bs{\theta}^{*}(\mathbf{m}))}{\partial m_i}
\triangleq
-\nabla_{\bs{\theta}}\mathcal{L}(\bs{\theta}^{*}(\mathbf{m}))^{\top}
\hessianbox{$\displaystyle
\Big[\frac{\partial^{2}\hat{\mathcal{L}}(\bs{\theta}^{*}(\mathbf{m}),\mathbf{m})}{\partial\bs{\theta}\partial\bs{\theta}^{T}}\Big]^{-1}
$}
\nabla_{\bs{\theta}}\ell(x_i,y_i;\bs{\theta}^{*}(\mathbf{m})).
\end{equation}It is worth noting that Eq.~\eqref{eq:hypergradient_01} requires computing the \emph{inverse Hessian matrix} of the lower-level objective function $\hat{\mathcal{L}}$, which often incurs \emph{substantial time and memory overhead}.

\textbf{Probabilistic Relaxation.} To avoid the costly inverse Hessian computation in Eq.~\eqref{eq:hypergradient_01}, \citep{zhou2022probabilistic} recently proposed a probabilistic relaxation, which models each mask $m_{i}$ as an \emph{independent} Bernoulli random variable, namely, $m_i \sim \mathrm{Bern}(\pi_i)$ where $\pi_i\in[0,1]$ is a learnable parameter. Let $\bs{\pi}\triangleq(\pi_1,\dots,\pi_{n})$ denote the corresponding parameter vector. Under the independence assumption, the induced distribution over binary vector $\mathbf{m}$ is $\text{p}(\mathbf{m}\mid \bs{\pi})
\triangleq
\prod_{i=1}^{n}\pi_i^{m_i}(1-\pi_i)^{1-m_i}$.
Accordingly, Eq.~\eqref{eq:single_source_bilevel} can be relaxed into an expected loss minimization problem with respect to the continuous parameter vector $\bs{\pi}\triangleq(\pi_1,\dots,\pi_{n})$, that is to say,\begin{equation}\label{eq:prob_relaxation}
	\begin{aligned}
		 &\min_{\bs{\pi}\in[0,1]^{n},\|\bs{\pi}\|_{1}\le K}\Phi(\bs{\pi})\triangleq\mathbb{E}_{\text{p}(\bs{m}|\bs{\pi})}\left[\mathcal{L}\left(\bs{\theta}^{*}\left(\mathbf{m}\right)\right)\right]\triangleq\sum_{\mathbf{m}\in\{0,1\}^{n}}\Big(\mathcal{L}(\bs{\theta}^{*}(\mathbf{m}))\cdot\text{p}(\bs{m}|\bs{\pi})\Big)\\&\quad\quad\quad\text{s.t.}\quad\quad\bs{\theta}^{*}(\mathbf{m})\in\mathop{\arg\min}_{\bs{\theta}}\ \hat{\mathcal{L}}(\bs{\theta},\mathbf{m})\triangleq\sum_{i=1}^{n}m_{i}\cdot\ell(x_{i},y_{i};\bs{\theta}).
	\end{aligned}
\end{equation}\vspace{-0.1em}A key advantage of \eqref{eq:prob_relaxation} is that its gradient admits an unbiased policy gradient estimator~\cite{williams1992simple,sutton1999policy},i.e.,\vspace{-0.1em}\begin{equation}\label{eq:policy_gradient}
\begin{aligned}
	&\frac{\partial\Phi(\bs{\pi})}{\partial\pi_{i}}
	\triangleq
	\mathbb{E}_{\text{p}(\mathbf{m}|\bs{\pi})}
	\Big[
	\mathcal{L}(\bs{\theta}^{*}(\mathbf{m}))\cdot
	\nabla_{\pi_{i}}\ln(\text{p}(\mathbf{m}|\bs{\pi}))
	\Big]\triangleq
	\mathbb{E}_{\text{p}(\mathbf{m}|\bs{\pi})}
	\Big[
	\mathcal{L}(\bs{\theta}^{*}(\mathbf{m}))
	\cdot\gradbox{$\displaystyle \frac{m_{i}-\pi_{i}}{\pi_{i}(1-\pi_{i})}$}
	\Big].
\end{aligned}
\end{equation}Eq.\eqref{eq:policy_gradient} implies that  $\mathcal{L}(\bs{\theta}^{*}(\mathbf{m}))\cdot\frac{m_{i}-\pi_{i}}{\pi_{i}(1-\pi_{i})}$ is an unbiased stochastic gradient estimator of the relaxed objective function $\Phi(\bs{\pi})$. In sharp contrast with Eq.\eqref{eq:hypergradient_01}, this estimator only requires evaluating the function value $\mathcal{L}(\bs{\theta}^{*}(\mathbf{m}))$ and is thus \emph{easier to implement}.  However, it is worth noting that when $m_{i}=1,\pi_{i}\rightarrow0$ or $m_{i}=0,\pi_{i}\rightarrow1$, the factor $\frac{m_{i}-\pi_{i}}{\pi_{i}(1-\pi_{i})}$ will become \emph{unbounded}, which can lead to \emph{high variance}, \emph{unstable updates}, and even \emph{divergence} in practice~\citep{greensmith2004variance,wu2018variance}. 

Furthermore, the Bernoulli product distribution $\text{p}(\mathbf{m}\mid \bs{\pi})$ in Eq.~\eqref{eq:prob_relaxation} poses a significant challenge for rounding, since it assigns \emph{non-zero} probability mass to subsets outside the $K$-size cardinality constraint. As a consequence, for an arbitrary parameter vector $\bs{\pi}$, sampling according to $\text{p}(\mathbf{m}\mid \bs{\pi})$ may yield an \emph{infeasible} subset whose cardinality exceeds $K$.

\subsection{Multi-Source Bilevel Data Selection}
We now turn to the multi-source scenarios. Specifically, we first suppose that the full dataset $\mathcal{D}$ consists of $h$ heterogeneous sources, that is to say, $\mathcal{D}\triangleq\cup_{i\in[h]}\mathcal{D}_{i}$, where each source is given by $\mathcal{D}_{i}\triangleq\{(x^{(i)}_{j},y^{(i)}_{j})\mid j\in[n_{i}]\}$ and $\mathcal{D}_{i}\cap\mathcal{D}_{j}=\emptyset,\forall i\neq j$.  Here, different sources may correspond to different tasks, domains or data providers, depending on the applications.

Given a set of predefined budgets $\{K_{1},\dots,K_{h}\}$, the goal of multi-source data selection is to select at most $K_i$ from each source $\mathcal{D}_i$ such that the model trained on the \emph{jointly selected subset} can achieve comparable performance as the model trained on $\mathcal{D}$. Formally, we consider the following problem:\vspace{-0.5em}\begin{equation}\label{eq:multi}
	\begin{aligned}
		&\min_{\mathbf{m}_{i}\in\{0,1\}^{n_{i}},\|\mathbf{m}_{i}\|_{0}\le K_{i},\forall i\in[h]}\mathcal{L}(\bs{\theta}^{*}(\mathbf{m}))\triangleq\sum_{i=1}^{h}\sum_{j=1}^{n_{i}}\ell(x^{(i)}_{j},y^{(i)}_{j};\bs{\theta}^{*}(\mathbf{m}))\\
		&\quad\quad\quad\quad\text{s.t.}\quad\quad\bs{\theta}^{*}(\mathbf{m})\in\mathop{\arg\min}_{\bs{\theta}}\ \hat{\mathcal{L}}(\bs{\theta},\mathbf{m})\triangleq\sum_{i=1}^{h}\sum_{j=1}^{n_{i}}m_{i,j}\cdot\ell(x^{(i)}_{j},y^{(i)}_{j};\bs{\theta}),
	\end{aligned}
\end{equation}\vspace{-0.1em}where the variable $m_{i,j}\in\{0,1\}$ indicates whether the $j$-th sample from the $i$-th source $\mathcal{D}_i$ is selected, $\mathbf{m}_{i}\triangleq(m_{i,1},\dots,m_{i,n_{i}})$ and $\mathbf{m}\triangleq(\mathbf{m}_{1},\dots,\mathbf{m}_{h})$.

Previously, when facing the multi-source data selection problem, many works~\citep{borsos2020coresets,borsos2024data,hao2024bilevel,zhou2022probabilistic} artificially decomposed it into multiple \emph{independent} single-source selection problems and then \emph{repeatedly} applied cardinality-constrained algorithms to each source. In contrast, the upper-level objective in our Eq.~\eqref{eq:multi} evaluates the model trained on the \emph{jointly selected subset across all sources}, so the benefit of selecting a sample from one source depends \emph{crucially} on what has been selected in other sources.

%% file: New_Relaxation.tex
\section{A New Probabilistic Relaxation via Learnable Categorical Priors}
 This section mainly introduces a new probabilistic relaxation for the multi-source data selection problem~\eqref{eq:multi}.  As discussed in Section~\ref{sec:single_source}, a critical drawback of the current \emph{probabilistic relaxation}~\citep{zhou2022probabilistic} is that the induced Bernoulli product distribution will assign \emph{non-zero} probability to subsets \emph{outside} the target constraint, which renders the subsequent rounding operation intrinsically \emph{infeasible}. To overcome this, a natural idea is to design a new sampling distribution whose support is \emph{entirely}  restricted to the \emph{feasible region}. In other words, for any combinatorial constraint $\mathcal{C}$, we seek to construct a parameterized distribution $\mathrm{p}(\mathbf{m}\mid\bs{\pi})$ such that $\mathrm{p}(\mathbf{m}\mid\bs{\pi})\equiv0$ for any $\mathbf{m}\notin \mathcal{C}$. An immediate advantage of such a construction is that every sample drawn from $\mathrm{p}(\mathbf{m}\mid\bs{\pi})$ is automatically \emph{feasible}.

Next, we show how to instantiate this idea for the multi-source data selection problem~\eqref{eq:multi}.   Recall that Eq.~\eqref{eq:multi} requires selecting at most $K_i$ samples from each source $\mathcal{D}_i$. From a high-level perspective, this process closely resembles an election procedure, where each source $\mathcal{D}_i$ is viewed as a local candidate pool and  $K_i$ samples with the highest vote counts are chosen. Motivated by this observation, we introduce a learnable categorical prior $\bs{\pi}_{i}\triangleq(\pi_{i,1},\dots,\pi_{i,n_i})$ for every source $\mathcal{D}_i\triangleq\{(x^{(i)}_{j},y^{(i)}_{j})\mid j\in[n_i]\}$.
Here, each $\pi_{i,j}$ represents the probability of selecting the $j$-th sample from source $\mathcal{D}_i$, while the residual probability $1-\sum_{j=1}^{n_i}\pi_{i,j}$ corresponds to selecting a dummy symbol $\emptyset$, meaning that \emph{no sample} from $\mathcal{D}_i$ is chosen. 

With these categorical priors $\bs{\pi}\triangleq(\bs{\pi}_{1},\dots,\bs{\pi}_{h})$, we then \emph{independently} draw $K_i$ samples from each source $\mathcal{D}_i$ according to the categorical distribution specified by $\bs{\pi}_i$. This operation will yield a random subset $\bigcup_{i=1}^{h}\bigcup_{k=1}^{K_i}d_{i,k}$, where $d_{i,k}$ denotes the $k$-th random draw from source $\mathcal{D}_i$. Note that this subset $\bigcup_{i=1}^{h}\bigcup_{k=1}^{K_i}d_{i,k}$ contains at most $K_i$ samples from each source $\mathcal{D}_i$. Therefore, this sampling rule induces a probability distribution supported only in the \emph{feasible set} of Eq.~\eqref{eq:multi}. Based on this constrained sampling scheme, we naturally can relax Eq.~\eqref{eq:multi} into the following continuous optimization problem over the categorical priors $\bs{\pi}$, that is to say,\vspace{-0.5em}\begin{equation}\label{eq:new_prob_relaxation}
	\begin{aligned}
\min_{\bs{\pi}_{i}\in\Delta_{n_i},\,\forall i\in[h]}\Phi(\bs{\pi})&\triangleq
		\sum_{d_{i,k}\in\mathcal{D}_{i}\cup\{\emptyset\},\forall i,k}
		\Bigg(
		\mathcal{L}\Big(\bs{\theta}^{*}\big(\cup_{i=1}^{h}\cup_{k=1}^{K_i}d_{i,k}\big)\Big)
		\cdot
		\prod_{i=1}^{h}\prod_{k=1}^{K_i}\mathrm{p}(d_{i,k}\mid\bs{\pi}_{i})
		\Bigg),\\
		\text{s.t.}\qquad
		\bs{\theta}^{*}(\mathcal{S})
		&\in
		\arg\min_{\bs{\theta}}~\hat{\mathcal{L}}(\bs{\theta},\mathcal{S})
		\triangleq
		\sum_{(x,y)\in\mathcal{S}}\ell(x,y;\bs{\theta}),\,\,\, \forall\mathcal{S}\subseteq\mathcal{D},
	\end{aligned}
\end{equation}\vspace{-0.15em}where $d_{i,k}$ denotes the $k$-th random draw from source $\mathcal{D}_i$ with $\mathrm{p}((x^{(i)}_{j},y^{(i)}_{j})\mid \bs{\pi}_i)\triangleq\pi_{i,j}$, $\Delta_{n_i}\triangleq\{\bs{p}\in[0,1]^{n_i}:\ \|\bs{p}\|_1\le 1\}$ denotes the sub-probability simplex and $\mathrm{p}(\emptyset\mid \bs{\pi}_i)\triangleq1-\sum_{j=1}^{n_i}\pi_{i,j}$.
\begin{remark}
   Note that, in Eq.~\eqref{eq:new_prob_relaxation}, we switch from the \textbf{binary notation} used in Eq.~\eqref{eq:multi} to an equivalent \textbf{set-based formulation}. Specifically, for any subset $\mathcal{S}\subseteq\mathcal{D}$, let $\bs{\theta}^{*}(\mathcal{S})$ denote the optimal lower-level solution trained on the data subset $\mathcal{S}$. Equivalently, $\bs{\theta}^{*}(\mathcal{S})$ corresponds to $\bs{\theta}^{*}(\mathbf{1}_{\mathcal{S}})$ of Eq.~\eqref{eq:multi}. 
\end{remark}\vspace{-0.3em}Compared with the continuous relaxation in Eq.~\eqref{eq:prob_relaxation}, a major advantage of  Eq.~\eqref{eq:new_prob_relaxation} is that it only assigns probability mass to \emph{feasible} subsets. Consequently, for any given categorical priors $\bs{\pi}\triangleq(\bs{\pi}_1,\dots,\bs{\pi}_h)$, we can directly sample a subset satisfying the multi-source budget constraints. Moreover, this sampling step is \emph{lossless in expectation} w.r.t. the relaxed objective value $\Phi(\bs{\pi})$. Specifically, if the subset $\mathcal{S}$ is sampled via categorical priors $\bs{\pi}$, then $\mathbb{E}_{\mathcal{S}}\left[\mathcal{L}(\bs{\theta}^{*}(\mathcal{S}))\right]=\Phi(\bs{\pi})$.

\subsection{Variance-Reduced Two-point Gradient Estimator and Reducible SGD}
To leverage gradient-based methods for optimizing Eq.~\eqref{eq:new_prob_relaxation}, this subsection focuses on deriving stochastic estimators for $\nabla_{\bs{\pi}}\Phi(\bs{\pi})$. A natural starting point is to follow the same route as in Eq.~\eqref{eq:policy_gradient} and compute a  policy gradient estimator for our new probabilistic relaxation $\Phi(\bs{\pi})$.  However, as shown in Appendix~\ref{appendix:pge}, such an estimator still involves an \emph{unbounded} score factor analogous to that in Eq.~\eqref{eq:policy_gradient}, which can again lead to high variance and unstable training behavior, especially when the categorical priors approach the boundary of the feasible domain.

Fortunately, the special structure of our proposed new probabilistic relaxation $\Phi(\bs{\pi})$ admits a \emph{variance-reduced} two-point gradient estimator, which can avoid the aforementioned \emph{unbounded} score factor. Before going into the details, we first introduce some notations. For any realization of the random draws $\{d_{i,k},\,\forall k\in[K_i],\,\forall i\in[h]\}$, we define  $\mathcal{S}_{-(i,1)}\triangleq \bigcup_{(i',k')\neq(i,1)}d_{i',k'}$ as the subset obtained by excluding the first draw $d_{i,1}$ from source $\mathcal{D}_{i}$.  Equivalently, $\mathcal{S}_{-(i,1)}$ is formed by taking the union of all sampled draws except $d_{i,1}$. With this notation, we can establish the following theorem:\begin{theorem}[Proof in Appendix~\ref{appendix:two-point}]\label{thm:two_point_representation}For any $i\in[h]$ and $j\in[n_i]$, the partial derivative of $\Phi(\bs{\pi})$ with respect to variable $\pi_{i,j}$ satisfies the following equality, namely,\begin{equation}\label{eq:two_point_representation}
\frac{\partial \Phi(\bs{\pi})}{\partial \pi_{i,j}}
\triangleq
K_i\cdot
\mathbb{E}\Bigg[
\underbrace{
\twopointbox{$\displaystyle
\mathcal{L}\Big(\bs{\theta}^{*}\big(\mathcal{S}_{-(i,1)}\cup(x^{(i)}_{j},y^{(i)}_{j})\big)\Big)
-
\mathcal{L}\Big(\bs{\theta}^{*}\big(\mathcal{S}_{-(i,1)}\big)\Big)
$}
}_{\text{\emph{Two-point Representation}}}
\Bigg],
\end{equation}\vspace{-0.2em}where the expectation is taken over all random draws except the first draw $d_{i,1}$ from source $\mathcal{D}_{i}$.
\end{theorem}
Theorem~\ref{thm:two_point_representation} suggests the partial derivative of $\Phi(\bs{\pi})$ w.r.t. $\pi_{i,j}$ can be estimated via the two-point difference $\mathcal{L}(\bs{\theta}^{*}(\mathcal{S}_{-(i,1)}\cup\{(x^{(i)}_{j},y^{(i)}_{j})\}))
-
\mathcal{L}(\bs{\theta}^{*}(\mathcal{S}_{-(i,1)}))$. Compared with the previous Eq.\eqref{eq:hypergradient_01} and Eq.\eqref{eq:policy_gradient}, this two-point representation enjoys three major advantages. First, it only depends on function evaluations of the upper-level objective, thereby avoiding \emph{costly backpropagation} and \emph{the computation of the inverse Hessian matrix} required in Eq.~\eqref{eq:hypergradient_01}. Second, compared to Eq.\eqref{eq:policy_gradient}, Eq.~\eqref{eq:two_point_representation} does not rely on any \emph{unbounded} factor, thereby leading to more \emph{stable} updates. Third, the two subsets $\mathcal{S}_{-(i,1)}\cup(x^{(i)}_{j},y^{(i)}_{j})$ and   $\mathcal{S}_{-(i,1)}$  only differ in \emph{one data point}, so their corresponding function values are typically very close to each other. This strong coupling can significantly \emph{reduce the fluctuation} of the resulting gradient estimator, and thus naturally yields a \emph{variance-reduction} effect~\citep{duchi2015optimal,nesterov2017random,JMLR:v18:16-632}.

It is worth noting that the two-point representation in Eq.~\eqref{eq:two_point_representation} requires evaluating the marginal contribution of every data point $(x_j^{(i)}, y_j^{(i)})$ with respect to its corresponding subset $\mathcal{S}_{-(i,1)}$, which will eventually lead to a total of $\mathcal{O}(h+\sum_{i=1}^{h} n_i)$ lower-level retrainings. Here, $\sum_{i=1}^{h} n_i$ denotes the total number of data points across all sources and is typically \emph{very large} in practice. As a result, directly computing stochastic gradients based on Eq.~\eqref{eq:two_point_representation} is generally \emph{impractical}. To address this issue, we next introduce a \emph{computationally-efficient} reformulation of Eq.~\eqref{eq:two_point_representation}.


\begin{algorithm}[t]
\caption{SGD via Two-point Reducible Representation (\texttt{ReSGD})}
\label{alg:reducible_sgd}
\KwIn{Datasets $\{\mathcal{D}_i\}_{i=1}^{h}$, Budgets $\{K_i\}_{i=1}^{h}$, Learning rate $\eta$, Number of iterations $T$}
\KwOut{Selected subset $\mathcal{S}
\triangleq
\bigcup_{i=1}^{h}\bigcup_{k=1}^{K_i} d_{i,k}$}

Initialize $\bs{\pi}_i^{(0)}\in\Delta_{n_i}$ for all $i\in[h]$ and train a holdout model $\bs{\theta}^{*}(\mathcal{D})$\;
\For{$t=0,1,\ldots,T-1$}{
    Draw $K_i$ samples from source $\mathcal{D}_i$ according to $\bs{\pi}_i^{(t)}$ for all $i\in[h]$\;
    Construct $\mathcal{S}^{(t)}_{-(i,1)}$ by excluding the first draw from source $\mathcal{D}_i$ for all $i\in[h]$\;
    Train $\bs{\theta}^{*}(\mathcal{S}_{-(i,1)})$ for all $i\in[h]$\;
    
    Compute $\widehat{\bs{g}}^{(t)}_i\triangleq(\widehat{g}^{(t)}_{i,1},\ldots,\widehat{g}^{(t)}_{i,n_i})$ by reducible representation Eq.~\eqref{eq:two_point_representation1} $\forall i\in[h], j\in[n_i]$\;
    Update $\bs{\pi}_{i}^{(t+1)}\triangleq
\Pi\rlap{\hspace{0.08em}\raisebox{-0.25ex}{$^{\footnotemark}$}}_{\Delta_{n_i}}\left(
    \bs{\pi}_{i}^{(t)}-\eta\cdot\widehat{\bs{g}}_{i}^{(t)}
    \right),\forall i\in[h]$\; \setcounter{projroundfn}{\value{footnote}}}
Sample an index $\tau$ uniformly at random from $\{0,1,\ldots,T-1\}$\;    
Draw{\color{red}\footnotemark[\value{projroundfn}]}$K_i$ samples $\{d_{i,1},\dots,d_{i,K_i}\}$ from source $\mathcal{D}_i$ via  $\bs{\pi}_i^{(\tau)},\forall i\in[h]$\;
\end{algorithm}\footnotetext[\value{projroundfn}]{We provide the details of \emph{projection operation} and a more effective \emph{reservoir rounding} scheme in Appendix~\ref{appendix:projection} and \ref{appendix:RR}.}


\textbf{Reducible Representation.} We now derive a \emph{tractable} expression for the two-point difference of Eq.~\eqref{eq:two_point_representation} under Bayesian statistical theory~\citep{gelman1995bayesian,MLAPP}. Specifically, we can treat the model parameter $\bs{\theta}$ as a random variable with prior $\mathrm{p}(\bs{\theta})$, and seek to infer its corresponding posterior  $\mathrm{p}(\bs{\theta}|\mathcal{S})$ from the already-seen data subset $\mathcal{S}\subseteq\mathcal{D}$. We further suppose that the model follows  a predictive distribution $\mathrm{p}(y| x;\mathcal{S})\triangleq\int \mathrm{p}(y| x;\bs{\theta})\cdot\mathrm{p}(\bs{\theta}|\mathcal{S})\mathrm{d}\bs{\theta}$. Note that when a point estimate of $\bs{\theta}$ is used, the posterior $\mathrm{p}(\bs{\theta}|\mathcal{S})$ can be viewed as a Dirac measure~\citep{durrett2019probability,evans2025measure}. Under this probabilistic framework, the loss function is commonly specified as the negative log-likelihood of the predictive distribution, namely, $\ell(x,y;\bs{\theta}^{*}(\mathcal{S}))\triangleq-\log\mathrm{p}(y|x;\mathcal{S})$. Then, throughout the Bayes' rule and conditional independence, numerous prior works~\citep{mindermann2022prioritized,tong2025coreset} have shown that the marginal contribution of any data point $(x,y)$ with respect to a subset $\mathcal{S}$ can admit the following reducible form:(\emph{See Appendix~\ref{appendix:Reducible2}.})\vspace{-0.3em}\begin{equation}
 \begin{aligned}
&\mathcal{L}\big(\bs{\theta}^{*}\big(\mathcal{S}\cup(x,y)\big)\big)-\mathcal{L}\big(\bs{\theta}^{*}\big(\mathcal{S}\big)\big)= \log\mathrm{p}(\bs{y}|\bs{x}; \mathcal{S})-\log\mathrm{p}(\bs{y}|\bs{x}; \mathcal{S}\cup (x, y))\\&=\log\mathrm{p}(\bs{y}|\bs{x}; \mathcal{S})-\log\frac{
        \mathrm{p}(y|x; \mathcal{D})\cdot
        \mathrm{p}(\bs{y}|\bs{x};\mathcal{S})
    }{
       \mathrm{p}(y|x;\mathcal{S})
    }=\ell\big(x,y;\bs{\theta}^{*}(\mathcal{D})\big)
-\ell\big(x,y;\bs{\theta}^{*}(\mathcal{S})\big),
\label{eq:reducible_Rep}
\end{aligned}   
\end{equation}\vspace{-0.2em}where the pair $(\bs{x},\bs{y})$ denotes all samples in $\mathcal{D}$ and the second equality comes from Eq.(2) in \citep{mindermann2022prioritized}.

By substituting Eq.~\eqref{eq:reducible_Rep} into Eq.~\eqref{eq:two_point_representation}, we then can obtain the following reducible representation:\vspace{-0.35em}
\begin{equation}\label{eq:two_point_representation1}
\frac{\partial \Phi(\bs{\pi})}{\partial \pi_{i,j}}
\triangleq
K_i\cdot
\mathbb{E}\Bigg[
\underbrace{
\gradbox{$\displaystyle
\ell\big(x^{(i)}_{j},y^{(i)}_{j};\bs{\theta}^{*}(\mathcal{D})\big)
-
\ell\big(x^{(i)}_{j},y^{(i)}_{j};\bs{\theta}^{*}(\mathcal{S}_{-(i,1)})\big)
$}
}_{\text{Reducible Two-point Representation}}
\Bigg].
\end{equation}\vspace{-0.1em}It is worth noting that the reducible two-point representation in Eq.~\eqref{eq:two_point_representation1} only requires \textcolor{red}{$\mathcal{O}(h)$} lower-level trainings to compute the optimal parameters $\bs{\theta}^{}(\mathcal{S}_{-(i,1)}),\forall i\in[h]$ and $\bs{\theta}^{}(\mathcal{D})$. This is \emph{significantly smaller} than the \textcolor{blue}{$\mathcal{O}(h+\sum_{i=1}^{h} n_i)$} retrainings required by Eq.~\eqref{eq:two_point_representation}. Furthermore, we emphasize that the above Bayesian probabilistic framework is well motivated, since several prior works have rigorously shown that many neural networks can be modeled via Bayesian probabilistic model~\citep{wild2023rigorous,liu2019variational,jospin2022hands,gal2016dropout}.

%% file: ReSGD.tex
\section{Stochastic Gradient Descent via Reducible Two-point Representation}
Based on the reducible two-point representation in Eq.~\eqref{eq:two_point_representation1}, we now present a novel stochastic gradient descent algorithm, termed \texttt{ReSGD}, for multi-source bilevel data selection problem. 

As summarized in Algorithm~\ref{alg:reducible_sgd}, \texttt{ReSGD} first initializes a categorical prior $\bs{\pi}_i^{(0)}\in\Delta_{n_i}$ for each source $\mathcal{D}_{i}$ and trains a holdout model on the full dataset $\mathcal{D}$. After that, at every iteration $t$, \texttt{ReSGD} draws $K_i$ samples from each source $\mathcal{D}_i$ according to the current categorical prior $\bs{\pi}_i^{(t)}$. Then, for any $i\in[h]$,  we construct $\mathcal{S}_{-(i,1)}^{(t)}$ by excluding the first draw of $\mathcal{D}_i$, and train the corresponding model $\bs{\theta}^{*}(\mathcal{S}_{-(i,1)}^{(t)})$. With these trained models, \texttt{ReSGD} computes the stochastic gradient vector $\widehat{\bs{g}}_i^{(t)}\triangleq(\widehat{g}_{i,1}^{(t)},\ldots,\widehat{g}_{i,n_i}^{(t)})$ via Eq.~\eqref{eq:two_point_representation1}, where each coordinate is given by $ \widehat{g}_{i,j}^{(t)}
\triangleq
K_i\left[
\ell\big(x_j^{(i)},y_j^{(i)};\bs{\theta}^{*}(\mathcal{D})\big)
-\ell\big(x_j^{(i)},y_j^{(i)};\bs{\theta}^{*}(\mathcal{S}_{-(i,1)}^{(t)})\big)
\right]$. After obtaining $\widehat{\bs{g}}_i^{(t)}$, we update each categorical prior by projected gradient descent. Note that the symbol $\Pi_{\Delta_{n_i}}(\cdot)$ in Algorithm~\ref{alg:reducible_sgd}  denotes the Euclidean projection onto $\Delta_{n_i}$.  Finally, after running the above SGD iterations for $T$ steps, \texttt{ReSGD} draws $K_i$ data samples from each source $\mathcal{D}_i$ and returns their union.
\subsection{Approximation Guarantee under Weak Submodularity}This subsection aims to theoretically verify the effectiveness of our proposed \texttt{ReSGD}. Specifically, we will exploit an intrinsic combinatorial structure of our concerned multi-source data selection problem, known as \emph{weak submodularity}~\citep{chen2018weakly,das2018approximate,elenberg2018restricted,thiery2022two},  to establish an approximation guarantee for \texttt{ReSGD}.

Before delving into the details, we first introduce some commonly used notations. Specifically, for any set function $f:2^{\mathcal{D}}\rightarrow\mathbb{R}$, we denote the marginal gain of adding $\mathcal{B}$ to $\mathcal{A}$ by $f(\mathcal{B}|\mathcal{A})\triangleq f(\mathcal{A}\cup\mathcal{B})-f(\mathcal{A})$. Especially when $\mathcal{B}\triangleq\{e\}$ is a singleton, we simply write
$f(e\mid \mathcal{A})\triangleq f(\{e\}\mid \mathcal{A})$.
In addition,  we say that a set function $f$ is \emph{monotone} if and only if $f(\mathcal{A})\le f(\mathcal{B})$ for any $\mathcal{A}\subseteq \mathcal{B}\subseteq\mathcal{D}$. Since \emph{weak submodularity} is typically used to characterize \emph{utility maximization} problems, in the remainder of this subsection, we consider an equivalent utility function $U(\bs{\theta}^{*}(\mathcal{S}))$ of the upper-level loss $\mathcal{L}(\bs{\theta}^{*}(\mathcal{S}))$ in Eq.~\eqref{eq:multi}. Specifically, we define $U(\bs{\theta}^{*}(\mathcal{S}))\triangleq\mathcal{L}(\bs{\theta}_{\mathrm{init}})-\mathcal{L}(\bs{\theta}^{*}(\mathcal{S}))$, which measures the reduction in the upper-level loss brought by training on $\mathcal{S}$ w.r.t. the initial parameter setup $\bs{\theta}_{\mathrm{init}}$. 

With the above notations, we now formally introduce the concept of \emph{weak submodularity}.\begin{definition}[$(\gamma,\beta)$-Weak Submodularity]
\label{def:two_sided_weak_submodularity}
A utility set function $f:2^{\mathcal{D}}\rightarrow\mathbb{R}$ is called $(\gamma,\beta)$-weakly submodular if there exist $\gamma\in(0,1]$ and $\beta\ge 1$ such that, for any $\mathcal{A}\subseteq\mathcal{B}\subseteq\mathcal{D}$,
\vspace{-0.3em}
\begin{equation}
\begin{alignedat}{2}
    \sum_{e\in\mathcal{B}\setminus\mathcal{A}} f(e\mid\mathcal{A})
    &\ge \gamma\cdot f(\mathcal{B}\mid\mathcal{A}),
    \qquad&
    \sum_{e\in\mathcal{B}\setminus\mathcal{A}} f(e\mid \mathcal{B}\setminus\{e\})
    &\le \beta\cdot f(\mathcal{B}\mid \mathcal{A}).
\end{alignedat}\label{eq:two_sided_weak_submodularity}
\end{equation}\vspace{-0.5em}Here, $\gamma$ is called the lower submodularity ratio, while $\beta$ is called the upper submodularity ratio.
\end{definition}Intuitively, from the perspective of data selection, Eq.~\eqref{eq:two_sided_weak_submodularity} measures how closely the accumulated marginal gains of individual data points approximate the joint utility gain of selecting the entire subset. Furthermore, weak submodularity has been widely observed in a broad range of applications, including variable selection~\citep{das2018approximate,elenberg2018restricted}, video summarization~\citep{chen2018weakly,ijcai2022p666}, neural network pruning~\citep{el2022data,tofigh2025acli}, and sparse optimal transport~\citep{manupriya2024submodular,melis18a}. Next, we will show that such a combinatorial property also broadly appears in bilevel data selection problems.

\textbf{Warm-up: Bayesian Optimal Design~\citep{chaloner1995bayesian,hashemi2019submodular,rainforth2024modern}.} We consider the data selection problem for Bayesian linear regression. Specifically, we assume that the data is generated from a linear model $\bs{y}\triangleq X\bs{\theta}+\bs{\epsilon}$, where $\bs{y}\triangleq[y_1,\ldots,y_n]^{\top}\in\mathbb{R}^{n}$, 
$X\triangleq[x_1,\ldots,x_n]^{\top}\in\mathbb{R}^{n\times p}$, the noise vector $\bs{\epsilon}$ follows a Gaussian distribution $\mathcal{N}(\bs{0},\sigma^{2}\mathrm{I}_n)$ and the model parameter $\bs{\theta}$ is equipped with a Gaussian prior, i.e., $\bs{\theta}\sim\mathcal{N}(\bs{0},\Lambda)$. In this case, the lower-level optimization of bilevel data selection problem can be written as $\bs{\theta}^{*}(\mathcal{S})
\triangleq\mathop{\arg\min}_{\bs{\theta}}
\sum_{(x_i,y_i)\in\mathcal{S}}
(y_i-x_i^{\top}\bs{\theta})^{2}+\sigma^{2}\bs{\theta}^{\top}\Lambda^{-1}\bs{\theta}$, and our goal is to select a small subset $\mathcal{S}\subseteq\mathcal{D}$ such that the resulting posterior estimator has a small prediction error over the full dataset. Equivalently, we seek to minimize this expected upper-level loss function $\mathcal{L}(\bs{\theta}^{*}(\mathcal{S}))
\triangleq
\frac{1}{2n}\mathbb{E}_{\bs{\epsilon},\bs{\theta}}
\left[
\|\bs{y}-X\bs{\theta}^{*}(\mathcal{S})\|_2^{2}
\right]$. Note that \citep{borsos2020coresets,borsos2024data} recently showed that $\mathcal{L}(\bs{\theta}^{*}(\mathcal{S}))$ is closely related to the Bayesian V-optimal criterion
$\mathcal{L}_{\mathrm{V}}(\bs{\theta}^{*}(\mathcal{S}))
\triangleq
\frac{1}{2n}\mathbb{E}_{\bs{\epsilon},\bs{\theta}}[\|X\big(\bs{\theta}-\bs{\theta}^{*}(\mathcal{S})\big)\|_2^{2}]$. Specifically, when the number of data points $n$ is sufficiently large, these two objectives differ only by an additive constant $\sigma^{2}/2$. Therefore, minimizing $\mathcal{L}(\bs{\theta}^{*}(\mathcal{S}))$ can be approximately reduced to minimizing $\mathcal{L}_{\mathrm{V}}(\bs{\theta}^{*}(\mathcal{S}))$. More importantly, we further observe that the utility function induced by the Bayesian V-optimal criterion $\mathcal{L}_{\mathrm{V}}(\bs{\theta}^{*}(\mathcal{S}))$ is $(\gamma,\beta)$-\emph{weakly submodular},  as stated below.\begin{theorem}[Proof in  Appendix~\ref{appendix:thm:bayesian_v_weak_submodularity}]\label{thm:bayesian_v_weak_submodularity}
The utility function $U_{\mathrm{V}}(\mathcal{S})
\triangleq
\mathcal{L}_{\mathrm{V}}(\bs{0})
-\mathcal{L}_{\mathrm{V}}(\bs{\theta}^{*}(\mathcal{S}))$ induced by the Bayesian V-optimal criterion $\mathcal{L}_{\mathrm{V}}(\bs{\theta}^{*}(\mathcal{S}))$ is \textbf{monotone} and $(1/c,c)$-\textbf{weakly submodular}, where $c\triangleq1+\frac{s^{2}}{\sigma^{2}}\cdot\lambda_{max}(\Lambda)$, $s\triangleq\max_{i\in[n]}\|x_{i}\|_{2}$ and $\lambda_{\max}(\Lambda)$ is the largest eigenvalue of the matrix $\Lambda$.
\end{theorem}\vspace{-0.1em}Motivated by the results of Theorem~\ref{thm:bayesian_v_weak_submodularity}, we next establish an approximation guarantee for our proposed \texttt{ReSGD} algorithm under the  \emph{weak submodularity}.\vspace{-0.1em}\begin{theorem}[Proof in Appendix~\ref{appendix:approximation}]
\label{thm_SGD}Let $n\triangleq|\mathcal{D}|$, $K_{max}\triangleq\max_{i\in[h]}K_i$, learning rate $\eta
\triangleq
\sqrt{1/nK_{max}^{2}\cdot T}$ and $\mathcal{S}^{*}$ denotes the optimal solution of Eq.~\eqref{eq:multi}. If the utility function $U(\bs{\theta}^{*}(\mathcal{S}))
\triangleq
\mathcal{L}(\bs{\theta}_{\mathrm{init}})
-
\mathcal{L}(\bs{\theta}^{*}(\mathcal{S}))$ induced by Problem Eq.~\eqref{eq:multi} is $(\gamma,\beta)$-weakly  monotone submodular, then the subset $\mathcal{S}$ returned by Algorithm~\ref{alg:reducible_sgd} satisfies\vspace{-0.1em}\begin{equation*}
    \mathbb{E}\!\left[
U(\bs{\theta}^{*}(\mathcal{S}))
\right]
\ge
\left(\frac{\gamma^{2}}{\beta+\beta(1-\gamma)+\gamma^{2}}\right)
\cdot U(\bs{\theta}^{*}(\mathcal{S}^{*}))
-
\mathcal{O}\Big(
\sqrt{\frac{n\cdot K_{max}^{2}}{T}}
\Big).
\end{equation*}
\end{theorem}\vspace{-0.2em}\begin{remark}
Theorem~\ref{thm_SGD} implies that, when the induced utility function 
$U(\bs{\theta}^{*}(\mathcal{S}))$ is $(\gamma,\beta)$-weakly  monotone submodular, if we set 
$T\triangleq\mathcal{O}(nK_{\max}^{2}/\epsilon^{2})$, the subset yielded by our proposed Algorithm~\ref{alg:reducible_sgd}  can secure a utility value of $(\rho\cdot U(\bs{\theta}^{*}(\mathcal{S}^{*}))-\epsilon)$ where $\rho\triangleq
\frac{\gamma^{2}}{\beta+\beta(1-\gamma)+\gamma^{2}}$.
 Equivalently, in terms of the original loss function, we have $\mathbb{E}\!\left[
\mathcal{L}(\bs{\theta}^{*}(\mathcal{S}))
\right]
\le
\rho\cdot \mathcal{L}(\bs{\theta}^{*}(\mathcal{S}^{*}))
+(1-\rho)\cdot \mathcal{L}(\bs{\theta}_{\mathrm{init}})
+\epsilon$.
\end{remark}

%% file: experiment.tex
\section{Experiments} \label{sec:experiments}
This section empirically validates the effectiveness of our proposed \texttt{ReSGD} algorithm via two different applications: Data summarization~\citep{borsos2024data,mirzasoleiman2013distributed,JMLR:v21:19-467} and Replay-based continual learning~\citep{borsos2020coresets,hao2024bilevel,tong2025coreset,zhou2022probabilistic}.

\begin{figure}[htbp]
    \centering
    \begin{subfigure}[b]{0.24\textwidth}
        \centering
        \includegraphics[width=\textwidth]{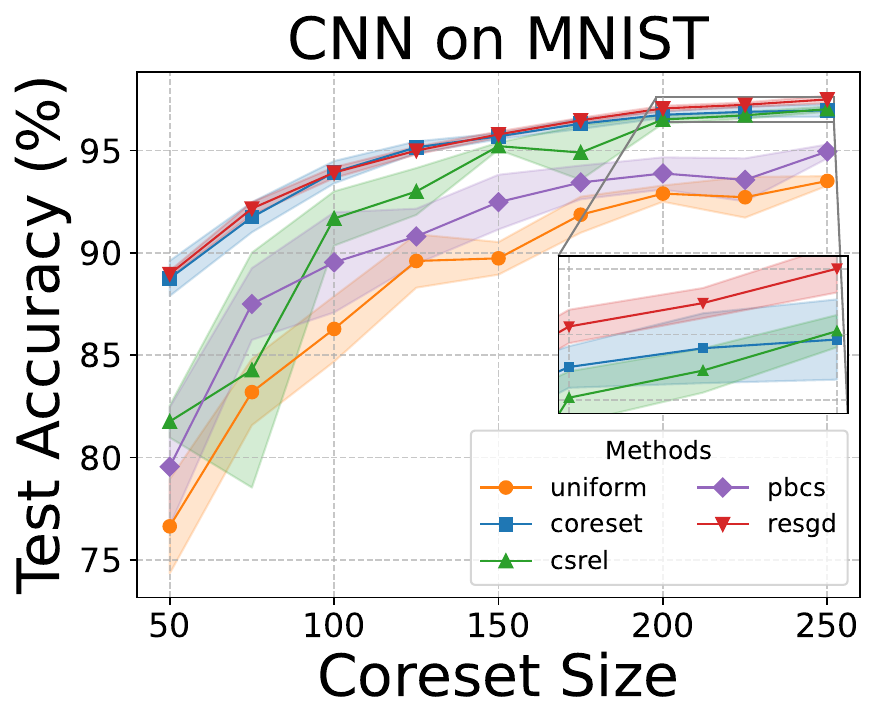}
        \label{fig:sub1}
    \end{subfigure}
    \hfill
    \begin{subfigure}[b]{0.24\textwidth}
        \centering
        \includegraphics[width=\textwidth]{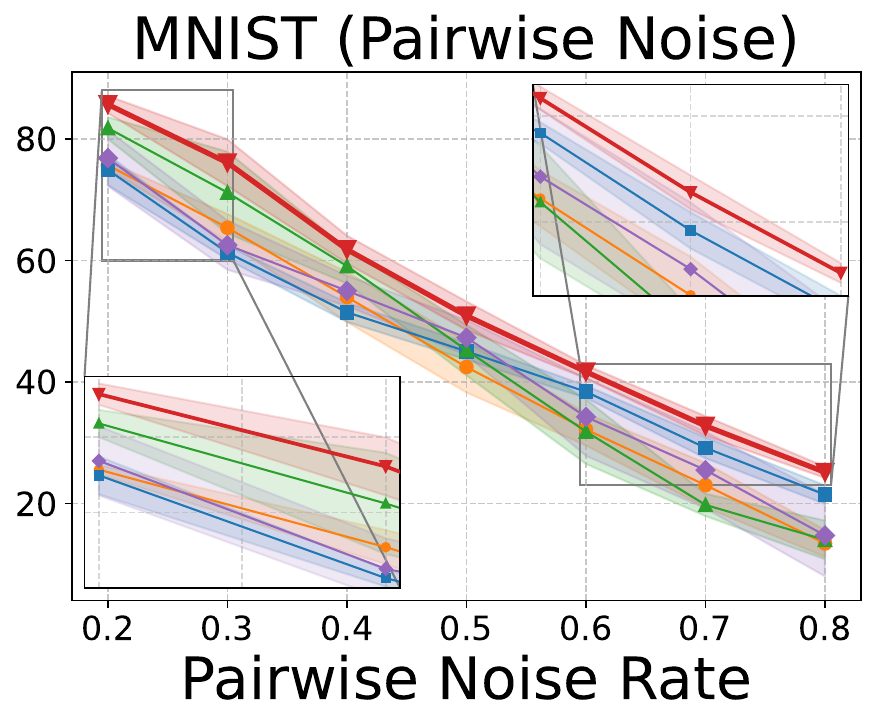}
        \label{fig:sub2}
    \end{subfigure}
    \hfill
    \begin{subfigure}[b]{0.24\textwidth}
        \centering
        \includegraphics[width=\textwidth]{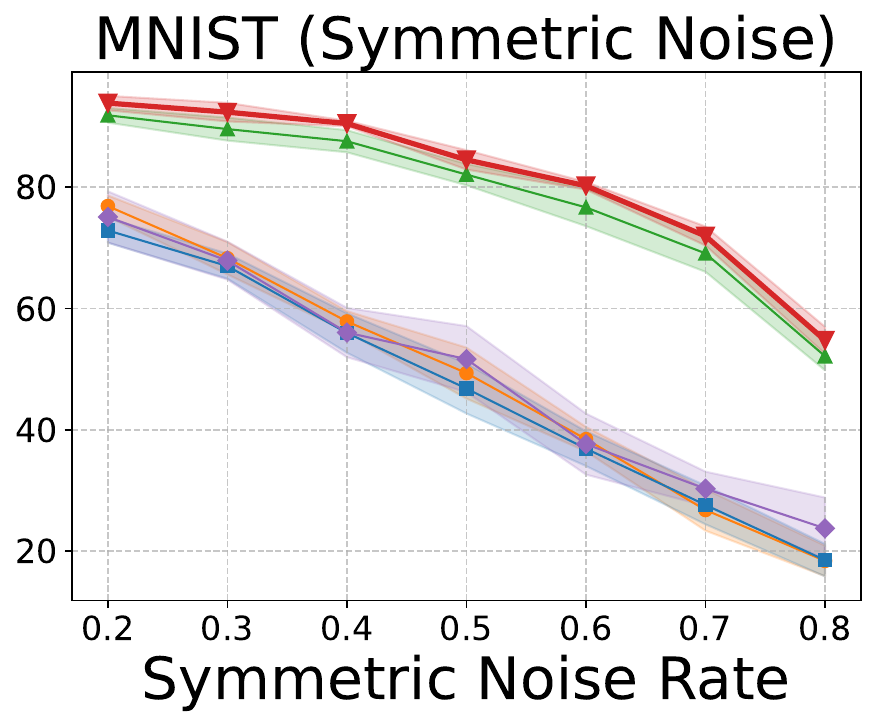}
        \label{fig:sub3}
    \end{subfigure}
    \hfill
    \begin{subfigure}[b]{0.24\textwidth}
        \centering
        \includegraphics[width=\textwidth]{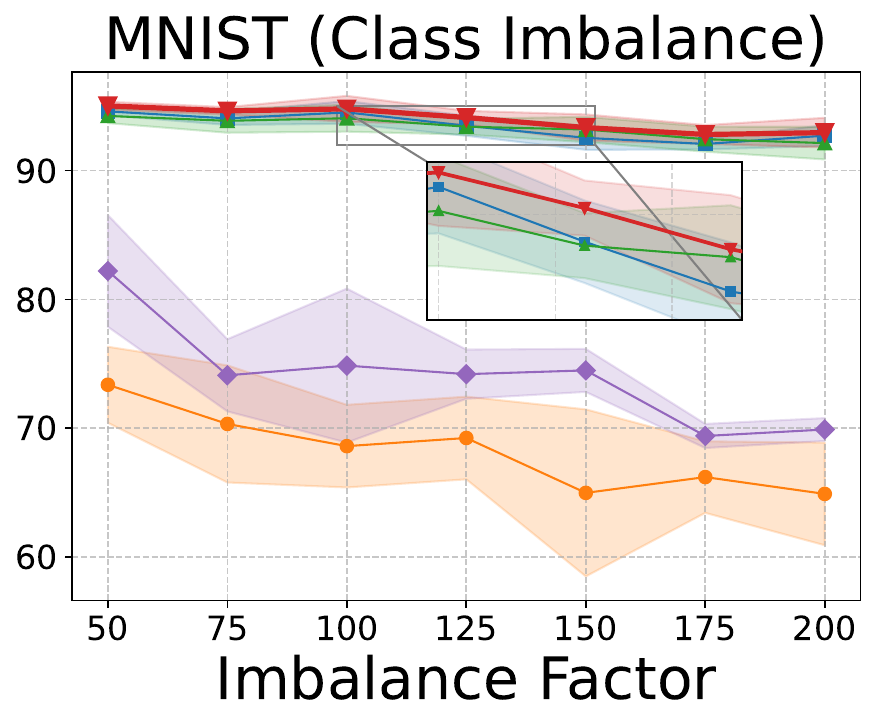}
        \label{fig:sub4}
    \end{subfigure}

    \begin{subfigure}[b]{0.24\textwidth}
        \centering
        \includegraphics[width=\textwidth]{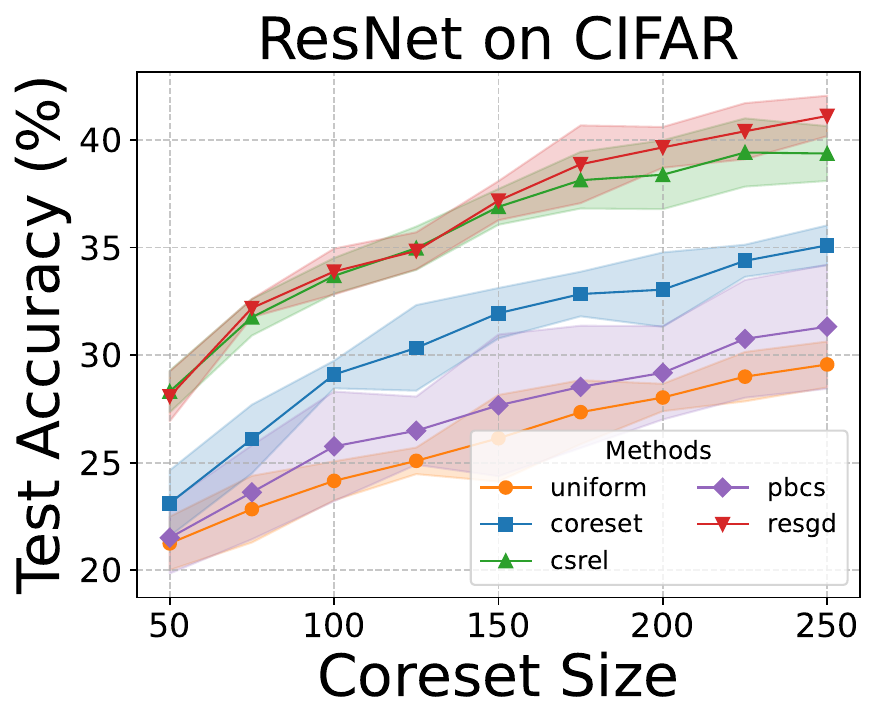}
        \label{fig:sub5}
    \end{subfigure}
    \hfill
    \begin{subfigure}[b]{0.24\textwidth}
        \centering
        \includegraphics[width=\textwidth]{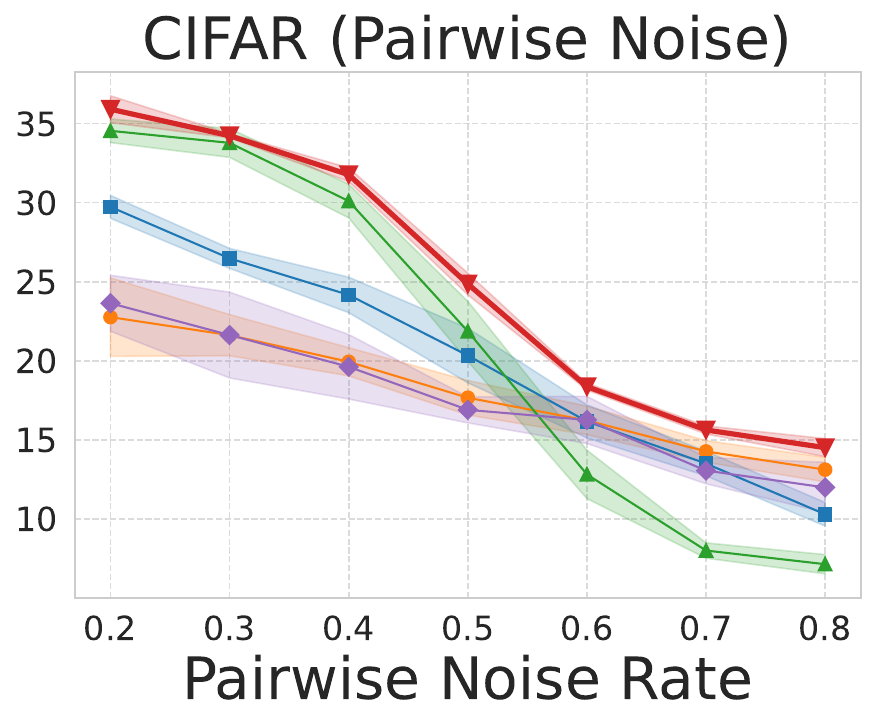}
        \label{fig:sub6}
    \end{subfigure}
    \hfill
    \begin{subfigure}[b]{0.24\textwidth}
        \centering
        \includegraphics[width=\textwidth]{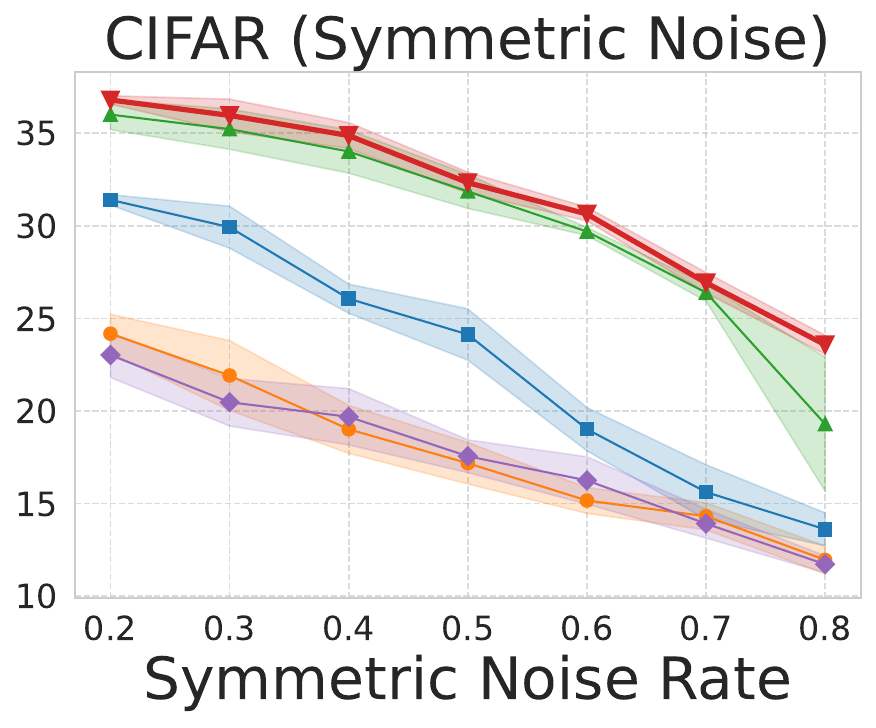}
        \label{fig:sub7}
    \end{subfigure}
    \hfill
    \begin{subfigure}[b]{0.24\textwidth}
        \centering
        \includegraphics[width=\textwidth]{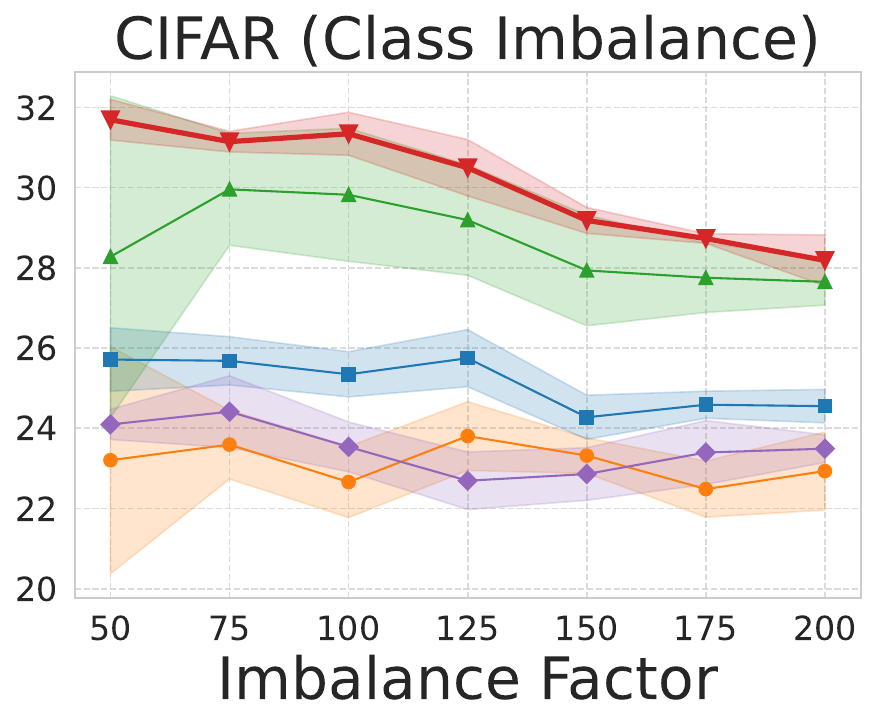}
        \label{fig:sub8}
    \end{subfigure}
   \vspace{-0.5em}\captionsetup{font=small}
    \caption{Performance comparison between \texttt{ReSGD} and other baselines on data summarization, illustrating the superior performance and stability of \texttt{ReSGD} under various data corruption scenarios.}
    \label{fig:total_plots}
\end{figure}\vspace{-0.5em}\textbf{Data Summarization.} We conduct experiments on three benchmarks: MNIST~\cite{deng2012mnist}, CIFAR-10~\cite{krizhevsky2009learning}, and CIFAR-100~\cite{krizhevsky2009learning}. To ensure a fair evaluation, we follow the model architecture and settings of~\cite{borsos2020coresets, tong2025coreset, zhou2022probabilistic}. We compare our \texttt{ReSGD} against a set of competitive baselines, including Greedy Coreset~\cite{borsos2020coresets}, PBCS~\cite{zhou2022probabilistic}, and CSRel~\cite{tong2025coreset}. Additionally, uniform sampling is included to provide a lower-bound performance reference. The experimental details are provided in the Appendix~\ref{appendix:experimental details}. Furthermore, to showcase the effectiveness of \texttt{ReSGD} in more practical and challenging scenarios, we follow the data corruptions via label noise (pairwise and symmetric) and class imbalance in~\cite{zhou2022probabilistic}. The details are deferred to Appendix~\ref{appendix:data_coruption}. As shown in Figure~\ref{fig:total_plots}, \texttt{ReSGD} consistently outperforms all baselines in terms of test accuracy across different coreset sizes, while exhibiting stronger \emph{robustness} to data corruptions, including label noise and class imbalance. Similar trends are observed on CIFAR-100, with detailed results provided in Appendix~\ref{appendix:ds_cifar100}.

\textbf{Continual Learning.} We evaluate \texttt{ReSGD} on Split MNIST~\cite{zenke2017continual}, Perm MNIST~\cite{goodfellow2013empirical}, Split CIFAR-10, and Split CIFAR-100. To ensure a fair evaluation, we follow the model architecture and settings of~\cite{borsos2020coresets, tong2025coreset, zhou2022probabilistic}. We set the memory size to 100 for Split MNIST and Perm MNIST, and 200 for Split CIFAR-10 and Split CIFAR-100. For the Split CIFAR-100 dataset, we split a total of 100 classes into 10 disjoint tasks and use ResNet-18~\cite{he2016deep} as the backbone. Detailed experiment settings and hyperparameters are shown in Appendix~\ref{appendix:experimental details}. The baselines include Uniform, $k$-Means Features~\cite{nguyen2017variational}, $k$-Center Embedding~\cite{sener2017active}, Hardest~\cite{aljundi2019task}, iCaRL~\cite{rebuffi2017icarl}, OCS ~\cite{yoon2022online}, BCSR~\cite{hao2024bilevel}, Greedy Coreset~\cite{borsos2020coresets}, PBCS~\cite{zhou2022probabilistic}, and CSReL-CL-Prv~\cite{tong2025coreset}. The average accuracies in Table~\ref{table:cl-result} show that our method performs on par with Greedy Coreset on the MNIST datasets and consistently outperforms all other baselines on the more complex datasets (Split CIFAR-10 and Split CIFAR-100), highlighting the superiority of our approach. Additionally, the experimental results indicate that considering previous tasks can further enhance CL performance. Further demonstrations of the effectiveness and compatibility of our method with other CL approaches are deferred to Appendix~\ref{appendix:cs_ex_cl}.

\textbf{Further Analysis.} We conduct effectiveness study in Appendix~\ref{appendix:adversarial} and meanwhile provide a detailed sensitivity analysis in Appendix~\ref{appendix:experimental details}.

\begin{table*}[t]
\small
\centering
\caption{{Performance comparison on continual learning benchmarks. We mark the Best results in bold. Our memory construction method performs on par with Greedy Coreset on the MNIST dataset and consistently outperforms other baselines on more complex datasets.}}
\scriptsize{
\resizebox{\linewidth}{!}{
\setlength\tabcolsep{3.pt}
\renewcommand\arraystretch{1.1}
\begin{tabular}{r||cccc}
\hline\thickhline
\rowcolor{gray!20}
\textbf{Method}      & \textbf{PermMNIST} & \textbf{SplitMNIST} & \textbf{SplitCIFAR-10} & \textbf{SplitCIFAR-100} \\
\hline\hline
Uniform              & 91.82 $\pm$ 0.94 & 78.09 $\pm$ 0.80 & 33.66 $\pm$ 1.11 & 7.06 $\pm$ 0.40 \\
\rowcolor{gray!10} 
$k$-Means Features   & 93.22 $\pm$ 1.47 & 78.29 $\pm$ 0.59 & 34.38 $\pm$ 1.26 & 8.05 $\pm$ 0.54 \\
$k$-Center Embedding & 93.69 $\pm$ 0.59 & 78.70 $\pm$ 0.56 & 34.07 $\pm$ 2.14 & 8.97 $\pm$ 0.47 \\
\rowcolor{gray!10} 
Hardest              & 88.17 $\pm$ 1.44 & 77.04 $\pm$ 0.84 & 25.35 $\pm$ 1.48 & 6.43 $\pm$ 0.37 \\
iCaRL                & 93.68 $\pm$ 0.45 & 79.62 $\pm$ 0.50 & 33.26 $\pm$ 0.82 & 8.31 $\pm$ 0.22 \\
\rowcolor{gray!10} 
OCS  & 84.86 $\pm$ 2.69 & 75.07 $\pm$ 0.91 & 34.12 $\pm$ 1.94 & 9.27 $\pm$ 0.45  \\

BCSR & 93.81 $\pm$ 0.91 & 78.30 $\pm$ 0.81 & 35.14 $\pm$ 3.64 & 12.11 $\pm$ 1.24  \\

\rowcolor{gray!10} 
Greedy Coreset       & \textbf{95.67} $\pm$ \textbf{0.21} & 79.31 $\pm$ 0.57 & 34.88 $\pm$ 2.76 & 8.41 $\pm$ 0.25 \\
PBCS                 & 94.22 $\pm$ 0.61 & 76.30 $\pm$ 0.81 & 34.07 $\pm$ 2.14 & 13.02 $\pm$ 0.29 \\
\rowcolor{gray!10} 
CSReL-CL-Prv         & 95.02 $\pm$ 0.23 & 78.64 $\pm$ 0.36 & 36.82 $\pm$ 0.95 & 15.49 $\pm$ 0.84 \\
\hline
\rowcolor[HTML]{D7F6FF}
\texttt{ReSGD}                & \textbf{95.72} $\pm$ \textbf{0.12} &\textbf{80.03} $\pm$ \textbf{0.17} & \textbf{37.68} $\pm$ \textbf{0.13} &\textbf{16.70} $\pm$ \textbf{0.12} \\
\hline
\end{tabular}}}
\vspace{-10pt}
\label{table:cl-result}
\end{table*}

%% file: conclusion.tex
\section{Conclusion}
In this paper, we present a novel probabilistic relaxation for the multi-source data selection problem, which only assigns probability mass to \emph{feasible} subsets and thus naturally induces a \emph{lossless} rounding scheme. 
Most importantly, we show that this relaxation admits an effective two-point gradient estimator, avoiding both the inverse-Hessian computation in \emph{0-1 relaxation} and the unbounded score in prior \emph{probabilistic relaxation}. 
From a Bayesian perspective, we further derive a reducible representation that reduces the gradient-computation cost from $\mathcal{O}(h+\sum_{i=1}^{h} n_i)$ retrainings to $\mathcal{O}(h)$, where $h$ is the number of sources and $\sum_{i=1}^{h} n_i$ is the total sample size. 
Building upon this reducible representation, we develop a stochastic algorithm, termed \texttt{ReSGD} and prove that, under $(\gamma,\beta)$-weak submodularity, its returned subset achieves a $\big(\frac{\gamma^{2}}{\beta+\beta(1-\gamma)+\gamma^{2}}\big)$-approximation guarantee with respect to the optimal subset $\mathcal{S}^{*}$. Finally, extensive experiments on data summarization and continual learning validate the effectiveness of our \texttt{ReSGD}.

%% file: appendix_weak-sub.tex
\section{Proof of Theorem~\ref{thm:bayesian_v_weak_submodularity} and the Equivalence of the V-Optimal Criterion}\label{appendix:thm:bayesian_v_weak_submodularity}
\subsection{Proof of Theorem~\ref{thm:bayesian_v_weak_submodularity}}
In this subsection, we show the detailed verification of Theorem~\ref{thm:bayesian_v_weak_submodularity}. Before that, we recall the bilevel data selection problem for Bayesian linear regression. Specifically, the data is generated from a linear model $\bs{y}\triangleq X\bs{\theta}+\bs{\epsilon}$, $\bs{\epsilon}$, where $\bs{y}\triangleq[y_1,\ldots,y_n]^{\top}\in\mathbb{R}^{n}$, 
$X\triangleq[x_1,\ldots,x_n]^{\top}\in\mathbb{R}^{n\times p}$, the noise vector $\bs{\epsilon}$ follows a Gaussian distribution $\mathcal{N}(\bs{0},\sigma^{2}\mathrm{I}_n)$ and the model parameter $\bs{\theta}$ is equipped with a Gaussian prior, i.e., $\bs{\theta}\sim\mathcal{N}(\bs{0},\Lambda)$.  In this case, the lower-level optimization of bilevel data selection problem can be written as $\bs{\theta}^{*}(\mathcal{S})
\triangleq\mathop{\arg\min}_{\bs{\theta}}
\sum_{(x_i,y_i)\in\mathcal{S}}
(y_i-x_i^{\top}\bs{\theta})^{2}+\sigma^{2}\bs{\theta}^{\top}\Lambda^{-1}\bs{\theta}$, and our goal is to select a small subset $\mathcal{S}\subseteq\mathcal{D}\triangleq\{x_1,\dots,x_n\}$ such that the resulting posterior estimator $\bs{\theta}^{*}(\mathcal{S})$ has a small prediction error over the full dataset. More importantly, from Bayesian statistical theory~\citep{chaloner1995bayesian,hashemi2019submodular,rainforth2024modern}, we can show that $\bs{\theta}^{*}(\mathcal{S})$  has the following closed form:
\begin{equation*}
 \bs{\theta}^{*}(\mathcal{S})\triangleq(X_{\mathcal{S}}^{T}X_{\mathcal{S}}+\sigma^{2}\Lambda^{-1})^{-1}X_{\mathcal{S}}^{T}\bs{y}_{\mathcal{S}},   
\end{equation*} where $X_{\mathcal{S}}$ and $\bs{y}_{\mathcal{S}}$ denotes the matrix $X$ and the vector $\bs{y}$ restricted to rows induced by $\mathcal{S}$, respectively.

Since the proof of Theorem~\ref{thm:bayesian_v_weak_submodularity} will repeatedly involve matrix inverses, we begin by recalling the classical Sherman-Morrisson-Woodbury formula, that is to say,
\begin{lemma}[Sherman-Morrisson-Woodbury formula~\citep{horn2012matrix,hager1989updating}]Let $A$, $U$, $C$ and $V$ be matrices of conformable
sizes. Then, $(A + UCV)^{-1}\triangleq A^{-1}-A^{-1}U(C^{-1}+VA^{-1}U)^{-1}VA^{-1}$.
\end{lemma}

With all above concepts and formula, we next prove Theorem~\ref{thm:bayesian_v_weak_submodularity}.
\begin{proof}
We first simplify the Bayesian V-optimal criterion $\mathcal{L}_{\mathrm{V}}(\bs{\theta}^{*}(\mathcal{S}))
\triangleq
\frac{1}{2n}\mathbb{E}_{\bs{\epsilon},\bs{\theta}}[\|X\big(\bs{\theta}-\bs{\theta}^{*}(\mathcal{S})\big)\|_2^{2}]$, namely,
\begin{equation*}
 \begin{aligned}
 &\mathcal{L}_{\mathrm{V}}(\bs{\theta}^{*}(\mathcal{S}))\triangleq
\frac{1}{2n}\mathbb{E}_{\bs{\epsilon},\bs{\theta}}[\|X\big(\bs{\theta}^{*}(\mathcal{S})-\bs{\theta}\big)\|_2^{2}]\\
&=\frac{1}{2n}\mathbb{E}_{\bs{\epsilon},\bs{\theta}}[\|X\big((X_{\mathcal{S}}^{T}X_{\mathcal{S}}+\sigma^{2}\Lambda^{-1})^{-1}X_{\mathcal{S}}^{T}\bs{y}_{\mathcal{S}}-\bs{\theta}\big)\|_2^{2}]\\
&=\frac{1}{2n}\mathbb{E}_{\bs{\epsilon},\bs{\theta}}[\|X\big((X_{\mathcal{S}}^{T}X_{\mathcal{S}}+\sigma^{2}\Lambda^{-1})^{-1}X_{\mathcal{S}}^{T}\big(X_{\mathcal{S}}\bs{\theta}+\bs{\epsilon}_{\mathcal{S}}\big)-\bs{\theta}\big)\|_2^{2}]\\
&=\frac{1}{2n}\mathbb{E}_{\bs{\epsilon},\bs{\theta}}[\|X\big((X_{\mathcal{S}}^{T}X_{\mathcal{S}}+\sigma^{2}\Lambda^{-1})^{-1}X_{\mathcal{S}}^{T}X_{\mathcal{S}}\bs{\theta}-\bs{\theta}+X\big((X_{\mathcal{S}}^{T}X_{\mathcal{S}}+\sigma^{2}\Lambda^{-1})^{-1}X_{\mathcal{S}}^{T}\bs{\epsilon}_{\mathcal{S}}\big)\|_2^{2}]\\
&=\frac{1}{2n}\mathbb{E}_{\bs{\epsilon},\bs{\theta}}[\|X\big((X_{\mathcal{S}}^{T}X_{\mathcal{S}}+\sigma^{2}\Lambda^{-1})^{-1}X_{\mathcal{S}}^{T}\bs{\epsilon}_{\mathcal{S}}-\sigma^{2}(X_{\mathcal{S}}^{T}X_{\mathcal{S}}+\sigma^{2}\Lambda^{-1})^{-1}\Lambda^{-1}\bs{\theta}\big)\|_2^{2}]\\
&=\frac{1}{2n}\mathbb{E}_{\bs{\epsilon},\bs{\theta}}[\|X(X_{\mathcal{S}}^{T}X_{\mathcal{S}}+\sigma^{2}\Lambda^{-1})^{-1}X_{\mathcal{S}}^{T}\bs{\epsilon}_{\mathcal{S}}\|_2^{2}]+\frac{1}{2n}\mathbb{E}_{\bs{\epsilon},\bs{\theta}}[\|\sigma^{2}\cdot X(X_{\mathcal{S}}^{T}X_{\mathcal{S}}+\sigma^{2}\Lambda^{-1})^{-1}\Lambda^{-1}\bs{\theta}\big)\|_{2}^{2}]\\
&=\frac{1}{2n}\mathrm{Tr}[X(X_{\mathcal{S}}^{T}X_{\mathcal{S}}+\sigma^{2}\Lambda^{-1})^{-1}X_{\mathcal{S}}^{T}\mathbb{E}_{\bs{\epsilon}}\big[\bs{\epsilon}_{\mathcal{S}}\bs{\epsilon}_{\mathcal{S}}^{T}\big]X_{\mathcal{S}}(X_{\mathcal{S}}^{T}X_{\mathcal{S}}+\sigma^{2}\Lambda^{-1})^{-1}X^{T}]\\
&+\frac{1}{2n}\mathrm{Tr}[\sigma^{2}\cdot X(X_{\mathcal{S}}^{T}X_{\mathcal{S}}+\sigma^{2}\Lambda^{-1})^{-1}\Lambda^{-1}\mathbb{E}_{\bs{\theta}}\big[\bs{\theta}\bs{\theta}^{T}\big]\Lambda^{-1}(X_{\mathcal{S}}^{T}X_{\mathcal{S}}+\sigma^{2}\Lambda^{-1})^{-1}X^{T}\cdot\sigma^{2}]\\&=\frac{\sigma^{2}}{2n}\mathrm{Tr}[X_{\mathcal{S}}(X_{\mathcal{S}}^{T}X_{\mathcal{S}}+\sigma^{2}\Lambda^{-1})^{-1}X^{T}X(X_{\mathcal{S}}^{T}X_{\mathcal{S}}+\sigma^{2}\Lambda^{-1})^{-1}X_{\mathcal{S}}^{T}]\\&+\frac{\sigma^{2}}{2n}\mathrm{Tr}[ \sigma^{2}\cdot X(X_{\mathcal{S}}^{T}X_{\mathcal{S}}+\sigma^{2}\Lambda^{-1})^{-1}\Lambda^{-1}(X_{\mathcal{S}}^{T}X_{\mathcal{S}}+\sigma^{2}\Lambda^{-1})^{-1}X^{T}]\\&=\frac{\sigma^{2}}{2n}\mathrm{Tr}[(X_{\mathcal{S}}^{T}X_{\mathcal{S}}+\sigma^{2}\Lambda^{-1})^{-1}X^{T}X(X_{\mathcal{S}}^{T}X_{\mathcal{S}}+\sigma^{2}\Lambda^{-1})^{-1}X_{\mathcal{S}}^{T}X_{\mathcal{S}}]\\&+\frac{\sigma^{2}}{2n}\mathrm{Tr}[(X_{\mathcal{S}}^{T}X_{\mathcal{S}}+\sigma^{2}\Lambda^{-1})^{-1}X^{T}X(X_{\mathcal{S}}^{T}X_{\mathcal{S}}+\sigma^{2}\Lambda^{-1})^{-1}\left(\sigma^{2}\cdot\Lambda^{-1}\right)]\\
&=\frac{\sigma^{2}}{2n}\mathrm{Tr}[(X_{\mathcal{S}}^{T}X_{\mathcal{S}}+\sigma^{2}\Lambda^{-1})^{-1}X^{T}X]=\frac{\sigma^{2}}{2n}\mathrm{Tr}[X(X_{\mathcal{S}}^{T}X_{\mathcal{S}}+\sigma^{2}\Lambda^{-1})^{-1}X^{T}],
 \end{aligned}   
\end{equation*} where $\mathbb{E}_{\bs{\epsilon}}\big[\bs{\epsilon}_{\mathcal{S}}\bs{\epsilon}_{\mathcal{S}}^{T}\big]=\sigma^{2}\mathrm{I}_{|\mathcal{S}|}$ and $\mathbb{E}_{\bs{\theta}}\big[\bs{\theta}\bs{\theta}^{T}\big]=\Lambda$.

With this equality, we can rewrite the utility function $U_{\mathrm{V}}(\mathcal{S})$ as 
\begin{equation*}
    \begin{aligned}
U_{\mathrm{V}}(\mathcal{S})&\triangleq
\mathcal{L}_{\mathrm{V}}(\bs{0})
-\mathcal{L}_{\mathrm{V}}(\bs{\theta}^{*}(\mathcal{S}))\\&=\frac{\sigma^{2}}{2n}\Big(\mathrm{Tr}[\sigma^{-2}\cdot X\Lambda X^{T}]-\mathrm{Tr}[X(X_{\mathcal{S}}^{T}X_{\mathcal{S}}+\sigma^{2}\Lambda^{-1})^{-1}X^{T}]\Big).
    \end{aligned}
\end{equation*} 

For simplicity, in the remainder of this proof, for any subset $\mathcal{S}$, we set $\mathcal{M}_{\mathcal{S}}\triangleq X_{\mathcal{S}}^{T}X_{\mathcal{S}}+\sigma^{2}\Lambda^{-1}$.

Next, we show this utility function $U_{\mathrm{V}}(\mathcal{S})$ satisfies the $(\gamma,\beta)$-weak submodularity.

At first, for any subset $\mathcal{A}\subseteq\mathcal{B}$ and $\mathcal{R}\triangleq\mathcal{B}\setminus\mathcal{A}$, we have that
\begin{equation*}
    \begin{aligned}
&U_{\mathrm{V}}(\mathcal{B})-U_{\mathrm{V}}(\mathcal{A})\\&= \frac{\sigma^{2}}{2n}\left(\mathrm{Tr}[X\mathcal{M}^{-1}_{\mathcal{A}}]X^{T}]-\mathrm{Tr}[X\mathcal{M}^{-1}_{\mathcal{B}}]X^{T}]\right)\\
&=\frac{\sigma^{2}}{2n}\left(\mathrm{Tr}[X\mathcal{M}^{-1}_{\mathcal{A}}]X^{T}]-\mathrm{Tr}[X\left(\mathcal{M}_{\mathcal{A}}+X_{\mathcal{R}}^{T}X_{\mathcal{R}}\right)^{-1}X^{T}]\right)\\
&=\frac{\sigma^{2}}{2n}\left(\mathrm{Tr}[X\mathcal{M}^{-1}_{\mathcal{A}}]X^{T}]-\mathrm{Tr}[X\left(\mathcal{M}_{\mathcal{A}}^{-1}-\mathcal{M}_{\mathcal{A}}^{-1}X_{\mathcal{R}}^{T}\left(I+X_{\mathcal{R}}\mathcal{M}_{\mathcal{A}}^{-1}X_{\mathcal{R}}^{T}\right)^{-1}X_{\mathcal{R}}\mathcal{M}_{\mathcal{A}}^{-1}\right)X^{T}]\right)\\
&=\frac{\sigma^{2}}{2n}\left(\mathrm{Tr}[X\mathcal{M}_{\mathcal{A}}^{-1}X_{\mathcal{R}}^{T}\left(I+X_{\mathcal{R}}\mathcal{M}_{\mathcal{A}}^{-1}X_{\mathcal{R}}^{T}\right)^{-1}X_{\mathcal{R}}\mathcal{M}_{\mathcal{A}}^{-1}X^{T}]\right)\\
&=\frac{\sigma^{2}}{2n}\left(\mathrm{Tr}[\left(I+X_{\mathcal{R}}\mathcal{M}_{\mathcal{A}}^{-1}X_{\mathcal{R}}^{T}\right)^{-1}X_{\mathcal{R}}\mathcal{M}_{\mathcal{A}}^{-1}X^{T}X\mathcal{M}_{\mathcal{A}}^{-1}X_{\mathcal{R}}^{T}]\right)
    \end{aligned}
\end{equation*}

Note that $\lambda_{min}(I+X_{\mathcal{R}}\mathcal{M}_{\mathcal{A}}^{-1}X_{\mathcal{R}}^{T})\ge1$, so we have that
\begin{equation*}
U_{\mathrm{V}}(\mathcal{B})-U_{\mathrm{V}}(\mathcal{A})\le\frac{\sigma^{2}}{2n}\mathrm{Tr}[X_{\mathcal{R}}\mathcal{M}_{\mathcal{A}}^{-1}X^{T}X\mathcal{M}_{\mathcal{A}}^{-1}X_{\mathcal{R}}^{T}].
\end{equation*}

Furthermore, we also have that (Note that $x_e\in\mathbb{R}^{p\times1}$)
\begin{equation*}
    \begin{aligned}
   U_{\mathrm{V}}(\mathcal{A}\cup\{e\})-U_{\mathrm{V}}(\mathcal{A})  &=\frac{\sigma^{2}}{2n}\left(\mathrm{Tr}[\left(1+x_{e}^{T}\mathcal{M}_{\mathcal{A}}^{-1}x_{e}\right)^{-1}x^{T}_{e}\mathcal{M}_{\mathcal{A}}^{-1}X^{T}X\mathcal{M}_{\mathcal{A}}^{-1}x_{e}]\right)\\
   &=\frac{\sigma^{2}}{2n}\cdot\frac{\mathrm{Tr}[x_{e}^{T}\mathcal{M}_{\mathcal{A}}^{-1}X^{T}X\mathcal{M}_{\mathcal{A}}^{-1}x_{e}]}{1+x_{e}^{T}\mathcal{M}_{\mathcal{A}}^{-1}x_{e}}
   \\&\ge\frac{\sigma^{2}}{2n}\cdot\frac{\mathrm{Tr}[x_{e}^{T}\mathcal{M}_{\mathcal{A}}^{-1}X^{T}X\mathcal{M}_{\mathcal{A}}^{-1}x_{e}]}{1+\frac{s^{2}}{\sigma^{2}}\cdot\lambda_{max}(\Lambda)},
    \end{aligned} 
\end{equation*} where the final inequality follows from $\mathcal{M}_{\mathcal{A}}=X_{\mathcal{S}}^{T}X_{\mathcal{S}}+\sigma^{2}\Lambda^{-1}$ such that $x_{e}^{T}\mathcal{M}_{\mathcal{A}}^{-1}x_{e}\le x_{e}^{T}\left(\sigma^{2}\Lambda^{-1}\right)^{-1}x_{e}\le\frac{x_{e}^{T}\Lambda x_{e}}{\sigma^{2}}\le\frac{s^{2}}{\sigma}\lambda_{max}(\Lambda)$.

Here, $s\triangleq\max_{i\in[n]}\|x_{i}\|_{2}$ and $\lambda_{\max}(\Lambda)$ is the largest eigenvalue of the matrix $\Lambda$. 

As a result, we have that
\begin{equation*}
    \begin{aligned}
\sum_{e\in\mathcal{B}\setminus\mathcal{A}}\left(U_{\mathrm{V}}(\mathcal{A}\cup\{e\})-U_{\mathrm{V}}(\mathcal{A})\right)&\ge\frac{\sigma^{2}}{2n}\cdot\frac{\sum_{e\in\mathcal{B}\setminus\mathcal{A}}\mathrm{Tr}[x_{e}^{T}\mathcal{M}_{\mathcal{A}}^{-1}X^{T}X\mathcal{M}_{\mathcal{A}}^{-1}x_{e}]}{1+\frac{s^{2}}{\sigma^{2}}\cdot\lambda_{max}(\Lambda)}\\&=\frac{\sigma^{2}}{2n}\cdot\frac{\mathrm{Tr}[X_{\mathcal{R}}\mathcal{M}_{\mathcal{A}}^{-1}X^{T}X\mathcal{M}_{\mathcal{A}}^{-1}X_{\mathcal{R}}^{T}]}{1+\frac{s^{2}}{\sigma^{2}}\cdot\lambda_{max}(\Lambda)}\ge\frac{U_{\mathrm{V}}(\mathcal{B})-U_{\mathrm{V}}(\mathcal{A})}{1+\frac{s^{2}}{\sigma^{2}}\cdot\lambda_{max}(\Lambda)},
    \end{aligned}
\end{equation*} where $\mathcal{R}\triangleq\mathcal{B}\setminus\mathcal{A}$.

Secondly, for any subset $\mathcal{A}\subseteq\mathcal{B}$ and $\mathcal{R}\triangleq\mathcal{B}\setminus\mathcal{A}$, we have that
\begin{equation*}
    \begin{aligned}
&U_{\mathrm{V}}(\mathcal{B})-U_{\mathrm{V}}(\mathcal{A})\\&= \frac{\sigma^{2}}{2n}\left(\mathrm{Tr}[X\mathcal{M}^{-1}_{\mathcal{A}}]X^{T}]-\mathrm{Tr}[X\mathcal{M}^{-1}_{\mathcal{B}}]X^{T}]\right)\\
&=\frac{\sigma^{2}}{2n}\left(\mathrm{Tr}[X\left(\mathcal{M}_{\mathcal{B}}-X_{\mathcal{R}}^{T}X_{\mathcal{R}}\right)^{-1}X^{T}]-\mathrm{Tr}[X\mathcal{M}^{-1}_{\mathcal{B}}]X^{T}]\right)\\
&=\frac{\sigma^{2}}{2n}\left(\mathrm{Tr}[X\left(\mathcal{M}_{\mathcal{B}}^{-1}+\mathcal{M}_{\mathcal{B}}^{-1}X_{\mathcal{R}}^{T}\left(I-X_{\mathcal{R}}\mathcal{M}_{\mathcal{B}}^{-1}X_{\mathcal{R}}^{T}\right)^{-1}X_{\mathcal{R}}\mathcal{M}_{\mathcal{B}}^{-1}\right)X^{T}-\mathrm{Tr}[X\mathcal{M}^{-1}_{\mathcal{B}}]X^{T}]]\right)\\
&=\frac{\sigma^{2}}{2n}\left(\mathrm{Tr}[X\mathcal{M}_{\mathcal{B}}^{-1}X_{\mathcal{R}}^{T}\left(I-X_{\mathcal{R}}\mathcal{M}_{\mathcal{B}}^{-1}X_{\mathcal{R}}^{T}\right)^{-1}X_{\mathcal{R}}\mathcal{M}_{\mathcal{B}}^{-1}X^{T}]\right)\\
&=\frac{\sigma^{2}}{2n}\left(\mathrm{Tr}[\left(I-X_{\mathcal{R}}\mathcal{M}_{\mathcal{B}}^{-1}X_{\mathcal{R}}^{T}\right)^{-1}X_{\mathcal{R}}\mathcal{M}_{\mathcal{B}}^{-1}X^{T}X\mathcal{M}_{\mathcal{B}}^{-1}X_{\mathcal{R}}^{T}]\right)
    \end{aligned}
\end{equation*}

Note that $\lambda_{max}(I-X_{\mathcal{R}}\mathcal{M}_{\mathcal{B}}^{-1}X_{\mathcal{R}}^{T})\le1$, so we have that
\begin{equation*}
U_{\mathrm{V}}(\mathcal{B})-U_{\mathrm{V}}(\mathcal{A})\ge\frac{\sigma^{2}}{2n}\mathrm{Tr}[X_{\mathcal{R}}\mathcal{M}_{\mathcal{B}}^{-1}X^{T}X\mathcal{M}_{\mathcal{B}}^{-1}X_{\mathcal{R}}^{T}].
\end{equation*}

Furthermore, we also have that
\begin{equation*}
    \begin{aligned}
   U_{\mathrm{V}}(\mathcal{B})-U_{\mathrm{V}}(\mathcal{B}\setminus\{e\})  =\frac{\sigma^{2}}{2n}\left(\mathrm{Tr}[\left(1-x_{e}^{T}\mathcal{M}_{\mathcal{B}}^{-1}x_{e}\right)^{-1}x_{e}^{T}\mathcal{M}_{\mathcal{B}}^{-1}X^{T}X\mathcal{M}_{\mathcal{B}}^{-1}x_{e}]\right).
    \end{aligned} 
\end{equation*}
From Sherman-Morrisson-Woodbury formula, we also can show that
\begin{equation*}
    \begin{aligned}
     \left(1-x_{e}^{T}\mathcal{M}_{\mathcal{B}}^{-1}x_{e}\right)^{-1}=1+x_{e}^{T}(\mathcal{M}_{\mathcal{B}}-x_{e}x_{e}^{T})^{-1}x_{e} =1+x_{e}^{T}\mathcal{M}_{\mathcal{B}\setminus\{e\}}^{-1}x_{e} \le1+\frac{s^{2}}{\sigma}\lambda_{max}(\Lambda), 
    \end{aligned}
\end{equation*} where $\mathcal{M}_{\mathcal{B}\setminus\{e\}}\triangleq X_{\mathcal{B}\setminus\{e\}}^{T}X_{\mathcal{B}\setminus\{e\}}+\sigma^{2}\Lambda^{-1}$ such that $x_{e}^{T}\mathcal{M}_{\mathcal{B}\setminus\{e\}}^{-1}x_{e}\le x_{e}^{T}\left(\sigma^{2}\Lambda^{-1}\right)^{-1}x_{e}\le\frac{x_{e}^{T}\Lambda x_{e}}{\sigma^{2}}\le\frac{s^{2}}{\sigma}\lambda_{max}(\Lambda)$.

As a result, we have that 
\begin{equation*}
    U_{\mathrm{V}}(\mathcal{B})-U_{\mathrm{V}}(\mathcal{B}\setminus\{e\})  \le\frac{\sigma^{2}}{2n}\left(1+\frac{s^{2}}{\sigma}\lambda_{max}(\Lambda)\right)\cdot\mathrm{Tr}[x_{e}^{T}\mathcal{M}_{\mathcal{B}}^{-1}X^{T}X\mathcal{M}_{\mathcal{B}}^{-1}x_{e}].
\end{equation*} 

So we can have that
\begin{equation*}
    \begin{aligned}
    \sum_{e\in\mathcal{B}\setminus\mathcal{A}}\left(U_{\mathrm{V}}(\mathcal{B})-U_{\mathrm{V}}(\mathcal{B}\setminus\{e\})\right)  &\le\frac{\sigma^{2}}{2n}\left(1+\frac{s^{2}}{\sigma}\lambda_{max}(\Lambda)\right)\cdot\sum_{e\in\mathcal{B}\setminus\mathcal{A}}\mathrm{Tr}[x_{e}^{T}\mathcal{M}_{\mathcal{B}}^{-1}X^{T}X\mathcal{M}_{\mathcal{B}}^{-1}x_{e}]\\&=\frac{\sigma^{2}}{2n}\left(1+\frac{s^{2}}{\sigma}\lambda_{max}(\Lambda)\right)\cdot\mathrm{Tr}[X_{\mathcal{R}}\mathcal{M}_{\mathcal{B}}^{-1}X^{T}X\mathcal{M}_{\mathcal{B}}^{-1}X_{\mathcal{R}}^{T}]\\&\le \left(1+\frac{s^{2}}{\sigma}\lambda_{max}(\Lambda)\right)\cdot\left(U_{\mathrm{V}}(\mathcal{B})-U_{\mathrm{V}}(\mathcal{A})\right).
    \end{aligned}
\end{equation*}

In summary, The utility function $U_{\mathrm{V}}(\mathcal{S})
\triangleq
\mathcal{L}_{\mathrm{V}}(\bs{0})
-\mathcal{L}_{\mathrm{V}}(\bs{\theta}^{*}(\mathcal{S}))$ be the induced by the Bayesian V-optimal criterion $\mathcal{L}_{\mathrm{V}}(\bs{\theta}^{*}(\mathcal{S}))$ is monotone and \textbf{ $(1/c,c)$-weakly submodular}, where $c\triangleq1+\frac{s^{2}}{\sigma^{2}}\cdot\lambda_{max}(\Lambda)$, $s\triangleq\max_{i\in[n]}\|x_{i}\|_{2}$ and $\lambda_{\max}(\Lambda)$ is the largest eigenvalue of the matrix $\Lambda$.
\end{proof}

\subsection{Equivalence of the V-Optimal Criterion}
In this subsection, we show the equivalence between the total 
prediction error $\mathcal{L}(\bs{\theta}^{*}(\mathcal{S}))
\triangleq
\frac{1}{2n}\mathbb{E}_{\bs{\epsilon},\bs{\theta}}
\left[
\|\bs{y}-X\bs{\theta}^{*}(\mathcal{S})\|_2^{2}
\right]$ and the Bayesian V-optimal criterion
$\mathcal{L}_{\mathrm{V}}(\bs{\theta}^{*}(\mathcal{S}))
\triangleq
\frac{1}{2n}\mathbb{E}_{\bs{\epsilon},\bs{\theta}}[\|X\big(\bs{\theta}-\bs{\theta}^{*}(\mathcal{S})\big)\|_2^{2}]$. Specifically, when the number of data points $n$ is sufficiently large, these two objectives differ only by an additive constant $\sigma^{2}/2$. 
Consequently, the weak submodularity of $\mathcal{L}_{\mathrm{V}}(\bs{\theta}^{*}(\mathcal{S}))$ established above, to some extent, also implies the weak submodularity of the total prediction error $\mathcal{L}(\bs{\theta}^{*}(\mathcal{S}))$.
\begin{lemma}Under the aforementioned assumptions of Bayesian linear regression,  we can show that
\[ \lim_{n \rightarrow \infty}  \mathcal{L}(\bs{\theta}^{*}(\mathcal{S}))- \mathcal{L}_{\mathrm{V}}(\bs{\theta}^{*}(\mathcal{S}))=\frac{\sigma^2}{2}.\]
\end{lemma}
\begin{proof}
\textbf{Fixed} $\bs{\theta}$, we have,
   \begin{eqnarray*}
 \mathcal{L}(\bs{\theta}^{*}(\mathcal{S}))&=& \frac{1}{2n}\mathbb{E}_{\bs{\epsilon}}\left[
\|\bs{y}-X\bs{\theta}^{*}(\mathcal{S})\|_2^{2}\right]\\
 &=& \frac{1}{2n} \mathbb{E}_{\bs{\epsilon}}\left[\|X\bs{\theta} - X\bs{\theta}^{*}(\mathcal{S})\|^2_2\right]  -  \frac{1}{n} \mathbb{E}_{\bs{\epsilon}}\left[\epsilon^\top(X\bs{\theta}^{*}(\mathcal{S}) - X\bs{\theta} )\right]  +  \frac{1}{2n} \mathbb{E}_{\bs{\epsilon}}\left[\|\bs{\epsilon}\|_2^2\right] \\
 &=& \mathcal{L}_{\mathrm{V}}(\bs{\theta}^{*}(\mathcal{S})) -  \frac{1}{n} \mathbb{E}_{\bs{\epsilon}}\left[\epsilon^\top(X\bs{\theta}^{*}(\mathcal{S}) - X\bs{\theta} )\right] + \frac{\sigma^2}{2} \\  
  &=& \mathcal{L}_{\mathrm{V}}(\bs{\theta}^{*}(\mathcal{S})) -  \frac{1}{n} \mathbb{E}_{\bs{\epsilon}}\left[\sum_{i \in S}\epsilon_i(x_i^T\bs{\theta}^{*}(\mathcal{S}) - x_i^T\bs{\theta} )\right] -  \frac{1}{n} \mathbb{E}_{\bs{\epsilon}}\left[\sum_{i \in [n] \setminus S}\epsilon_i(x_i^T\bs{\theta}^{*}(\mathcal{S}) - x_i^T\bs{\theta})\right]+ \frac{\sigma^2}{2} \\ 
 \end{eqnarray*}
When $n\rightarrow\infty$, we have $\lim_{n\rightarrow\infty}\frac{1}{n} \mathbb{E}_{\bs{\epsilon}}\left[\sum_{i \in S}\epsilon_i(x_i^T\bs{\theta}^{*}(\mathcal{S}) - x_i^T\bs{\theta} )\right] = 0$ since $\mathcal{S}$ is a finite set. Since $\bs{\theta}^{*}(\mathcal{S})$ only depends on $\bs{\epsilon}_{\mathcal{S}}$, we have
\begin{equation*}
\mathbb{E}_{\bs{\epsilon}}\left[\sum_{i \in [n] \setminus S}\epsilon_i(x_i^T\bs{\theta}^{*}(\mathcal{S}) - x_i^T\bs{\theta})\right] = \sum_{i \in [n] \setminus S}\mathbb{E}_{\bs{\epsilon}}\left[\epsilon_i\right]\mathbb{E}_{\bs{\epsilon}}\left[x_i^T\bs{\theta}^{*}(\mathcal{S}) - x_i^T\bs{\theta}\right]=0.
\end{equation*}
    As a consequence, as $\lim_{n \rightarrow \infty}  \mathcal{L}(\bs{\theta}^{*}(\mathcal{S}))- \mathcal{L}_{\mathrm{V}}(\bs{\theta}^{*}(\mathcal{S}))=\frac{\sigma^2}{2}$.

\end{proof}

%% file: appendix_gradient_computation.tex
\section{Gradient Estimators for $\Phi(\bs{\pi})$}
\label{appendix:ge}
In this section, we will show two different gradient estimators for our proposed probabilistic relaxation $\Phi(\bs{\pi})$, namely policy gradient estimator and two-point estimator in Eq.\eqref{eq:two_point_representation}.
\subsection{Policy Gradient Estimator}\label{appendix:pge}
At first, from Eq.\eqref{eq:new_prob_relaxation}, we know that
\begin{equation*}
\begin{aligned}
 \Phi(\bs{\pi})&\triangleq\mathbb{E}_{d_{i,k}\sim\bs{\pi}_{i},\forall i,k}\left[\mathcal{L}\Big(\bs{\theta}^{*}\big(\cup_{i=1}^{h}\cup_{k=1}^{K_i}d_{i,k}\big)\Big)\right]\\
&=\sum_{d_{i,k}\in\mathcal{D}_{i}\cup\{\emptyset\},\forall i,k}
		\Bigg(
		\mathcal{L}\Big(\bs{\theta}^{*}\big(\cup_{i=1}^{h}\cup_{k=1}^{K_i}d_{i,k}\big)\Big)
		\cdot
		\prod_{i=1}^{h}\prod_{k=1}^{K_i}\mathrm{p}(d_{i,k}\mid\bs{\pi}_{i})
		\Bigg),   
\end{aligned}
\end{equation*} where the symbol $d_{i,k}\sim\bs{\pi}_{i}$ means that we sample $d_{i,k}$ from $\mathcal{D}_{i}\cup\{\emptyset\}$ according to the categorical distribution $\bs{\pi}_{i}$. Then, like the policy gradient methods~\citep{zhou2022probabilistic,sutton1999policy,attia2022stochastic}, we can show that
\begin{equation*}
    \begin{aligned}
&\frac{\partial\Phi(\bs{\pi})}{\partial\pi_{i,j}}=\sum_{d_{i,k}\in\mathcal{D}_{i}\cup\{\emptyset\},\forall i\in[h],k\in[K_{i}]}
		\Bigg(\mathcal{L}\Big(\bs{\theta}^{*}\big(\cup_{i=1}^{h}\cup_{k=1}^{K_i}d_{i,k}\big)\Big)
		\nabla_{\pi_{i,j}}\left(\prod_{i=1}^{h}\prod_{k=1}^{K_i}\mathrm{p}(d_{i,k}\mid\bs{\pi}_{i})\right)
		\Bigg)\\&=\sum_{d_{i,k}\in\mathcal{D}_{i}\cup\{\emptyset\},\forall i,k}
		\Bigg(\mathcal{L}\Big(\bs{\theta}^{*}\big(\cup_{i=1}^{h}\cup_{k=1}^{K_i}d_{i,k}\big)\Big)
		\nabla_{\pi_{i,j}}\log\left(\prod_{i=1}^{h}\prod_{k=1}^{K_i}\mathrm{p}(d_{i,k}\mid\bs{\pi}_{i})\right)\left(\prod_{i=1}^{h}\prod_{k=1}^{K_i}\mathrm{p}(d_{i,k}\mid\bs{\pi}_{i})\right)
		\Bigg)\\
        &=\mathbb{E}_{d_{i,k}\sim\bs{\pi}_{i},\forall i\in[h],k\in[K_{i}]}\left[\mathcal{L}\Big(\bs{\theta}^{*}\big(\cup_{i=1}^{h}\cup_{k=1}^{K_i}d_{i,k}\big)\Big)
		\nabla_{\pi_{i,j}}\log\left(\prod_{i=1}^{h}\prod_{k=1}^{K_i}\mathrm{p}(d_{i,k}\mid\bs{\pi}_{i})\right)\right]\\
&=\mathbb{E}_{d_{i,k}\sim\bs{\pi}_{i},\forall i\in[h],k\in[K_{i}]}\left[\mathcal{L}\Big(\bs{\theta}^{*}\big(\cup_{i=1}^{h}\cup_{k=1}^{K_i}d_{i,k}\big)\Big)\sum_{k=1}^{K_{i}}\left(\frac{\mathbb{I}[d_{i,k}=(x_{j}^{(i)},y_{j}^{(i)})]}{\pi_{i,j}}-\frac{\mathbb{I}[d_{i,k}=\emptyset]}{1-\sum_{k=1}^{K_{i}}\pi_{i,k}}\right)\right],
    \end{aligned}
\end{equation*} where the second equality follows from $\frac{\mathrm{d}\log(x)}{\mathrm{d}x}=1/x$ and $\mathbb{I}[\cdot]$ represents the indictor function.

Note that, like Eq.~\eqref{eq:policy_gradient}, the policy gradient estimator of our proposed probabilistic relaxation
 still involves an \emph{unbounded} score factor $\left(\frac{\mathbb{I}[d_{i,k}=(x_{j}^{(i)},y_{j}^{(i)})]}{\pi_{i,j}}-\frac{\mathbb{I}[d_{i,k}=\emptyset]}{1-\sum_{k=1}^{K_{i}}\pi_{i,k}}\right)$. Specifically, when $\pi_{i,k}\rightarrow0$ or $\sum_{k=1}^{K_{i}}\pi_{i,k}\rightarrow1$, the factor $\left(\frac{\mathbb{I}[d_{i,k}=(x_{j}^{(i)},y_{j}^{(i)})]}{\pi_{i,j}}-\frac{\mathbb{I}[d_{i,k}=\emptyset]}{1-\sum_{k=1}^{K_{i}}\pi_{i,k}}\right)$ will also become unbounded.
 \subsection{Two-point Gradient Estimator}\label{appendix:two-point}
 In this subsection, we verify Theorem~\ref{thm:two_point_representation}.
 \begin{proof}
At first, we can show that
\begin{equation}\label{add+1}
    \begin{aligned}
      &\frac{\partial\Phi(\bs{\pi})}{\partial\pi_{i,j}}=\sum_{d_{\hat{i},\hat{k}}\in\mathcal{D}_{\hat{i}}\cup\{\emptyset\},\forall \hat{i}\in[h],\hat{k}\in[K_{i}]}
		\Bigg(\mathcal{L}\Big(\bs{\theta}^{*}\big(\cup_{\hat{i}=1}^{h}\cup_{\hat{k}=1}^{K_i}d_{\hat{i},\hat{k}}\big)\Big)
		\nabla_{\pi_{i,j}}\left(\prod_{\hat{i}=1}^{h}\prod_{\hat{k}=1}^{K_i}\mathrm{p}(d_{\hat{i},\hat{k}}\mid\bs{\pi}_{\hat{i}})\right)
		\Bigg)\\
&=\sum_{d_{\hat{i},\hat{k}}\in\mathcal{D}_{\hat{i}}\cup\{\emptyset\},\forall \hat{i},\hat{k}}
		\Bigg(\mathcal{L}\Big(\bs{\theta}^{*}\big(\cup_{\hat{i}=1}^{h}\cup_{\hat{k}=1}^{K_i}d_{\hat{i},\hat{k}}\big)\Big)
		\sum_{k=1}^{K_{i}}\left(\nabla_{\pi_{i,j}}\mathrm{p}(d_{i,k}\mid\bs{\pi}_{i})\cdot\prod_{(\hat{i},\hat{k})\neq(i,k)}\mathrm{p}(d_{\hat{i},\hat{k}}\mid\bs{\pi}_{\hat{i}})\right)
		\Bigg)\\
&=\sum_{k=1}^{K_{i}}\left(\sum_{d_{\hat{i},\hat{k}}\in\mathcal{D}_{\hat{i}}\cup\{\emptyset\},\forall \hat{i},\hat{k}}
		\Big(\mathcal{L}\Big(\bs{\theta}^{*}\big(\cup_{\hat{i}=1}^{h}\cup_{\hat{k}=1}^{K_i}d_{\hat{i},\hat{k}}\big)\Big)
		\nabla_{\pi_{i,j}}\mathrm{p}(d_{i,k}\mid\bs{\pi}_{i})\cdot\prod_{(\hat{i},\hat{k})\neq(i,k)}\mathrm{p}(d_{\hat{i},\hat{k}}\mid\bs{\pi}_{\hat{i}})
		\Big)\right)\\
&=K_{i}\cdot\left(\sum_{d_{\hat{i},\hat{k}}\in\mathcal{D}_{\hat{i}}\cup\{\emptyset\},\forall \hat{i},\hat{k}}
		\Big(\mathcal{L}\Big(\bs{\theta}^{*}\big(\cup_{\hat{i}=1}^{h}\cup_{\hat{k}=1}^{K_i}d_{\hat{i},\hat{k}}\big)\Big)
		\nabla_{\pi_{i,j}}\mathrm{p}(d_{i,1}\mid\bs{\pi}_{i})\cdot\prod_{(\hat{i},\hat{k})\neq(i,1)}\mathrm{p}(d_{\hat{i},\hat{k}}\mid\bs{\pi}_{\hat{i}})
		\Big)\right)\\
&=K_{i}\left(\sum_{d_{\hat{i},\hat{k}}\in\mathcal{D}_{\hat{i}}\cup\{\emptyset\},(\hat{i},\hat{k})\neq(i,1)}\left(\sum_{d_{i,1}\in\mathcal{D}_{i}\cup\{\emptyset\}}\mathcal{L}\Big(\bs{\theta}^{*}\big(\cup_{\hat{i},\hat{k}}d_{\hat{i},\hat{k}}\big)\Big)\nabla_{\pi_{i,j}}\mathrm{p}(d_{i,1}\mid\bs{\pi}_{i})\right)\prod_{(\hat{i},\hat{k})\neq(i,1)}\mathrm{p}(d_{\hat{i},\hat{k}}\mid\bs{\pi}_{\hat{i}})\right),
    \end{aligned}
\end{equation}where the fourth equality comes from the independence between $\{d_{i,1},d_{i,2},\dots,d_{i,K_{i}}\}$.

When $d_{i,1}=(x_{j}^{(i)},y_{j}^{(i)})$ or $d_{i,1}=\emptyset$, we can know that $\nabla_{\pi_{i,j}}\mathrm{p}(d_{i,1}\mid\bs{\pi}_{i})=1$ or $\nabla_{\pi_{i,j}}\mathrm{p}(d_{i,1}\mid\bs{\pi}_{i})=-1$, respectively. As for other cases, we have $\nabla_{\pi_{i,j}}\mathrm{p}(d_{i,1}\mid\bs{\pi}_{i})=0$. Thus, if we define $\mathcal{S}_{-(i,1)}\triangleq \bigcup_{(i',k')\neq(i,1)}d_{i',k'}$ as the subset obtained by excluding the first draw $d_{i,1}$ from source $\mathcal{D}_{i}$, we have the following equality:
\begin{equation*}
    \begin{aligned}
      \sum_{d_{i,1}\in\mathcal{D}_{i}\cup\{\emptyset\}}\mathcal{L}\Big(\bs{\theta}^{*}\big(\cup_{\hat{i},\hat{k}}d_{i,k}\big)\Big)\nabla_{\pi_{i,j}}\mathrm{p}(d_{i,1}\mid\bs{\pi}_{i})=\mathcal{L}\Big(\bs{\theta}^{*}\big(\mathcal{S}_{-(i,1)}\cup(x^{(i)}_{j},y^{(i)}_{j})\big)\Big)
-
\mathcal{L}\Big(\bs{\theta}^{*}\big(\mathcal{S}_{-(i,1)}\big)\Big)
.  
    \end{aligned}
\end{equation*}

As a result, we have that
     \begin{equation}\label{add+2}
\frac{\partial \Phi(\bs{\pi})}{\partial \pi_{i,j}}
\triangleq
K_i\cdot
\mathbb{E}\Bigg[\displaystyle
\mathcal{L}\Big(\bs{\theta}^{*}\big(\mathcal{S}_{-(i,1)}\cup(x^{(i)}_{j},y^{(i)}_{j})\big)\Big)
-
\mathcal{L}\Big(\bs{\theta}^{*}\big(\mathcal{S}_{-(i,1)}\big)\Big)
\Bigg],
\end{equation} where the expectation is taken over all random draws except $d_{i,1}$.
 \end{proof}
 \clearpage
 \subsection{Reducible Two-point Representation}\label{appendix:Reducible2}
 At first, we review the Bayes rule: for any random event $A$, $B$ and $C$, we have 
\begin{equation*}
    \mathrm{p}(A\mid B,C)\triangleq\frac{\mathrm{p}(B\mid A,C)\cdot\mathrm{p}(A\mid C)}{\mathrm{p}(B\mid C)}.
\end{equation*}

As a result, in Eq.\eqref{eq:reducible_Rep}, if we set $A=\bs{y}$, $B=y$ and $C=\{\bs{x},x,\mathcal{S}\}$, we have that
\begin{equation*}
    \begin{aligned}
     \mathrm{p}(\bs{y}|\bs{x}; \mathcal{S}\cup (x, y))&=\frac{\mathrm{p}(y|x; \mathcal{S}\cup (\bs{x}, \bs{y}))\cdot \mathrm{p}(\bs{y}|\bs{x},x,\mathcal{S})}{\mathrm{p}(y|x,\bs{x},\mathcal{S})}\\
     &=\frac{\mathrm{p}(y|x; \mathcal{S}\cup\mathcal{D})\cdot \mathrm{p}(\bs{y}|\bs{x},\mathcal{S})}{\mathrm{p}(y|x,\mathcal{S})}=\frac{\mathrm{p}(y|x;\mathcal{D})\cdot \mathrm{p}(\bs{y}|\bs{x},\mathcal{S})}{\mathrm{p}(y|x,\mathcal{S})},
    \end{aligned}
\end{equation*} where the second equality comes from the the conditional independence $\mathrm{p}(y_{i}|x_{i},x_{j};\mathcal{S})=\mathrm{p}(y_{i}|x_{i};\mathcal{S})$(Please refer to Eq.(2) in \citep{sujit2023prioritizing}).

As a result, we have that
\begin{equation*}
  \log \mathrm{p}(\bs{y}|\bs{x}; \mathcal{S}\cup (x, y))-\log\mathrm{p}(\bs{y}|\bs{x},\mathcal{S})=\log\mathrm{p}(y|x;\mathcal{D})-\log \mathrm{p}(y|x,\mathcal{S}).
\end{equation*}

Furthermore, in Bayesian probabilistic framework, the loss function is commonly specified as the negative log-likelihood of the predictive distribution, namely, $\ell(x,y;\bs{\theta}^{*}(\mathcal{S}))\triangleq-\log\mathrm{p}(y|x;\mathcal{S})$, then we can have that, throughout the independence between samples:
\begin{equation*}
    \begin{aligned}
    &\mathcal{L}\big(\bs{\theta}^{*}\big(\mathcal{S}\big)\big)=\sum_{i=1}^{n}\ell(x_i,y_i;\bs{\theta}^{*}(\mathcal{S}))\\
    &=-\sum_{i=1}^{n}\log\mathrm{p}(y_i|x_i;\mathcal{S})=-\log\left(\prod_{i=1}^{n}\mathrm{p}(y_i|x_i;\mathcal{S})\right)=-\log\mathrm{p}(\bs{y}|\bs{x}; \mathcal{S}).
    \end{aligned}
\end{equation*}

As a result, we have 
\begin{equation*}
    \begin{aligned}
&\mathcal{L}\big(\bs{\theta}^{*}\big(\mathcal{S}\cup(x,y)\big)\big)-\mathcal{L}\big(\bs{\theta}^{*}\big(\mathcal{S}\big)\big)\\&= \log\mathrm{p}(\bs{y}|\bs{x}; \mathcal{S})-\log\mathrm{p}(\bs{y}|\bs{x}; \mathcal{S}\cup (x, y))\\&=\log\mathrm{p}(y|x;\mathcal{D})-\log \mathrm{p}(y|x,\mathcal{S})=\ell\big(x,y;\bs{\theta}^{*}(\mathcal{D})\big)
-\ell\big(x,y;\bs{\theta}^{*}(\mathcal{S})\big).
    \end{aligned}
\end{equation*}

%% file: appendix_approximation.tex
\section{Proof of Theorem~\ref{thm_SGD}}
\label{appendix:approximation}
In this section, we prove the core theorem of this paper, namely, Theorem~\ref{thm_SGD}. Before that, we prove a key first-order inequality regarding the upper-level utility function  $U(\bs{\theta}^{*}(\mathcal{S}))
\triangleq
\mathcal{L}(\bs{\theta}_{\mathrm{init}})
-
\mathcal{L}(\bs{\theta}^{*}(\mathcal{S}))$.

Next, we define some commonly used notations. At first, we set the corresponding probabilistic relaxation of he upper-level utility function  $U(\bs{\theta}^{*}(\mathcal{S}))$ as $\Psi(\bs{\pi})$, that is,
\begin{equation*}
\begin{aligned}
 \Psi(\bs{\pi})&\triangleq
		\sum_{d_{i,k}\in\mathcal{D}_{i}\cup\{\emptyset\},\forall i,k}
		\Bigg(U\Big(\bs{\theta}^{*}\big(\cup_{i=1}^{h}\cup_{k=1}^{K_i}d_{i,k}\big)\Big)
		\cdot
		\prod_{i=1}^{h}\prod_{k=1}^{K_i}\mathrm{p}(d_{i,k}\mid\bs{\pi}_{i})
		\Bigg)\\
       &=\sum_{d_{i,k}\in\mathcal{D}_{i}\cup\{\emptyset\},\forall i,k}
		\Bigg(\left(\mathcal{L}(\bs{\theta}_{\mathrm{init}})
-\mathcal{L}\Big(\bs{\theta}^{*}\big(\cup_{i=1}^{h}\cup_{k=1}^{K_i}d_{i,k}\big)\Big)\right)
		\cdot
		\prod_{i=1}^{h}\prod_{k=1}^{K_i}\mathrm{p}(d_{i,k}\mid\bs{\pi}_{i})
		\Bigg)\\
        &=\mathcal{L}(\bs{\theta}_{\mathrm{init}})-\Phi(\bs{\pi}),
\end{aligned}
\end{equation*} where $\Phi(\bs{\pi})$ is our proposed probabilistic relaxation in Eq.\eqref{eq:new_prob_relaxation}. Furthermore, for any feasible data subset $\mathcal{S}$ satisfying 
$|\mathcal{S}\cap\mathcal{D}_{i}|= K_i$ for all $i\in[h]$, we define its 
corresponding \emph{uniform prior} as $\bs{\pi}_{\mathcal{S}}
\triangleq
\sum_{i=1}^{h}\frac{1}{K_i}\one_{\mathcal{S}\cap\mathcal{D}_{i}}$.  In other words,  for each source $\mathcal{D}_i$, $\bs{\pi}_{\mathcal{S}}$ places mass
$1/K_i$ on every candidate in $\mathcal{S}\cap\mathcal{D}_i$ and mass $0$ on the
remaining candidates. Note that $\one_{\mathcal{A}}$ denotes the indicator vector of subset $\mathcal{A}$, whose
entry is $1$ if data candidates is included in $\mathcal{A}$ and $0$ otherwise.

With the above preparations in place, we then show that
\begin{theorem}\label{thm_core_inequality}
When the utility function $U(\bs{\theta}^{*}(\mathcal{S}))
\triangleq
\mathcal{L}(\bs{\theta}_{\mathrm{init}})
-
\mathcal{L}(\bs{\theta}^{*}(\mathcal{S}))$ induced by Problem Eq.~\eqref{eq:multi} is $(\gamma,\beta)$-weakly  monotone submodular, for any feasible subset $\mathcal{B}$, namely, $|\mathcal{B}\cap\mathcal{D}_{i}|\le K_{i},\forall i\in[h]$, the following inequality holds:
\begin{equation*}
\gamma^{2}\cdot U(\bs{\theta}^{*}(\mathcal{B}))-(\beta+\beta(1-\gamma)+\gamma^{2}) \cdot\Psi(\bs{\pi})\le\left\langle\bs{\pi}_{\mathcal{B}}-\bs{\pi},\nabla\Psi(\bs{\pi})\right\rangle.
\end{equation*}
\end{theorem}
\begin{proof}
At first, we investigate $\langle \bs{\pi}_{\mathcal{B}}, \nabla \Psi(\bs{\pi}) \rangle$. To this end, for each source $\mathcal{D}_i$,  we rewrite $\mathcal{B}\cap\mathcal{D}_i
=
\{(x^{(i)}_{j_1},y^{(i)}_{j_1}),\dots,
(x^{(i)}_{j_{K_i}},y^{(i)}_{j_{K_i}})\}$,
where $j_1,\dots,j_{K_i}\in[n_i]$ are distinct indices.

\begin{equation}\label{equ_appendixD_1}
	\begin{aligned}
\langle\bs{\pi}_{\mathcal{S}},\nabla\Psi(\bs{\pi})\rangle&=\langle\sum_{i=1}^{h}\frac{1}{K_i}\one_{\mathcal{S}\cap\mathcal{D}_{i}},\nabla\Psi(\bs{\pi})\rangle=\sum_{i=1}^{h}\sum_{k=1}^{K_{i}}\frac{1}{K_i}\cdot\frac{\partial \Psi(\bs{\pi})}{\partial \pi_{i,j_{k}}}=-\sum_{i=1}^{h}\sum_{k=1}^{K_{i}}\frac{1}{K_i}\cdot\frac{\partial \Phi(\bs{\pi})}{\partial \pi_{i,j_{k}}}\\
&=\sum_{i=1}^{h}\sum_{k=1}^{K_{i}}\mathbb{E}\Bigg[\mathcal{L}\Big(\bs{\theta}^{*}\big(\mathcal{S}_{-(i,1)}\big)\Big)-\mathcal{L}\Big(\bs{\theta}^{*}\big(\mathcal{S}_{-(i,1)}\cup(x^{(i)}_{j_k},y^{(i)}_{j_k})\big)\Big)\Bigg]\\
&=\sum_{i=1}^{h}\sum_{k=1}^{K_{i}}\mathbb{E}\Bigg[U\Big(\bs{\theta}^{*}\big(\mathcal{S}_{-(i,1)}\cup(x^{(i)}_{j_k},y^{(i)}_{j_k})\big)\Big)
-U\Big(\bs{\theta}^{*}\big(\mathcal{S}_{-(i,1)}\big)\Big)\Bigg]\\
&=\sum_{i=1}^{h}\sum_{k=1}^{K_{i}}\mathbb{E}\Bigg[U\Big(\bs{\theta}^{*}\big(\mathcal{S}_{-(i,k)}\cup(x^{(i)}_{j_k},y^{(i)}_{j_k})\big)\Big)
-U\Big(\bs{\theta}^{*}\big(\mathcal{S}_{-(i,k)}\big)\Big)\Bigg]\\
\end{aligned}
\end{equation} where $\mathcal{S}_{-(i,k)}\triangleq \bigcup_{(i',k')\neq(i,k)}d_{i',k'}$ as the subset obtained by excluding the $k$-th draw $d_{i,k}$ from source $\mathcal{D}_{i}$ and the final inequality comes from the fact that
$\mathcal{S}_{-(i,k)}$ and $\mathcal{S}_{-(i,1)}$ are identically distributed.

Furthermore, if a set function $f$ is $(\gamma,\beta)$-weakly monotone submodular,  we firstly can show that for any two data candidates $e_{1},e_{2}\in\mathcal{D}$ and the subset $\mathcal{D}\subseteq\mathcal{D}$, we have that  
\begin{equation*}
\gamma\big(f(e_1| \mathcal{B}\cup\{e_{2}\})+f(e_{2}|\mathcal{B})\big)=\gamma\big(f(\mathcal{B}\cup\{e_{1},e_{2}\})-f(\mathcal{B})\big)\le f(e_{1}|\mathcal{B})+ f(e_{2}|\mathcal{B}),
\end{equation*}such that 
\begin{equation}\label{equ_appendix2_3}
f(e_{1}|\mathcal{B})\ge \gamma f(e_1| \mathcal{B}\cup\{e_{2}\})-(1-\gamma)f(e_{2}|\mathcal{B}).
\end{equation}

Note that, in the field of set function optimization, we often denote the marginal gain of adding $\mathcal{B}$ to $\mathcal{A}$ by $f(\mathcal{B}|\mathcal{A})\triangleq f(\mathcal{A}\cup\mathcal{B})-f(\mathcal{A})$. Especially when $\mathcal{B}\triangleq\{e\}$ is a singleton, we simply write
$f(e\mid \mathcal{A})\triangleq f(\{e\}\mid \mathcal{A})$.

According to Eq.\eqref{equ_appendix2_3}, we subsequently can have that
\begin{equation}\label{equ_appendix2_4}
	\begin{aligned}
		&\langle\bs{\pi}_{\mathcal{S}},\nabla\Psi(\bs{\pi})\rangle\\
		&=\sum_{i=1}^{h}\sum_{k=1}^{K_{i}}\mathbb{E}\Bigg[U\Big(\bs{\theta}^{*}\big(\mathcal{S}_{-(i,k)}\cup(x^{(i)}_{j_k},y^{(i)}_{j_k})\big)\Big)
-U\Big(\bs{\theta}^{*}\big(\mathcal{S}_{-(i,k)}\big)\Big)\Bigg]\\
&=\sum_{i=1}^{h}\sum_{k=1}^{K_{i}}\mathbb{E}\Bigg[
U\Big(\bs{\theta}^{*}\big((x^{(i)}_{j_k},y^{(i)}_{j_k})\mid \mathcal{S}_{-(i,k)}\big)\Big)\Bigg]\\
	 &\ge\sum_{i=1}^{h}\sum_{k=1}^{K_{i}}\mathbb{E}\Bigg[
\gamma\cdot U\Big(\bs{\theta}^{*}\big((x^{(i)}_{j_k},y^{(i)}_{j_k})\mid \mathcal{S}_{-(i,k)}\cup d_{i,k}\big)\Big)-(1-\gamma)\cdot U\Big(\bs{\theta}^{*}\big(d_{i,k}|\mathcal{S}_{-(i,k)}\big)\Big)\Bigg],
	\end{aligned}
\end{equation} where $d_{i,k}$ is the $k$-th draw  from source $\mathcal{D}_{i}$.

Then, from the first inequality in Eq.\eqref{eq:two_sided_weak_submodularity},  we have that 
\begin{equation}\label{equ_appendix2_5}
	\begin{aligned}
&\sum_{i=1}^{h}\sum_{k=1}^{K_{i}}\mathbb{E}\Bigg[U\Big(\bs{\theta}^{*}\big((x^{(i)}_{j_k},y^{(i)}_{j_k})\mid \mathcal{S}_{-(i,k)}\cup d_{i,k}\big)\Big)\Bigg]\\
	&\ge\gamma\cdot\mathbb{E}\Bigg[U\Big(\bs{\theta}^{*}\big(\cup_{i=1}^{h}\cup_{k=1}^{K_{i}}(x^{(i)}_{j_k},y^{(i)}_{j_k})\mid \mathcal{S}_{-(i,k)}\cup d_{i,k}\big)\Big)\Bigg]\\
    &=\gamma\cdot\mathbb{E}\Bigg[U\Big(\bs{\theta}^{*}\big(\mathcal{B}\mid \cup_{i=1}^{h}\cup_{k=1}^{K_{i}} d_{i,k}\big)\Big)\Bigg]\\
&=\gamma\cdot\mathbb{E}\Bigg[U\Big(\bs{\theta}^{*}\big(\mathcal{B}\cup( \cup_{i=1}^{h}\cup_{k=1}^{K_{i}} d_{i,k})\big)\Big)-U\Big(\bs{\theta}^{*}\big( \cup_{i=1}^{h}\cup_{k=1}^{K_{i}} d_{i,k}\big)\Big)\Bigg]\\
	&\ge\gamma\cdot\Big(U(\bs{\theta}^{*}(\mathcal{B}))-\Psi(\bs{\pi})\Big),
\end{aligned}
\end{equation} where the final inequality follows from the monotonicity of the utility function $U$.

Furthermore,  from the second inequality in Eq.\eqref{eq:two_sided_weak_submodularity}, we also have, 
\begin{equation}\label{equ_appendix2_6}
	\begin{aligned}
		&\sum_{i=1}^{h}\sum_{k=1}^{K_{i}}\mathbb{E}\Bigg[U\Big(\bs{\theta}^{*}\big(d_{i,k}|\mathcal{S}_{-(i,k)}\big)\Big)\Bigg]\\
        &\le\beta\cdot\mathbb{E}\Bigg[ U\Big(\bs{\theta}^{*}\big(\cup_{i=1}^{h}\cup_{k=1}^{K_{i}} d_{i,k}\big)\Big)-U(\bs{\theta}^{*}(\emptyset))\Bigg]\\
		&\le\beta\cdot\mathbb{E}\Bigg[ U\Big(\bs{\theta}^{*}\big(\cup_{i=1}^{h}\cup_{k=1}^{K_{i}} d_{i,k}\big)\Big)\Bigg]=\beta\cdot\Psi(\bs{\pi}), 
	\end{aligned}
\end{equation}where the second inequality follows from Eq.\eqref{eq:two_sided_weak_submodularity}.
	
Merging Eq.\eqref{equ_appendix2_6} and Eq.\eqref{equ_appendix2_5} into Eq.\eqref{equ_appendix2_4}, we have that
	\begin{equation}\label{add+4}
		\begin{aligned}
			&\langle\bs{\pi}_{\mathcal{S}},\nabla\Psi(\bs{\pi})\rangle\\
&\ge\sum_{i=1}^{h}\sum_{k=1}^{K_{i}}\mathbb{E}\Bigg[
\gamma\cdot U\Big(\bs{\theta}^{*}\big((x^{(i)}_{j_k},y^{(i)}_{j_k})\mid \mathcal{S}_{-(i,k)}\cup d_{i,k}\big)\Big)-(1-\gamma)\cdot U\Big(\bs{\theta}^{*}\big(d_{i,k}|\mathcal{S}_{-(i,k)}\big)\Big)\Bigg]\\
&\ge\gamma^{2}\Big(U(\bs{\theta}^{*}(\mathcal{B}))-\Psi(\bs{\pi})\Big)-(1-\gamma)\beta\cdot\Psi(\bs{\pi})\\
&=\gamma^{2}\cdot U(\bs{\theta}^{*}(\mathcal{B}))-(\beta(1-\gamma)+\gamma^{2})\cdot\Psi(\bs{\pi}).
		\end{aligned}
	\end{equation}

Secondly, we investigate $\langle \bs{\pi}, \nabla \Psi(\bs{\pi}) \rangle$.

At first,  from the fifth equality in Eq~\eqref{add+1} and Eq~\eqref{add+2}, we have that 
\begin{equation*}
	\begin{aligned}
		&\frac{\partial\Psi(\bs{\pi})}{\partial\pi_{i,j}}=-\frac{\partial\Phi(\bs{\pi})}{\partial\pi_{i,j}}\\&=-K_{i}\left(\sum_{d_{\hat{i},\hat{k}}\in\mathcal{D}_{\hat{i}}\cup\{\emptyset\},(\hat{i},\hat{k})\neq(i,1)}\left(\sum_{d_{i,1}\in\mathcal{D}_{i}\cup\{\emptyset\}}\mathcal{L}\Big(\bs{\theta}^{*}\big(\cup_{\hat{i},\hat{k}}d_{\hat{i},\hat{k}}\big)\Big)\nabla_{\pi_{i,j}}\mathrm{p}(d_{i,1}\mid\bs{\pi}_{i})\right)\prod_{(\hat{i},\hat{k})\neq(i,1)}\mathrm{p}(d_{\hat{i},\hat{k}}\mid\bs{\pi}_{\hat{i}})\right)\\
&=-K_{i}\left(\sum_{d_{\hat{i},\hat{k}}\in\mathcal{D}_{\hat{i}}\cup\{\emptyset\},(\hat{i},\hat{k})\neq(i,1)}\mathcal{L}\Big(\bs{\theta}^{*}\big((x^{(i)}_{j},y^{(i)}_{j})\mid \cup_{(\hat{i},\hat{k})\neq(k,1)}d_{\hat{i},\hat{k}}\big)\Big)\prod_{(\hat{i},\hat{k})\neq(i,1)}\mathrm{p}(d_{\hat{i},\hat{k}}\mid\bs{\pi}_{\hat{i}})\right)\\
&=K_{i}\left(\sum_{d_{\hat{i},\hat{k}}\in\mathcal{D}_{\hat{i}}\cup\{\emptyset\},(\hat{i},\hat{k})\neq(i,1)}U\Big(\bs{\theta}^{*}\big((x^{(i)}_{j},y^{(i)}_{j})\mid \cup_{(\hat{i},\hat{k})\neq(k,1)}d_{\hat{i},\hat{k}}\big)\Big)\prod_{(\hat{i},\hat{k})\neq(i,1)}\mathrm{p}(d_{\hat{i},\hat{k}}\mid\bs{\pi}_{\hat{i}})\right)
	\end{aligned}
\end{equation*} 
As a result, we have that
{\small
\begin{equation*}
\begin{aligned}
	&\langle \bs{\pi}, \nabla \Psi(\bs{\pi}) \rangle\\&=\sum_{i=1}^{h}\sum_{j=1}^{n_{i}}\pi_{i,j}\cdot\frac{\partial\Psi(\bs{\pi})}{\partial\pi_{i,j}}\\
&=\sum_{i=1}^{h}\sum_{j=1}^{n_{i}}\pi_{i,j}K_{i}\left(\sum_{d_{\hat{i},\hat{k}}\in\mathcal{D}_{\hat{i}}\cup\{\emptyset\},(\hat{i},\hat{k})\neq(i,1)}U\Big(\bs{\theta}^{*}\big((x^{(i)}_{j},y^{(i)}_{j})\mid \cup_{(\hat{i},\hat{k})\neq(k,1)}d_{\hat{i},\hat{k}}\big)\Big)\prod_{(\hat{i},\hat{k})\neq(i,1)}\mathrm{p}(d_{\hat{i},\hat{k}}\mid\bs{\pi}_{\hat{i}})\right)\\    
&=\sum_{i=1}^{h}K_{i}\sum_{j=1}^{n_{i}}\mathrm{p}\left(d_{i,1}=(x^{(i)}_{j},y^{(i)}_{j})\mid\bs{\pi}_{i}\right)\left(\sum_{d_{\hat{i},\hat{k}}\in\mathcal{D}_{\hat{i}}\cup\{\emptyset\},(\hat{i},\hat{k})\neq(i,1)}U\Big(\bs{\theta}^{*}\big((x^{(i)}_{j},y^{(i)}_{j})\mid \cup_{(\hat{i},\hat{k})\neq(k,1)}d_{\hat{i},\hat{k}}\big)\Big)\prod_{(\hat{i},\hat{k})\neq(i,1)}\mathrm{p}(d_{\hat{i},\hat{k}}\mid\bs{\pi}_{\hat{i}})\right)\\ &=\sum_{i=1}^{h}K_{i}\sum_{j=1}^{n_{i}}\left(\sum_{d_{\hat{i},\hat{k}}\in\mathcal{D}_{\hat{i}}\cup\{\emptyset\},(\hat{i},\hat{k})\neq(i,1)}U\Big(\bs{\theta}^{*}\big((x^{(i)}_{j},y^{(i)}_{j})\mid \cup_{(\hat{i},\hat{k})\neq(k,1)}d_{\hat{i},\hat{k}}\big)\Big)\cdot\mathrm{p}\left(d_{i,1}=(x^{(i)}_{j},y^{(i)}_{j})\mid\bs{\pi}_{i}\right)\prod_{(\hat{i},\hat{k})\neq(i,1)}\mathrm{p}(d_{\hat{i},\hat{k}}\mid\bs{\pi}_{\hat{i}})\right)\\ &=\sum_{i=1}^{h}K_{i}\sum_{j=1}^{n_{i}}\left(\sum_{d_{\hat{i},\hat{k}}\in\mathcal{D}_{\hat{i}}\cup\{\emptyset\},\forall\hat{i},\hat{k}}U\Big(\bs{\theta}^{*}\big(d_{k,1}\mid \cup_{(\hat{i},\hat{k})\neq(k,1)}d_{\hat{i},\hat{k}}\big)\Big)\prod_{\hat{i}=1}^{h}\prod_{\hat{k}=1}^{K_{i}}\mathrm{p}(d_{\hat{i},\hat{k}}\mid\bs{\pi}_{\hat{i}})\right)\\   &=\sum_{i=1}^{h}K_{i}\cdot\mathbb{E}\Bigg[U\Big(\bs{\theta}^{*}\big(d_{i,1}\mid \cup_{(\hat{i},\hat{k})\neq(k,1)}d_{\hat{i},\hat{k}}\big)\Big)\Bigg]=\sum_{i=1}^{h}\sum_{k=1}^{K_{i}}\cdot\mathbb{E}\Bigg[U\Big(\bs{\theta}^{*}\big(d_{i,k}\mid \cup_{(\hat{i},\hat{k})\neq(k,1)}d_{\hat{i},\hat{k}}\big)\Big)\Bigg]\\
&=\sum_{i=1}^{h}\sum_{k=1}^{K_{i}}\mathbb{E}\Bigg[U\Big(\bs{\theta}^{*}\big(d_{i,k}|\mathcal{S}_{-(i,k)}\big)\Big)\Bigg].
\end{aligned}
\end{equation*}} 

Then, from the weak submodularity of $U$, we have 
\begin{equation}\label{add+3}
    \langle \bs{\pi}, \nabla \Psi(\bs{\pi}) \rangle\le\sum_{i=1}^{h}\sum_{k=1}^{K_{i}}\mathbb{E}\Bigg[U\Big(\bs{\theta}^{*}\big(d_{i,k}|\mathcal{S}_{-(i,k)}\big)\Big)\Bigg]\le\beta\cdot\Psi(\bs{\pi}),
\end{equation} where the final inequality  comes from Eq.\eqref{eq:two_sided_weak_submodularity}.

As a result, by merging Eq.~\ref{add+3} and Eq.~\ref{add+4}, we then have that
\begin{equation*}
    \begin{aligned}
      \left\langle\bs{\pi}_{\mathcal{B}}-\bs{\pi},\nabla\Psi(\bs{\pi})\right\rangle&=\left\langle\bs{\pi}_{\mathcal{B}},\nabla\Psi(\bs{\pi})\right\rangle-\left\langle\bs{\pi},\nabla\Psi(\bs{\pi})\right\rangle\\
      &=\ge \gamma^{2}\cdot U(\bs{\theta}^{*}(\mathcal{B}))-(\beta(1-\gamma)+\gamma^{2})\cdot\Psi(\bs{\pi})-\beta\cdot\Psi(\bs{\pi})\\
      &=\ge\gamma^{2}\cdot U(\bs{\theta}^{*}(\mathcal{B}))-(\beta+\beta(1-\gamma)+\gamma^{2}) \cdot\Psi(\bs{\pi}).
    \end{aligned}
\end{equation*}
\end{proof}

Beside Theorem~\ref{thm_core_inequality}, we next prove the boundedness of $\nabla\Psi(\bs{\pi})$, that is to say, 
\begin{lemma}\label{lemma_bound_gradient}
 if we denote the maximum marginal value of the utility function $U(\bs{\theta}^{*}(\mathcal{S}))$ as $M_{U}\triangleq\max_{\mathcal{S}\subseteq\mathcal{D},e\in\mathcal{D}\setminus\mathcal{S}}\Big(U(\bs{\theta}^{*}(\mathcal{S}\cup\{e\}))-U(\bs{\theta}^{*}(\mathcal{S}))\Big)$, we have that, for any $i\in[h]$, $j\in[n_{i}]$,
	\begin{equation*}
		\Big|\frac{\partial\Psi(\bs{\pi})}{\partial\pi_{i,j}}\Big|\le M_{U}\cdot K_{max},
	\end{equation*} where $K_{max}\triangleq\max_{i\in[h]}K_i$ is the maximum budget over the $h$ communities $\{\mathcal{D}_{1},\dots,\mathcal{D}_{h}\}$. Similarly, we also can infer that the two-point gradient estimator in Eq.\eqref{eq:two_point_representation}
    also can be bounded by $M_{U}\cdot K_{max}$.
\end{lemma} 

Next, we prove Theorem~\ref{thm_SGD}.
\begin{proof}
At first, let $\mathcal{S}^{*}\subseteq\mathcal{D}$ denote the optimal solution for Eq.\eqref{eq:multi}. Then, 
we denote $\bs{\pi}^{(t)}\triangleq(\bs{\pi}^{(t+1)}_1,\dots,\bs{\pi}^{(t+1)}_h)$ and set the correponding two-point gradient estimator as $\widehat{\nabla\Phi}(\bs{\pi})$, namely, $\widehat{\nabla\Phi}(\bs{\pi})\triangleq[\widehat{\frac{\partial\Phi(\bs{\pi})}{\partial\pi_{i,j}}}]_{i,j}$ with $\widehat{\frac{\partial\Phi(\bs{\pi})}{\partial\pi_{i,j}}}\triangleq
K_{i}\left(\mathcal{L}\Big(\bs{\theta}^{*}\big(\mathcal{S}_{-(i,1)}\cup(x^{(i)}_{j},y^{(i)}_{j})\big)\Big)
-
\mathcal{L}\Big(\bs{\theta}^{*}\big(\mathcal{S}_{-(i,1)}\big)\Big)\right)$. 

Subsequently, according to the projective step in \texttt{ReSGD}, we can show that
\begin{equation*}
\begin{aligned}
\left\|\bs{\pi}^{(t+1)}-\bs{\pi}_{\mathcal{S}^{*}}\right\|_{2}^{2}&\le\left\|\bs{\pi}^{(t)}-\eta\cdot\widehat{\nabla\Phi}(\bs{\pi}^{(t)})-\bs{\pi}_{\mathcal{S}^{*}}\right\|_{2}^{2}\\
&=\left\|\bs{\pi}^{(t)}-\bs{\pi}_{\mathcal{S}^{*}}\right\|_{2}^{2}-2\eta\cdot\left\langle\widehat{\nabla\Phi}(\bs{\pi}^{(t)}), \bs{\pi}^{(t)}-\bs{\pi}_{\mathcal{S}^{*}}\right\rangle+\eta^{2}\left\|\widehat{\nabla\Phi}(\bs{\pi}^{(t)})\right\|_{2}^{2}\\
&=\left\|\bs{\pi}^{(t)}-\bs{\pi}_{\mathcal{S}^{*}}\right\|_{2}^{2}-2\eta\cdot\left\langle\widehat{\nabla\Phi}(\bs{\pi}^{(t)}), \bs{\pi}^{(t)}-\bs{\pi}_{\mathcal{S}^{*}}\right\rangle+\eta^{2}\cdot n\cdot M_{U}^{2}\cdot K_{max}^{2}\\
\end{aligned}		
\end{equation*} where  the final  inequality comes from Lemma~\ref{lemma_bound_gradient}.

As a result, we have that
\begin{equation}\label{equ_SGA_1}
	\begin{aligned}
		&\mathbb{E}\left\|\bs{\pi}^{(t+1)}-\bs{\pi}_{\mathcal{S}^{*}}\right\|_{2}^{2}\le\left\|\bs{\pi}^{(t)}-\eta\cdot\widehat{\nabla\Phi}(\bs{\pi}^{(t)})-\bs{\pi}_{\mathcal{S}^{*}}\right\|_{2}^{2}\\
&=\mathbb{E}\left\|\bs{\pi}^{(t)}-\bs{\pi}_{\mathcal{S}^{*}}\right\|_{2}^{2}-2\eta\cdot\mathbb{E}\left\langle\widehat{\nabla\Phi}(\bs{\pi}^{(t)}), \bs{\pi}^{(t)}-\bs{\pi}_{\mathcal{S}^{*}}\right\rangle+\eta^{2}\cdot n\cdot M_{U}^{2}\cdot K_{max}^{2}\\
&=\mathbb{E}\left\|\bs{\pi}^{(t)}-\bs{\pi}_{\mathcal{S}^{*}}\right\|_{2}^{2}-2\eta\cdot\mathbb{E}\bigg[\left\langle\mathbb{E}\big[\widehat{\nabla\Phi}(\bs{\pi}^{(t)})\mid \bs{\pi}^{(t)}\big], \bs{\pi}^{(t)}-\bs{\pi}_{\mathcal{S}^{*}}\right\rangle\bigg]+\eta^{2}\cdot n\cdot M_{U}^{2}\cdot K_{max}^{2}\\
&=\mathbb{E}\left\|\bs{\pi}^{(t)}-\bs{\pi}_{\mathcal{S}^{*}}\right\|_{2}^{2}-2\eta\cdot\mathbb{E}\bigg[\left\langle\nabla\Phi(\bs{\pi}^{(t)}), \bs{\pi}^{(t)}-\bs{\pi}_{\mathcal{S}^{*}}\right\rangle\bigg]+\eta^{2}\cdot n\cdot M_{U}^{2}\cdot K_{max}^{2}\\
&=\mathbb{E}\left\|\bs{\pi}^{(t)}-\bs{\pi}_{\mathcal{S}^{*}}\right\|_{2}^{2}+2\eta\cdot\mathbb{E}\bigg[\left\langle\nabla\Psi(\bs{\pi}^{(t)}), \bs{\pi}^{(t)}-\bs{\pi}_{\mathcal{S}^{*}}\right\rangle\bigg]+\eta^{2}\cdot n\cdot M_{U}^{2}\cdot K_{max}^{2},
	\end{aligned}		
\end{equation} where the final equality follows from $\nabla\Psi=-\nabla\Phi$.

Furthermore, according to $(\gamma,\beta)$-weak submodularity, namely, Theorem~\ref{thm_core_inequality}, we also can show that,
\begin{equation*}
    \begin{aligned}
      &\mathbb{E}\left\|\bs{\pi}^{(t+1)}-\bs{\pi}_{\mathcal{S}^{*}}\right\|_{2}^{2}\le\mathbb{E}\left\|\bs{\pi}^{(t)}-\bs{\pi}_{\mathcal{S}^{*}}\right\|_{2}^{2}+2\eta\cdot\mathbb{E}\bigg[\left\langle\nabla\Psi(\bs{\pi}^{(t)}), \bs{\pi}^{(t)}-\bs{\pi}_{\mathcal{S}^{*}}\right\rangle\bigg]+\eta^{2}\cdot n\cdot M_{U}^{2}\cdot K_{max}^{2}\\
      &\le\mathbb{E}\left\|\bs{\pi}^{(t)}-\bs{\pi}_{\mathcal{S}^{*}}\right\|_{2}^{2}-2\eta\cdot\mathbb{E}\bigg[\left\langle\nabla\Psi(\bs{\pi}^{(t)}), \bs{\pi}_{\mathcal{S}^{*}}-\bs{\pi}^{(t)}\right\rangle\bigg]+\eta^{2}\cdot n\cdot M_{U}^{2}\cdot K_{max}^{2}\\
      &\le\mathbb{E}\left\|\bs{\pi}^{(t)}-\bs{\pi}_{\mathcal{S}^{*}}\right\|_{2}^{2}+2\eta\cdot\left((\beta+\beta(1-\gamma)+\gamma^{2})\cdot\mathbb{E}[\Psi(\bs{\pi}^{(t)})]-\gamma^{2}\cdot U(\bs{\theta}^{*}(\mathcal{B}))\right)+\eta^{2}\cdot n\cdot M_{U}^{2}\cdot K_{max}^{2}
    \end{aligned}
\end{equation*}

 We further can have that
\begin{equation}\label{add++}
    \begin{aligned}
     &\sum_{t=0}^{T-1}2\eta\cdot\left(\gamma^{2}\cdot U(\bs{\theta}^{*}(\mathcal{B}))-(\beta+\beta(1-\gamma)+\gamma^{2})\cdot\mathbb{E}[\Psi(\bs{\pi}^{(t)})]\right)\\&\le\sum_{t=0}^{T-1}\left(\mathbb{E}\left\|\bs{\pi}^{(t)}-\bs{\pi}_{\mathcal{S}^{*}}\right\|_{2}^{2}-\mathbb{E}\left\|\bs{\pi}^{(t+1)}-\bs{\pi}_{\mathcal{S}^{*}}\right\|_{2}^{2} \right)+nT\eta^{2}\cdot M_{U}^{2}\cdot K_{max}^{2}\\&=\mathbb{E}\left\|\bs{\pi}^{(0)}-\bs{\pi}_{\mathcal{S}^{*}}\right\|_{2}^{2}-\mathbb{E}\left\|\bs{\pi}^{(T)}-\bs{\pi}_{\mathcal{S}^{*}}\right\|_{2}^{2}+nT\eta^{2}\cdot M_{U}^{2}\cdot K_{max}^{2}.
    \end{aligned}
\end{equation}

Due to the randomness of the index $\tau$,
 we finally can infer that the subset return by our \texttt{ReSGD} satifies that
\begin{equation*}
	\begin{aligned}
		&\mathbb{E}[U(\bs{\theta}^{*}(\mathcal{S}))]=\mathbb{E}[\Psi(\bs{\pi}^{(\tau)})]=\frac{\sum_{t=0}^{T-1}\mathbb{E}[\Psi(\bs{\pi}^{(t)})]}{T}\\
        &\ge\Big(\frac{\gamma^{2}}{\beta+\beta(1-\gamma)+\gamma^{2}}\Big)\cdot U(\bs{\theta}^{*}(\mathcal{B}))-\frac{\mathbb{E}\left\|\bs{\pi}^{(0)}-\bs{\pi}_{\mathcal{S}^{*}}\right\|_{2}^{2}-\mathbb{E}\left\|\bs{\pi}^{(T)}-\bs{\pi}_{\mathcal{S}^{*}}\right\|_{2}^{2}}{2\eta\cdot T\cdot(\beta+\beta(1-\gamma)+\gamma^{2})}-\frac{n\cdot M_{U}^{2}\cdot K_{max}^{2}}{2(\beta+\beta(1-\gamma)+\gamma^{2})}\cdot\eta,   
	\end{aligned}	
\end{equation*} where the second equality follows from the randomness of the index $\tau$ and the final inequality comes from Eq.\eqref{add++}.

Since the Lemma B.1 in \citep{thiery2022two} show that  $\beta(1-\gamma)+\gamma^{2}\ge\frac{3}{4}$, when $\eta
\triangleq
\sqrt{\frac{1}{nK_{max}^{2}\cdot T}}$, we have $\mathbb{E}\!\left[
U(\bs{\theta}^{*}(\mathcal{S}))
\right]
\ge
\frac{\gamma^{2}}{\beta+\beta(1-\gamma)+\gamma^{2}}
\cdot U(\bs{\theta}^{*}(\mathcal{S}^{*}))
-
\mathcal{O}\Big(
\sqrt{\frac{n\cdot K_{max}^{2}}{T}}
\Big)$. 

\end{proof}

%% file: appendix_Projection.tex
\section{Projection Calculation}\label{appendix:projection}
\label{app:projection}
\begin{algorithm}[t]
\caption{Projection onto $\Delta_{n}$}
\label{alg:projection_delta}
\KwIn{Vector $\bs{z}\in\mathbb{R}^{n}$}
\KwOut{Vector $\bs{\pi}\triangleq\Pi_{\Delta_n}(\bs{z})$}

Set $\bs{z}_{+}\leftarrow \max\{\bs{z},\bs{0}\}$\;

\If{$\bs{1}^{\top}\bs{z}_{+}\le 1$}{
    \Return{$\bs{z}_{+}$}\;
}

Find $\lambda^{*}\ge 0$ such that $\sum_{j=1}^{n}\max\{z_j-\lambda^{*},0\}=1$ by bisection.

Set $\pi_j\leftarrow\max\{z_j-\lambda^{*},0\}$ for all $j\in[n]$\;

\Return{$\bs{\pi}$}\;
\end{algorithm}
In Algorithm~\ref{alg:reducible_sgd}, each prior vector is updated by projected
gradient descent, namely, 
\begin{equation*}
    \bs{\pi}_{i}^{(t+1)}
\triangleq
\Pi_{\Delta_{n_i}}\left(
\bs{\pi}_{i}^{(t)}-\eta\widehat{\bs{g}}_{i}^{(t)}
\right),
\quad \forall i\in[h],
\end{equation*} where $\Delta_{n_i}\triangleq\{\bs{p}\in[0,1]^{n_i}:\ \|\bs{p}\|_1\le 1\}$ denotes the sub-probability simplex.

If we set  $\bs{z}
\triangleq
\bs{\pi}_{i}^{(t)}-\eta\widehat{\bs{g}}_{i}^{(t)}$, the projection operation is equivalent to solving the following quadratic
programming:
\begin{equation}\label{eq:qp}
    \min_{\bs{\pi}\in\mathbb{R}^{n}}
\frac{1}{2}\|\bs{\pi}-\bs{z}\|_2^2,
\quad
\text{s.t.}\quad
\bs{\pi}\ge \bs{0},\quad
\bs{1}^{\top}\bs{\pi}\le 1.
\end{equation}
The main purpose of this section is to provide an efficient method for this
singly constrained quadratic programming problem Eq.~\eqref{eq:qp}, namely, Algorithm~\ref{alg:projection_delta}.

Specifically, Algorithm~\ref{alg:projection_delta} separates the projection
into two cases. It first computes the nonnegative part
$\bs{z}_{+}\triangleq\max\{\bs{z},\bs{0}\}$. When
$\|\bs{z}_{+}\|_{1}\le 1$, this vector already lies in $\Delta_n$ and is
therefore the Euclidean projection of $\bs{z}$ onto $\Delta_n$.

Specifically, Algorithm~\ref{alg:projection_delta} proceeds in two cases.
It first computes the nonnegative truncation
$\bs{z}_{+}\triangleq\max\{\bs{z},\bs{0}\}$. If
$\|\bs{z}_{+}\|_{1}\le 1$, then $\bs{z}_{+}$ is already feasible for
Eq.~\eqref{eq:qp}. In this case, no further normalization is needed, and the
projection is simply given by $\bs{z}_{+}$.

If $\|\bs{z}_{+}\|_{1}>1$, we find the optimal dual variable
$\lambda^{*}\ge 0$ by solving the following one-dimensional equation: $\sum_{j=1}^{n}\max\{z_j-\lambda^{*},0\}=1$.Since the left-hand side is monotonically decreasing in $\lambda^{*}$, this
equation can be efficiently solved by bisection~\citep{pardalos1990algorithm}. After obtaining
$\lambda^{*}$, the projected vector is computed coordinate-wise as $\pi_j\triangleq\max\{z_j-\lambda^{*},0\},
\quad \forall j\in[n]$.

We now justify Algorithm~\ref{alg:projection_delta}. If
$\bs{z}_{+}\triangleq\max\{\bs{z},\bs{0}\}$ satisfies
$\bs{1}^{\top}\bs{z}_{+}\le 1$, then $\bs{z}_{+}\in\Delta_n$. Moreover, for
each coordinate, the closest nonnegative point to $z_j$ is
$\max\{z_j,0\}$. Hence, under this case, the nonnegativity projection already
satisfies the $\ell_1$ constraint, and therefore $\Pi_{\Delta_n}(\bs{z})=\bs{z}_{+}$.

It remains to consider the case $\bs{1}^{\top}\bs{z}_{+}>1$. In this case, the
constraint $\bs{1}^{\top}\bs{\pi}\le 1$ must be active at the optimum; otherwise,
one could still decrease the objective by moving $\bs{\pi}$ closer to
$\bs{z}_{+}$. Therefore, the problem reduces to finding the closest point to
$\bs{z}$ over the set $\left\{
\bs{\pi}\in\mathbb{R}^{n}_{+}:
\bs{1}^{\top}\bs{\pi}=1
\right\}$. Consider the following Lagrangian
\[
\mathcal{L}(\bs{\pi},\lambda,\bs{\mu})
=
\frac{1}{2}\|\bs{\pi}-\bs{z}\|_2^2
+
\lambda(\bs{1}^{\top}\bs{\pi}-1)
-
\bs{\mu}^{\top}\bs{\pi},
\]
where $\lambda\in\mathbb{R}$ is the multiplier for the equality constraint
$\bs{1}^{\top}\bs{\pi}=1$, and $\bs{\mu}\ge\bs{0}$ is the multiplier for the
nonnegativity constraint $\bs{\pi}\ge\bs{0}$. The KKT conditions give
\[
\pi_j-z_j+\lambda-\mu_j=0,\qquad
\mu_j\pi_j=0,\qquad
\pi_j\ge 0,\qquad
\mu_j\ge 0,
\quad \forall j\in[n].
\]
From these conditions, if $\pi_j>0$, then $\mu_j=0$ and thus
$\pi_j=z_j-\lambda$. If $\pi_j=0$, then $\lambda-z_j=\mu_j\ge0$ such that $z_j-\lambda\le0$. Therefore, the optimal solution must take the form $\pi_j^{*}
=
\max\{z_j-\lambda^{*},0\},
\quad \forall j\in[n]$, for some scalar $\lambda^{*}$. Since we know that the budget constraint is active in this
case, $\lambda^{*}$ is determined by $\sum_{j=1}^{n}\max\{z_j-\lambda^{*},0\}=1$.

%% file: appendix_reservoir.tex
\section{Effective Reservoir Rounding}\label{appendix:RR}
\begin{algorithm}[t]
\caption{Reservoir Rounding Scheme}
\label{alg:Reservoir_Rounding}
\KwIn{Priors $\{\bs{\pi}_i\}_{i=1}^{h}$, budgets $\{K_i\}_{i=1}^{h}$, datasets $\{\mathcal{D}_i\}_{i=1}^{h}$}
\KwOut{Selected subset $\mathcal{S}\triangleq
\bigcup_{i=1}^{h}\bigcup_{k=1}^{K_i}d_{i,k}$}

\For{$i=1,2,\ldots,h$}{
    Set $K_i^{\min}\leftarrow \min\{\|\bs{\pi}_i\|_0,K_i\}$\;
    Draw $K_i^{\min}$ distinct samples
    $\{d_{i,1},\ldots,d_{i,K_i^{\min}}\}$ from
    $\mathcal{D}_i$ via $\bs{\pi}_i$ \emph{without replacement}\;

    \If{$K_i^{\min}<K_i$}{
        Draw $K_i-K_i^{\min}$ additional samples $\{d_{i,K_i^{\min}+1},\ldots,d_{i,K_i}\}$ \emph{uniformly at random} from $\mathcal{D}_i\setminus\{d_{i,1},\ldots,d_{i,K_i^{\min}}\}$\;
    }
}
\end{algorithm}

A straightforward rounding strategy is to \emph{independently} sample $K_i$ data
candidates from each source $\mathcal{D}_i$ according to the learned prior
$\bs{\pi}_i$. However, this operation may select the same
data candidate multiple times. Consequently, the number of distinct selected
data points from source $\mathcal{D}_i$ can be smaller than the prescribed
budget $K_i$, i.e., $\left|\bigcup_{k=1}^{K_i}\{d_{i,k}\}\right|<K_i$. When the induced utility function is monotone, such duplicate selections can
reduce the practical utility of the final selected subset, since part of the
selection budget is effectively wasted.

To address this issue, we adopt a simple without-replacement sampling scheme, termed Reservoir Rounding. As detailed in Algorithm~\ref{alg:Reservoir_Rounding}, Reservoir Rounding first computes
$K_i^{\min}\triangleq\min\{\|\bs{\pi}_i\|_0,K_i\}$, which is the maximum number
of distinct candidates that can be sampled according to the nonzero support of
$\bs{\pi}_i$. It then draws $K_i^{\min}$ distinct data candidates from
$\mathcal{D}_i$ according to $\bs{\pi}_i$ \emph{without replacement}. If
$K_i^{\min}<K_i$,  the remaining $K_i-K_i^{\min}$ candidates are \emph{uniformly} sampled
from the unselected data points in $\mathcal{D}_i$. In this way, reservoir
rounding scheme always returns exactly $K_i$ distinct candidates from each source
$\mathcal{D}_i$.  In particular, we show that this reservoir rounding scheme not only ensures
that the final selected subset contains exactly $K_i$ distinct data candidates
from each source $\mathcal{D}_i$ for all $i\in[h]$, but also guarantees that
the expected utility of the resulting subset is \emph{lossless}, that is to say,
\begin{theorem}
\label{thm:Reservoir_Rounding}
Suppose the induced utility function
$U(\bs{\theta}^{*}(\mathcal{B}))$ is monotone with respect to the selected
subset $\mathcal{B}$. For any priors
$\bs{\pi}_i\in\Delta_{n_i}$ and $\|\bs{\pi}_i\|_1=1$,$\forall i\in[h]$, let
$\mathcal{S}$ be the subset returned by Algorithm~\ref{alg:Reservoir_Rounding}.
Then $\mathcal{S}$ satisfies the following two properties: 1) $|\mathcal{S}\cap\mathcal{D}_i|=K_i,\forall i\in[h]$ and 2) $\mathbb{E}\left[
U(\bs{\theta}^{*}(\mathcal{S}))
\right]
\ge
\Psi(\bs{\pi}_1,\ldots,\bs{\pi}_h)$, where $\Psi$ denotes the probabilistic relaxation of the utility function
$U(\bs{\theta}^{*}(\mathcal{B}))$.
\end{theorem} 
\begin{remark}
The condition $\|\bs{\pi}_i\|_1=1$ is imposed only for simplicity.
When $\|\bs{\pi}_i\|_1<1$, one can naively normalize the prior as $\widetilde{\bs{\pi}}_i
\triangleq
\frac{\bs{\pi}_i}{\|\bs{\pi}_i\|_1}$. Since the induced utility function $U$ is monotone, increasing the total sampling
mass from $\|\bs{\pi}_i\|_1$ to $1$ can increase the value of the
corresponding probabilistic relaxation.
\end{remark}
\begin{proof}
We first verify the exact-budget property. For each source $\mathcal{D}_i$,
Algorithm~\ref{alg:Reservoir_Rounding} sets $K_i^{\min}\triangleq \min\{\|\bs{\pi}_i\|_0,K_i\}$. It then samples $K_i^{\min}$ distinct candidates according to $\bs{\pi}_i$
without replacement. If $K_i^{\min}<K_i$, the algorithm additionally samples
$K_i-K_i^{\min}$ candidates from the remaining unselected data points in
$\mathcal{D}_i$. Therefore, all selected candidates from $\mathcal{D}_i$ are
distinct, and $|\mathcal{S}\cap\mathcal{D}_i|=K_i, \forall i\in[h]$.

It remains to prove the utility guarantee. 
At first, we review the definition $\Psi(\bs{\pi})$. Specifically, for each source $\mathcal{D}_i$,  we independently draw $K_{i}$ data candidates $\{e_{i,1},\dots,e_{i,K_{i}}\}$ from $\mathcal{D}_i\cup\{\emptyset\}$. If we let $\mathcal{S}_{ind}
\triangleq
\bigcup_{i=1}^{h}\bigcup_{k=1}^{K_i}e_{i,k}$,
we have $\Psi(\bs{\pi})
=
\mathbb{E}_{\mathcal{S}}\big[U(\bs{\theta}^{*}(\mathcal{S}_{ind}))\big]$. Then, we denote by $\mathcal{A}_i$ the
first-stage subset sampled from $\mathcal{D}_i$ without replacement, namely, $\mathcal{A}_i
\triangleq
\{d_{i,1},\ldots,d_{i,K_i^{\min}}\}$.

Since the second sampling step only adds extra candidates to
$\bigcup_{i=1}^{h}\mathcal{A}_i$, we naturally have that $\mathbb{E}\big[U(\bs{\theta}^{*}(\mathcal{S}))\big]
\ge
\mathbb{E}\left[U(\bs{\theta}^{*}(\bigcup_{i=1}^{h}\mathcal{A}_i))
\right]$, where $\mathcal{S}$ is the output of Algorithm~\ref{alg:Reservoir_Rounding}.
 Thus, it suffices to show that $\mathbb{E}\left[U(\bs{\theta}^{*}(\bigcup_{i=1}^{h}\mathcal{A}_i))\right]
\ge\mathbb{E}_{\mathcal{S}}\big[U(\bs{\theta}^{*}(\mathcal{S}_{ind}))\big]=\Psi(\bs{\pi})$.

Specifically, we will first show that, for any subset
$\widetilde{\mathcal{S}}$ independent of the randomness from source
$\mathcal{D}_i$, the following inequality holds
\begin{equation}\label{add+++}
    \mathbb{E}_{\mathcal{A}_i}
\left[
U\left(\bs{\theta}^{*}(\widetilde{\mathcal{S}}\cup\mathcal{A}_i)\right)
\right]
\ge
\mathbb{E}
\left[
U\left(\bs{\theta}^{*}(
\widetilde{\mathcal{S}}
\cup
\big(\bigcup_{k=1}^{K_i}e_{i,k}\big)
)\right)\right],
\end{equation}where $\{e_{i,1},\dots,e_{i,K_{i}}\}$ are drawn independently according to $\bs{\pi}_{i}$.

If $\|\bs{\pi}_i\|_0<K_i$, then the first-stage without-replacement sampling
selects all data candidates with nonzero probability under $\bs{\pi}_i$.
Therefore,  $\bigcup_{k=1}^{K_i}e_{i,k}\subseteq\mathcal{A}_i$. By monotonicity of $U$, the above
inequality Eq.\eqref{add+++} holds.

It remains to consider the case $\|\bs{\pi}_i\|_0\ge K_i$. Let $\mathcal{A}_i=\{p^{(i)}_{N_1},\ldots,p^{(i)}_{N_{K_{i}}}\}$ where $p^{(i)}_{N_k}$ denotes the data pair sampled without replacement from
source $\mathcal{D}_i$. We then prove Eq.\eqref{add+++}.

It is worth noting that when conditioned on the first
$K_i-1$ selected indices $N_1,\ldots,N_{K_i-1}$, 
the last without-replacement draw selects an unchosen candidate $m$ with
probability $\frac{\pi_{i,m}}{1-P},
\quad
m\notin\{N_1,\ldots,N_{K_i-1}\}$, where  $P\triangleq \sum_{r=1}^{K_i-1}\pi_{i,N_r}$.

Hence,
\begin{equation*}
  \begin{aligned}
&\mathbb{E}\!\left[
U\left(\bs{\theta}^{*}(
\widetilde{\mathcal{S}}
\cup
\{v_i^{N_1},\ldots,v_i^{N_{K_i-1}}\}
\cup
\{v_i^{N_{K_i}}\}
)\right)
\middle|
N_1,\ldots,N_{K_i-1}
\right]  \\
&=
\sum_{m\notin\{N_1,\ldots,N_{K_i-1}\}}
\frac{\pi_{i,m}}{1-P}
U\left(\bs{\theta}^{*}(
\widetilde{\mathcal{S}}
\cup
\{v_i^{N_1},\ldots,v_i^{N_{K_i-1}}\}
\cup
\{v_i^m\}
)\right)\\
&=\sum_{m\notin\{N_1,\ldots,N_{K_i-1}\}}
(\pi_{i,m}+\frac{P\cdot\pi_{i,m}}{1-P})\cdot
U\left(\bs{\theta}^{*}(
\widetilde{\mathcal{S}}
\cup
\{v_i^{N_1},\ldots,v_i^{N_{K_i-1}}\}
\cup
\{v_i^m\}
)\right)\\
&\ge\sum_{m\notin\{N_1,\ldots,N_{K_i-1}\}}
\pi_{i,m}\cdot
U\left(\bs{\theta}^{*}(
\widetilde{\mathcal{S}}
\cup
\{v_i^{N_1},\ldots,v_i^{N_{K_i-1}}\}
\cup
\{v_i^m\}
)\right)\\&+\frac{P\cdot\sum_{m\notin\{N_1,\ldots,N_{K_i-1}\}}\pi_{i,m}}{1-P}\cdot
U\left(\bs{\theta}^{*}(
\widetilde{\mathcal{S}}
\cup
\{v_i^{N_1},\ldots,v_i^{N_{K_i-1}}\}
)\right)\\
&=\sum_{m\notin\{N_1,\ldots,N_{K_i-1}\}}
\pi_{i,m}\cdot
U\left(\bs{\theta}^{*}(
\widetilde{\mathcal{S}}
\cup
\{v_i^{N_1},\ldots,v_i^{N_{K_i-1}}\}
\cup
\{v_i^m\}
)\right)+P\cdot U\left(\bs{\theta}^{*}(
\widetilde{\mathcal{S}}
\cup
\{v_i^{N_1},\ldots,v_i^{N_{K_i-1}}\}
)\right)\\
&=\sum_{m\notin\{N_1,\ldots,N_{K_i-1}\}}
\pi_{i,m}\cdot
U\left(\bs{\theta}^{*}(
\widetilde{\mathcal{S}}
\cup
\{v_i^{N_1},\ldots,v_i^{N_{K_i-1}}\}
\cup
\{v_i^m\}
)\right)\\&+\sum_{m\in\{N_1,\ldots,N_{K_i-1}\}}\pi_{i,m}\cdot U\left(\bs{\theta}^{*}(
\widetilde{\mathcal{S}}
\cup
\{v_i^{N_1},\ldots,v_i^{N_{K_i-1}}\}\cup
\{v_i^m\}
)\right)\\
&=\mathbb{E}_{e_{i,K_i}}\left[U\left(\bs{\theta}^{*}(
\widetilde{\mathcal{S}}
\cup
\{v_i^{N_1},\ldots,v_i^{N_{K_i-1}}\}\cup
\{e_{i,K_i}\}\right])\right),
\end{aligned}  
\end{equation*} where the fourth equality follows from the fact that when $m\in\{N_1,\ldots,N_{K_i-1}\}$, we have
\begin{equation*}
    U\left(\bs{\theta}^{*}(
\widetilde{\mathcal{S}}
\cup
\{v_i^{N_1},\ldots,v_i^{N_{K_i-1}}\cup
\{v_i^m\}\}
)\right)=U\left(\bs{\theta}^{*}(
\widetilde{\mathcal{S}}
\cup
\{v_i^{N_1},\ldots,v_i^{N_{K_i-1}}\}
)\right).
\end{equation*}

This shows that the last without-replacement draw can be replaced by an
independent draw without increasing the expected function value. Repeating the
same argument for the \((K_i-1)\)-th draw, the \((K_i-2)\)-th draw, and so on,
we obtain the result of Eq.\eqref{add+++}.

Finally, applying the above source-wise comparison Eq.\eqref{add+++} sequentially over
$i=h,h-1,\ldots,1$ and using the tower property of conditional expectation, we
obtain 
\begin{equation*}
    \mathbb{E}\left[U(\bs{\theta}^{*}(\bigcup_{i=1}^{h}\mathcal{A}_i))\right]
\ge\mathbb{E}_{\mathcal{S}}\big[U(\bs{\theta}^{*}(\mathcal{S}_{ind}))\big]=\Psi(\bs{\pi}).
\end{equation*}

Combining this with the fact $\mathbb{E}\big[U(\bs{\theta}^{*}(\mathcal{S}))\big]
\ge
\mathbb{E}\left[U(\bs{\theta}^{*}(\bigcup_{i=1}^{h}\mathcal{A}_i))
\right]$, we have $\mathbb{E}\!\left[
U(\bs{\theta}^{*}(\mathcal{S}))
\right]
\ge
\mathbb{E}_{\mathcal{S}}\big[U(\bs{\theta}^{*}(\mathcal{S}_{ind}))\big]=\Psi(\bs{\pi})$.
\end{proof}

%% file: appendix_ds_cifar100.tex
\section{Data Summarization on Cifar-100 Dataset}\label{appendix:ds_cifar100}

\begin{figure}[ht]
     \centering

     \begin{subfigure}[b]{0.48\textwidth}
         \centering
         \includegraphics[width=\textwidth]{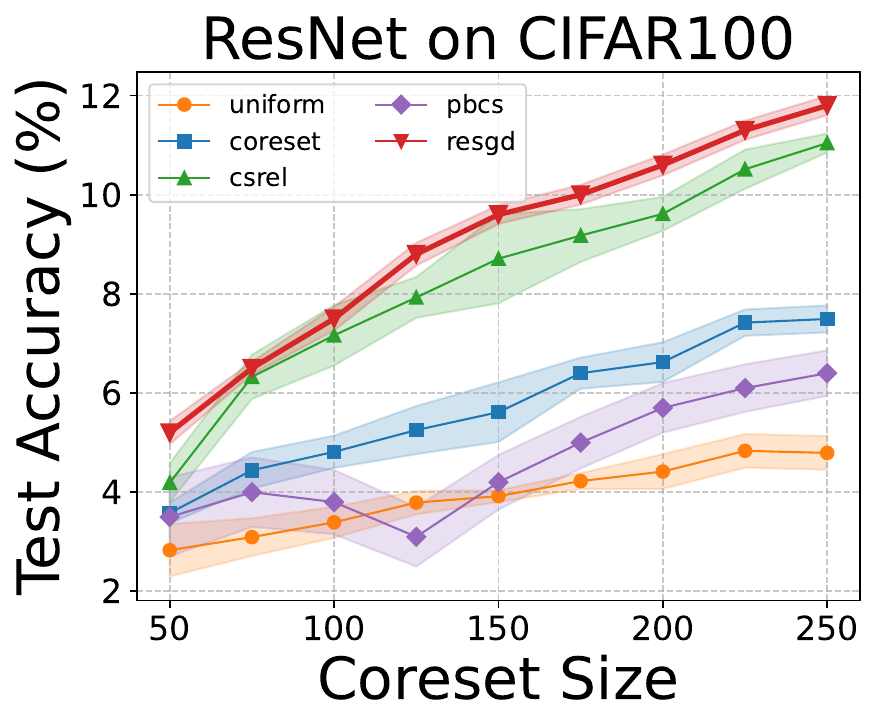}
         \caption{}
         \label{fig:cifar100_main}
     \end{subfigure}
     \hfill
     \begin{subfigure}[b]{0.48\textwidth}
         \centering
         \includegraphics[width=\textwidth]{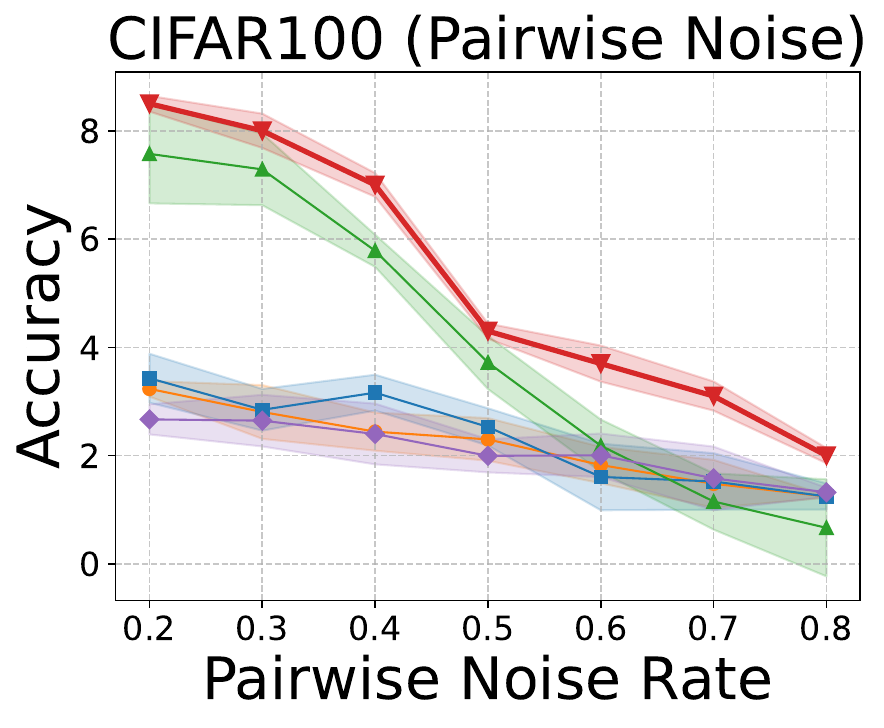}
         \caption{}
         \label{fig:cifar100_pairwise}
     \end{subfigure}
     \vspace{1em}

     \begin{subfigure}[b]{0.48\textwidth}
         \centering
         \includegraphics[width=\textwidth]{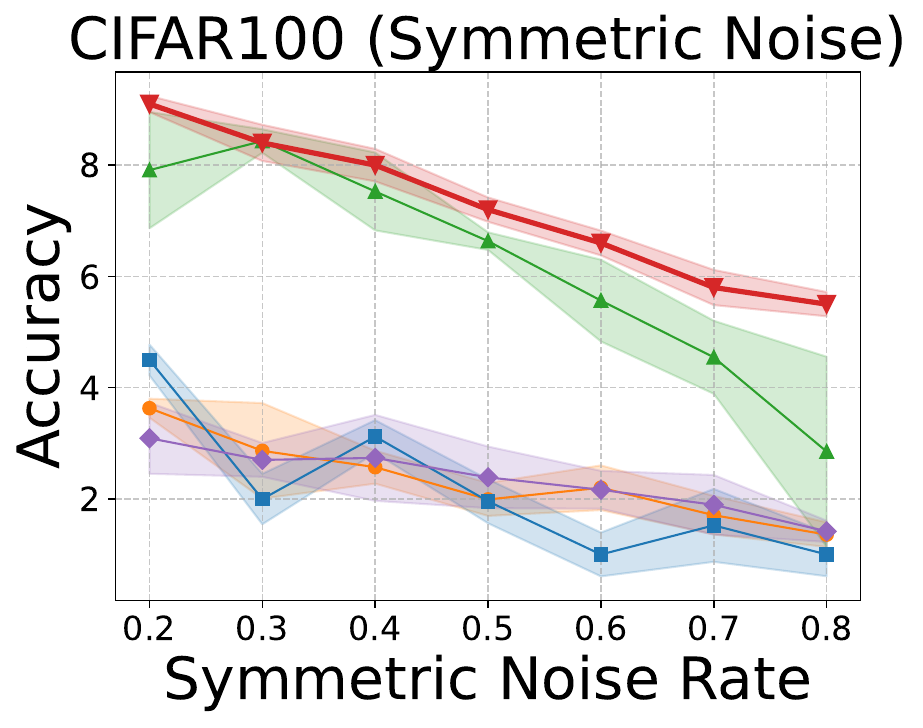}
         \caption{}
         \label{fig:cifar100_symmetric}
     \end{subfigure}
     \hfill
     \begin{subfigure}[b]{0.48\textwidth}
         \centering
         \includegraphics[width=\textwidth]{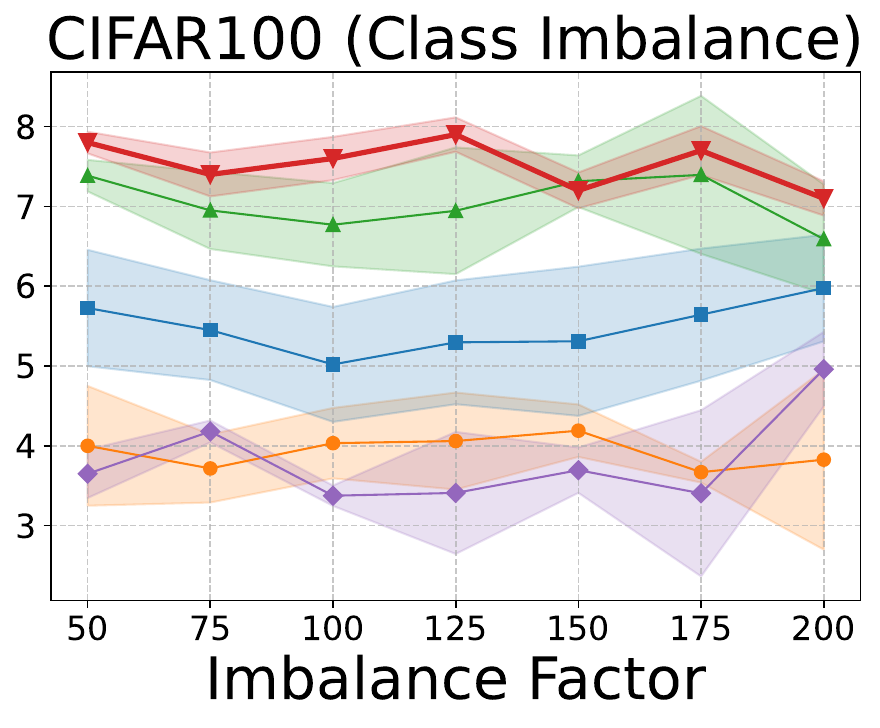}
         \caption{}
         \label{fig:cifar100_imbalance}
     \end{subfigure}
     \caption{Performance comparison between \texttt{ReSGD} and baseline methods on data summarization. (a) Test accuracy across different coreset sizes for ResNet on CIFAR100. (b) and (c) evaluate the robustness of each method against varying rates of pairwise noise and symmetric noise on CIFAR100, respectively. (d) illustrates the performance under different class imbalance factors on CIFAR100. Shaded areas denote the variance across multiple runs. Overall, the results demonstrate the superior performance and stability of \texttt{ReSGD} under various data corruption and distribution shift scenarios.}
\end{figure}

%% file: appendix_cs_ex_cl.tex
\section{Coreset Selection for Existing Continual Learning Methods}\label{appendix:cs_ex_cl}
We instantiate our method by replacing the reservoir sampling component in ER~\cite{riemer2018learning} with \texttt{ReSGD} and CS-Rel~\cite{tong2025coreset}. Following~\cite{tong2025coreset}, we evaluate all methods on Split CIFAR-100 and Split Tiny-ImageNet under memory budgets of 200 and 500. Both benchmarks are divided into 10 class-incremental tasks, and all methods use a ResNet-18 backbone~\cite{he2016deep}. We compare against standard rehearsal-based baselines, including ER, A-GEM~\cite{chaudhry2018efficient}, and FDR~\cite{benjamin2018measuring}. Results are summarized in Table~\ref{table:ex-cl-result}. For Split CIFAR-100, the 100 classes are evenly split into 10 tasks with 10 classes per task. For Split Tiny-ImageNet, the 200 classes are evenly split into 10 tasks with 20 classes per task. In both cases, all samples from the current task are used for training. We adopt the same data augmentation pipeline as Mammoth~\cite{buzzega2020dark}, including random crop, random horizontal flip, and normalization. We also use the same ResNet-18 implementation provided by Mammoth for both datasets. For optimization, we use SGD with batch size 32 on both benchmarks. The learning rate is set to $2\times10^{-2}$ for Split CIFAR-100 and $3\times10^{-2}$ for Split Tiny-ImageNet. The cross-entropy loss factor is fixed to 1.0. The number of training epochs is kept consistent with the corresponding baselines.

\begin{table*}[h]
\centering
\caption{Average accuracy (\%) on Split CIFAR-100 and Split Tiny-ImageNet. Results are averaged over 10 incremental tasks with memory buffers of 200 and 500 samples. We use ResNet-18 as the backbone for all experiments. Bold indicates the best performance.}
\scriptsize{
\resizebox{\linewidth}{!}{
\setlength\tabcolsep{3.pt}
\renewcommand\arraystretch{1}
\begin{tabular}{r||cc|cc}
\hline\thickhline
\rowcolor{gray!20}
\textbf{Methods} 
& \multicolumn{2}{c|}{\textbf{Split CIFAR-100}} 
& \multicolumn{2}{c}{\textbf{Split Tiny ImageNet}} \\
\cline{2-5}\rowcolor{gray!20}
& \textbf{200} & \textbf{500} & \textbf{200} & \textbf{500} \\
\hline\hline
A-GEM    & 9.40 $\pm$ 0.05  & 9.42 $\pm$ 0.08  & 8.07 $\pm$ 0.08 & 8.06 $\pm$ 0.04 \\
\rowcolor{gray!10}
ER       & 14.18 $\pm$ 0.45 & 21.08 $\pm$ 0.16 & 8.49 $\pm$ 0.16 & 9.99 $\pm$ 0.29 \\
FDR      & 15.32 $\pm$ 0.73 & 22.83 $\pm$ 0.73 & 8.70$\pm$ 0.19 & 10.54 $\pm$ 0.21 \\
\rowcolor{gray!10}
CSReL-ER & 15.35 $\pm$ 0.73 & 22.65 $\pm$ 0.81 & 8.66 $\pm$ 0.06 & 10.44 $\pm$ 0.17 \\
\hline
\rowcolor[HTML]{D7F6FF}
\texttt{ReSGD-ER}                &\textbf{15.38} $\pm$ \textbf{0.25} & \textbf{23.21} $\pm$ \textbf{0.22} & \textbf{8.87} $\pm$ \textbf{0.05} &\textbf{10.67} $\pm$ \textbf{0.12} \\
\hline
\end{tabular}}}
\label{table:ex-cl-result}
\end{table*}

%% file: appendix_adversarial.tex
\section{Experiment with Adversarial Initialization}\label{appendix:adversarial}
In this section, we detail the empirical setup designed to rigorously evaluate the error-correction capabilities and robustness of the \texttt{ReSGD} algorithm. The experiment specifically investigates the algorithm's capacity to filter corrupted labels and recover a high-quality coreset when subjected to an extreme, adversarial initialization state.
\\
\\
\textbf{Dataset Partitioning and Asymmetric Label Corruption.} 
The experiment utilizes a subset of the MNIST dataset, partitioned into two disjoint subsets: Holdout Set $\mathcal{D}_{holdout}$ consists of $10,000$ pristine samples and serves as the reference distribution; Training Pool $\mathcal{D}_{train}$ consists of $10,000$ separate samples, acting as the active selection pool from which the coreset is sampled. To simulate severe, systematic data corruption, we introduce asymmetric label noise. We uniformly sample a subset $\mathcal{C} \subset \mathcal{D}_{train}$ containing $1,000$ instances. For each instance $(x_i, y_i) \in \mathcal{C}$, the ground-truth label $y_i$ is deterministically distorted according to the following mapping:$$y_{i}^{noisy} = (y_{i}^{clean} + 2) \pmod{10}$$The indices of the corrupted subset $\mathcal{C}$ are recorded for downstream evaluation of the selection purity.
\\\\
\textbf{Reference Model Pre-training.}
Prior to the \texttt{ReSGD} sampling loop, a baseline reference model, denoted as $\theta^*(\mathcal{D})$, is trained to convergence exclusively on the clean dataset $\mathcal{D}_{holdout}$. During the subsequent \texttt{ReSGD} iterations, the parameters of $\theta^*(\mathcal{D})$ are strictly frozen. This model acts as an oracle to compute the reference loss $\ell(x_i, y_i; \theta^*(\mathcal{D}))$ and does not undergo further parameter updates.
\\\\
\textbf{Adversarial Probability Initialization and $\epsilon$-Exploration.}
To stress-test the \texttt{ReSGD} framework, we initialize the sampling probability vector $\boldsymbol{m}^{(1)} \in \mathbb{R}^{10000}$ in a highly adversarial manner. We randomly select a further subset $\mathcal{A} \subset \mathcal{C}$, where $|\mathcal{A}| = 200$. To prevent algorithmic stagnation (where clean samples would never be selected or evaluated if assigned a probability of $0$), we introduce an $\epsilon$-exploration constant. The initial probability mass is distributed as follows:$$m_i^{(1)} = 
\begin{cases} 
\frac{0.90}{|\mathcal{A}|}, & \text{if } i \in \mathcal{A} \\
\frac{0.10}{|\mathcal{D}_{train}| - |\mathcal{A}|}, & \text{if } i \notin \mathcal{A} 
\end{cases}$$This initialization places $90\%$ of the sampling probability strictly on corrupted samples, forcing the algorithm to recover from a nearly adversarial starting state while maintaining the mathematical viability of gradient computation across the entire pool.
\\\\
\textbf{Evaluation Metrics}
To empirically validate the selection efficacy of the proposed method, we monitor three primary metrics:
\begin{itemize}
    \item Distribution Shift: We plot the trajectory of the total probability mass assigned to the corrupted subset $\mathcal{C}$ over $T$ iterations: $\sum_{i \in \mathcal{C}} m_i^{(t)}$. We hypothesize a rapid decay from the initial $0.90$ towards $0$.
    \item Coreset Purity: At the final iteration $T$, we extract the Top-$K$ samples with the highest probability values in $\boldsymbol{m}^{(T+1)}$. The purity is defined as the empirical fraction of clean samples within this Top-$K$ subset.
    \item Downstream Utility: We train a new model from scratch using the final selected Top-$K$ coreset and report its top-1 classification accuracy on the standard MNIST test set, establishing a comparison against a uniform random sampling baseline.
\end{itemize} 

\textbf{Ablation Results.}
The result is illustrated in Figures~\ref{fig:noisy_curves} and~\ref{fig:noisy_curves_cifar}.
\begin{figure}[htbp]
    \centering
    \begin{tabular}{ccc}
        \includegraphics[width=0.32\textwidth]{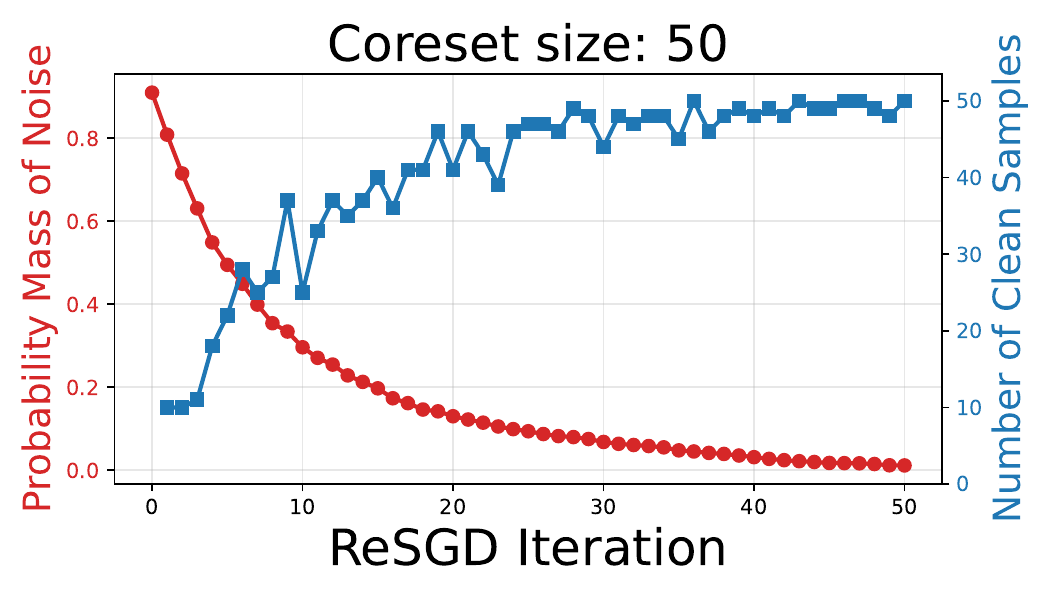} &
        \includegraphics[width=0.32\textwidth]{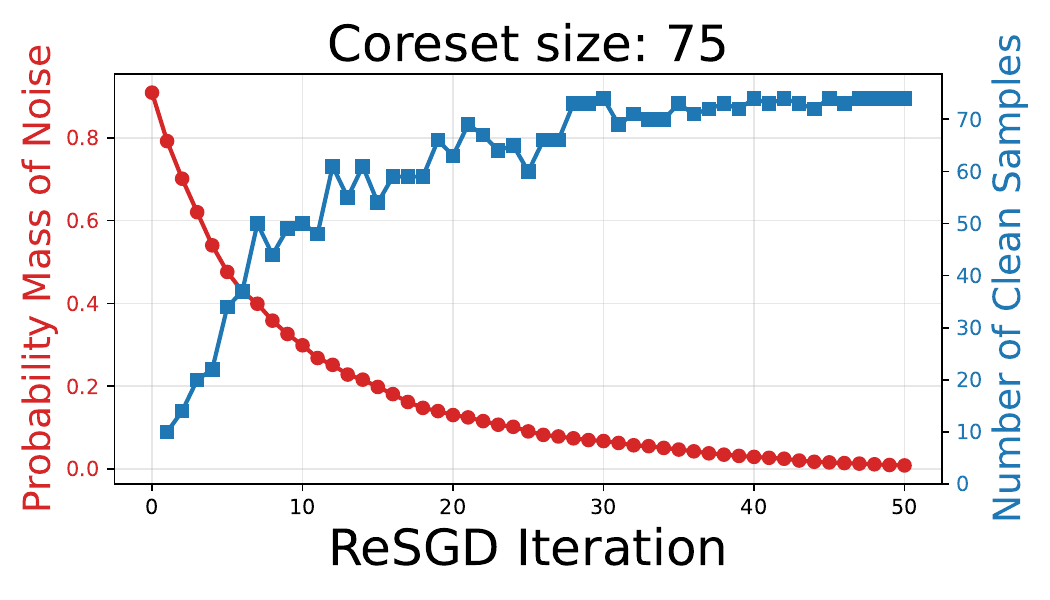} &
        \includegraphics[width=0.32\textwidth]{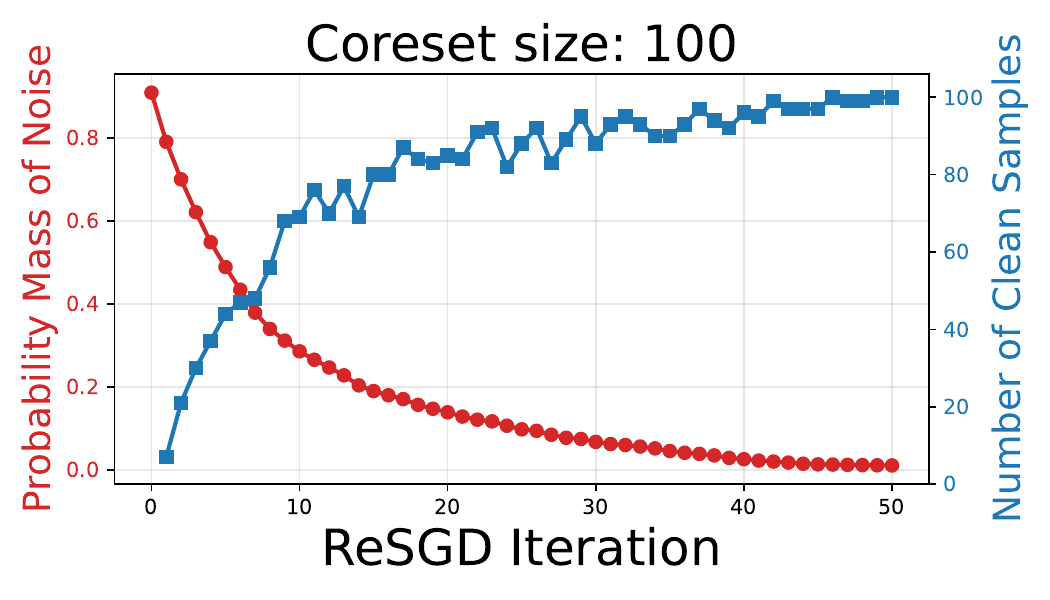} \\
        \includegraphics[width=0.32\textwidth]{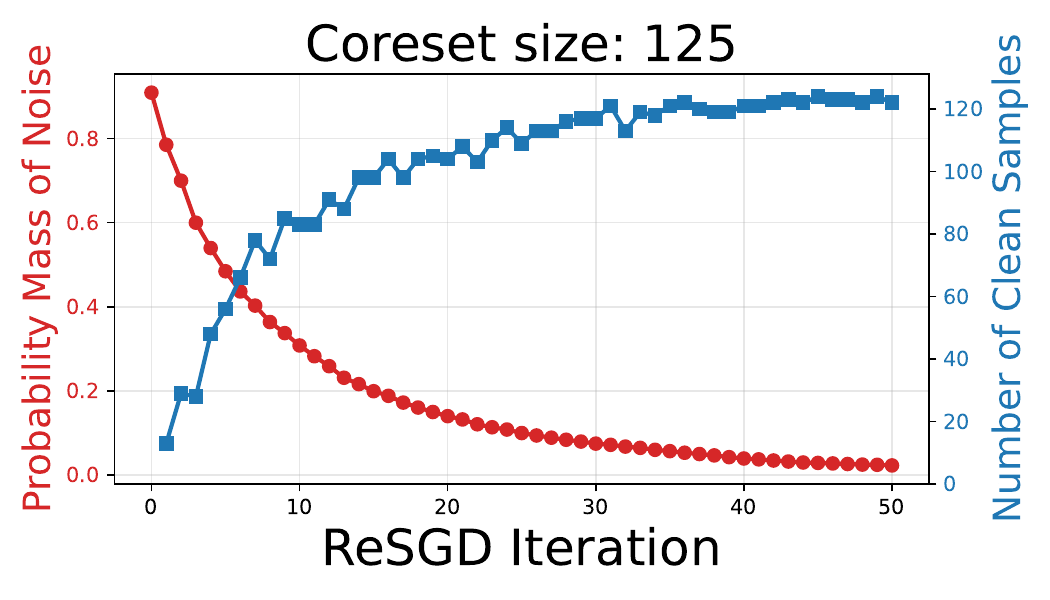} &
        \includegraphics[width=0.32\textwidth]{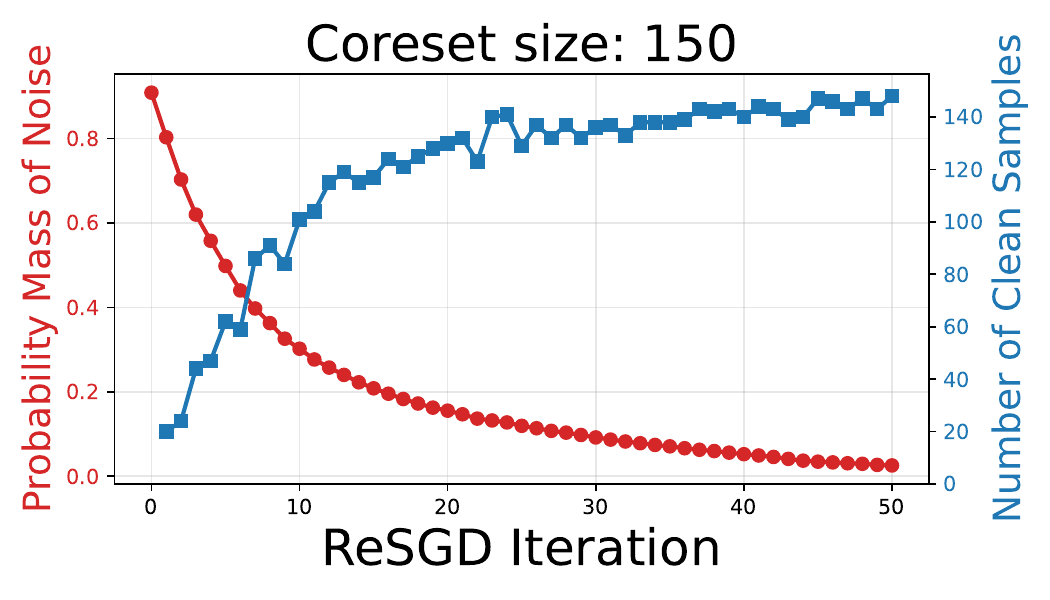} &
        \includegraphics[width=0.32\textwidth]{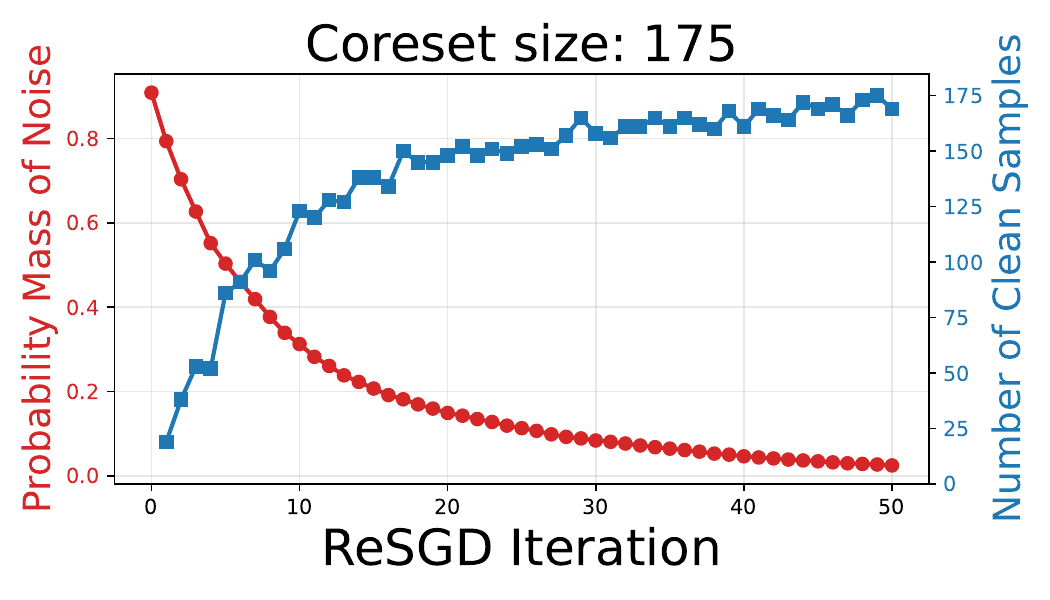} \\
        \includegraphics[width=0.32\textwidth]{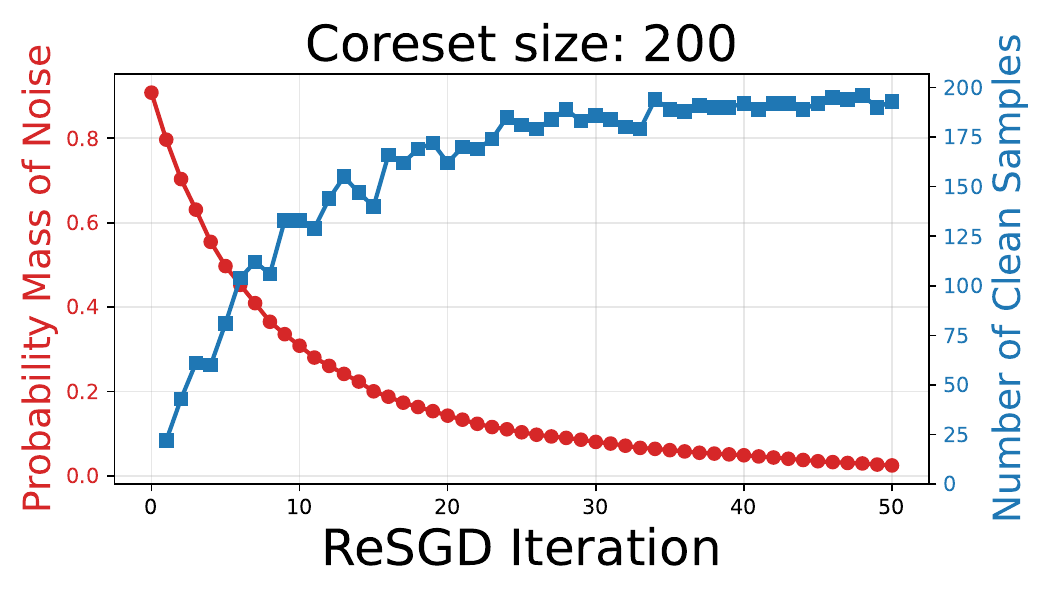} &
        \includegraphics[width=0.32\textwidth]{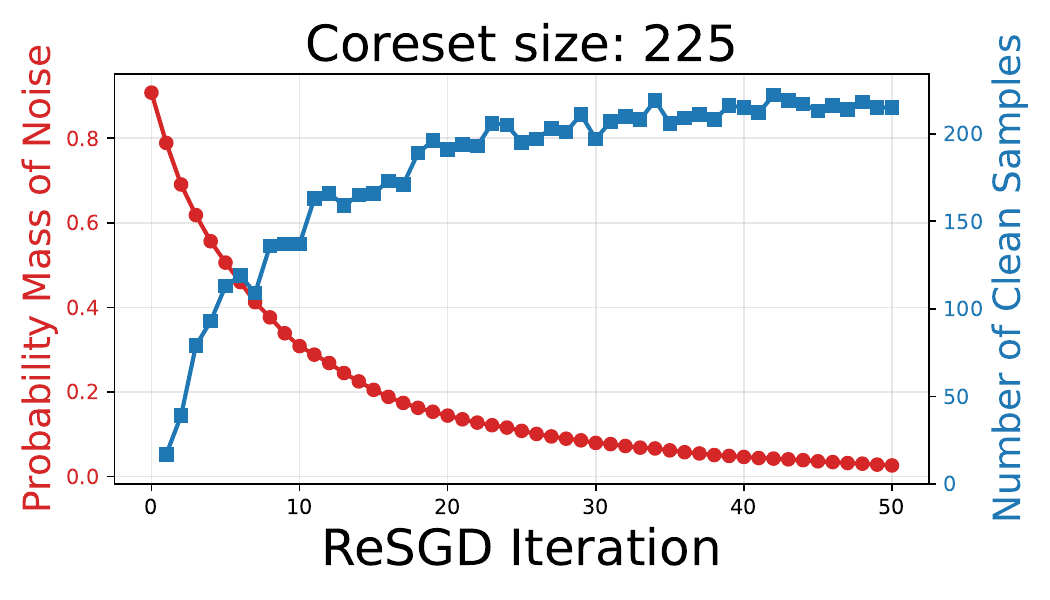} &
        \includegraphics[width=0.32\textwidth]{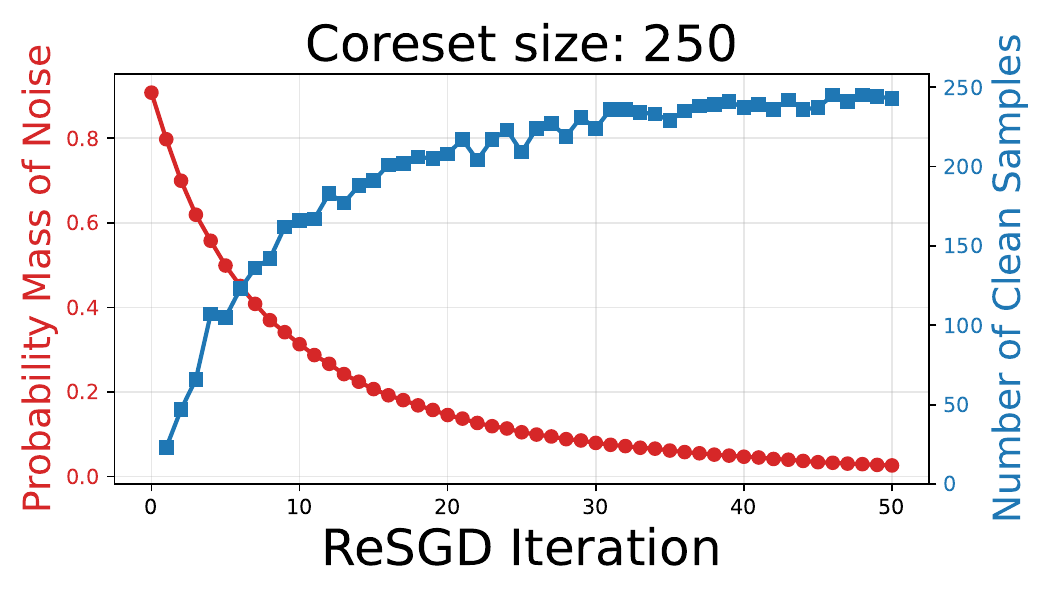}
    \end{tabular}
    \caption{Performance on MNIST under Asymmetric Noise. We track the evolution of the sampling probability mass on corrupted labels (red line) and the coreset purity (blue line) across 50 iterations for various coreset sizes $K \in \{50, \dots, 250\}$. Despite an adversarial start where 90\% of the initial mass is concentrated on noisy samples, \texttt{ReSGD} consistently reallocates probability to clean samples, achieving near-perfect purity across all tested coreset scales.}
    \label{fig:noisy_curves}
\end{figure}

\begin{figure}[htbp]
    \centering
    \begin{tabular}{ccc}
        \includegraphics[width=0.32\textwidth]{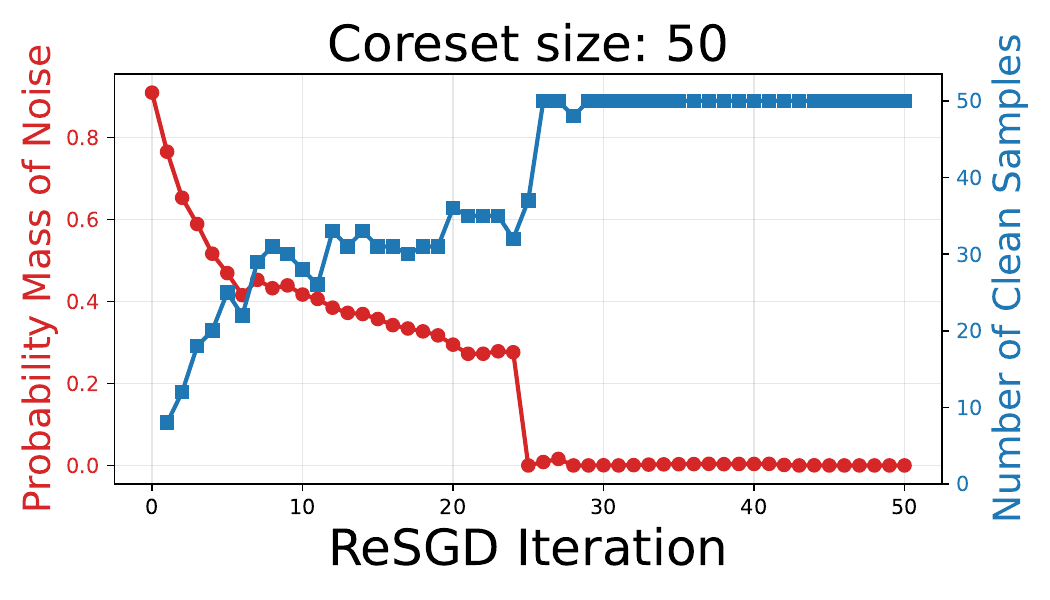} &
        \includegraphics[width=0.32\textwidth]{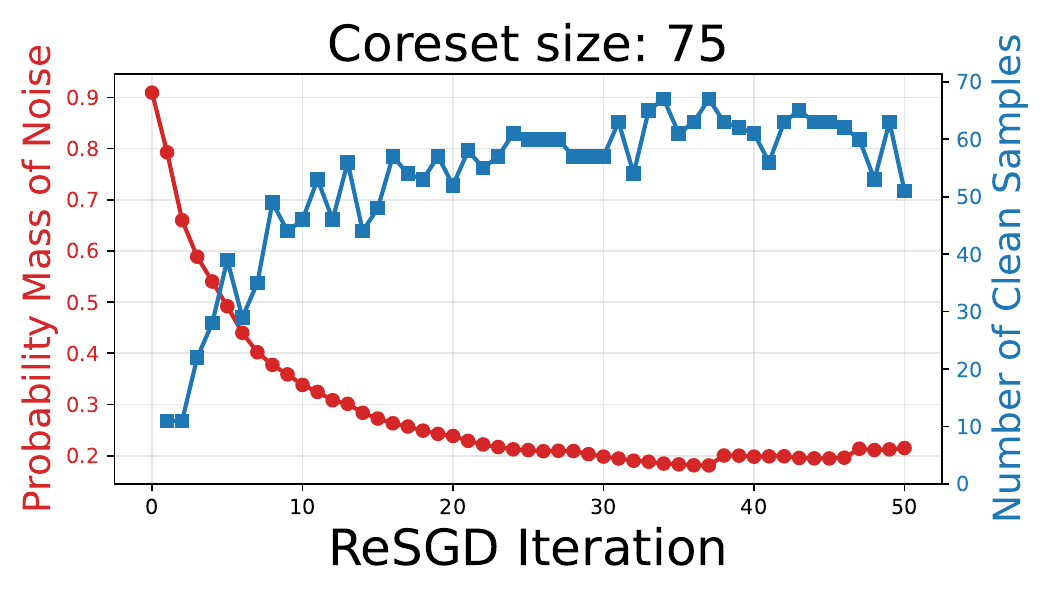} &
        \includegraphics[width=0.32\textwidth]{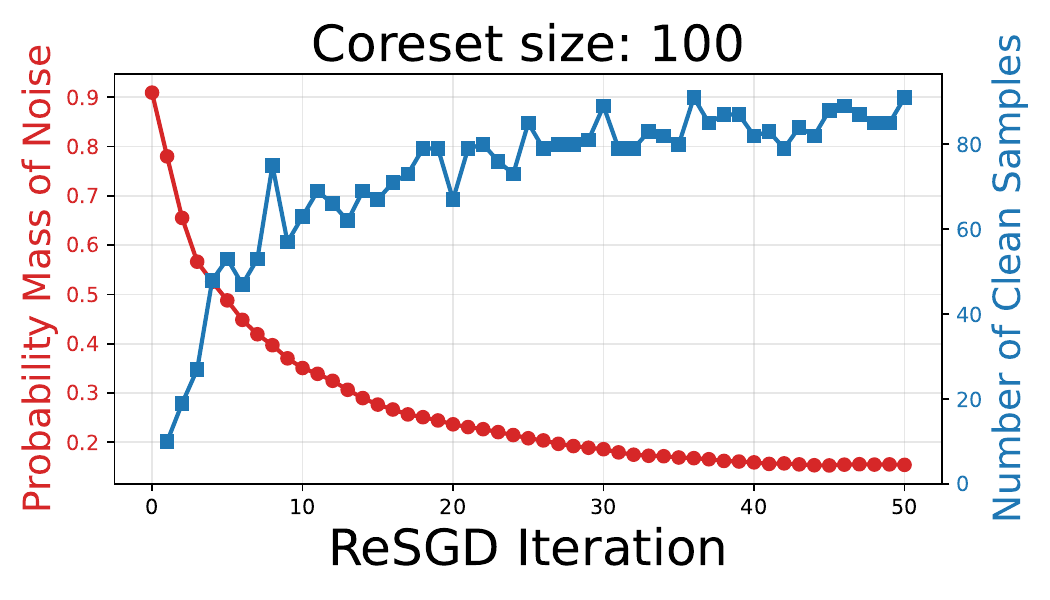} \\
        \includegraphics[width=0.32\textwidth]{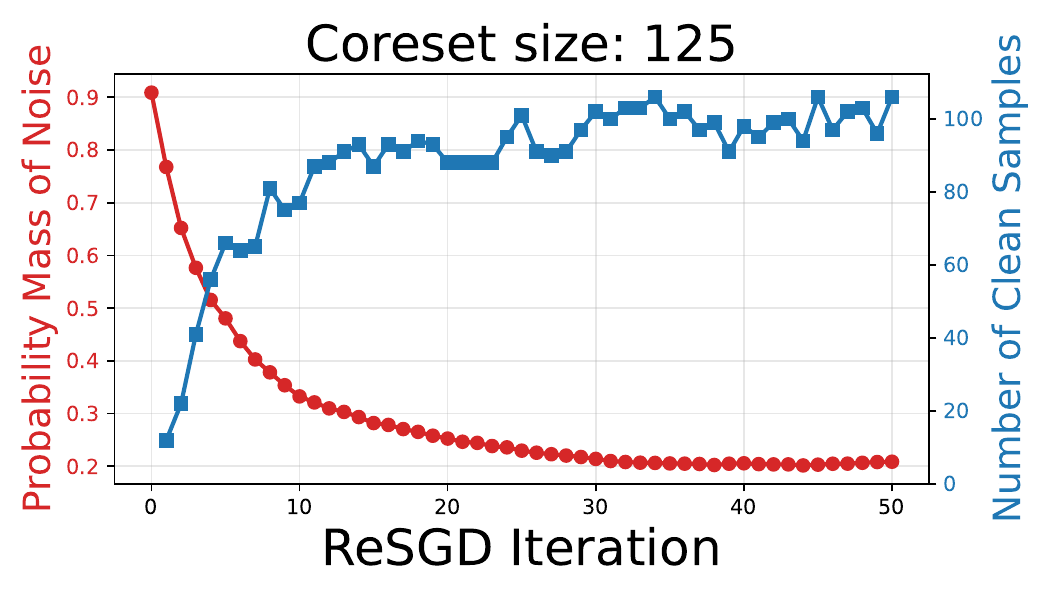} &
        \includegraphics[width=0.32\textwidth]{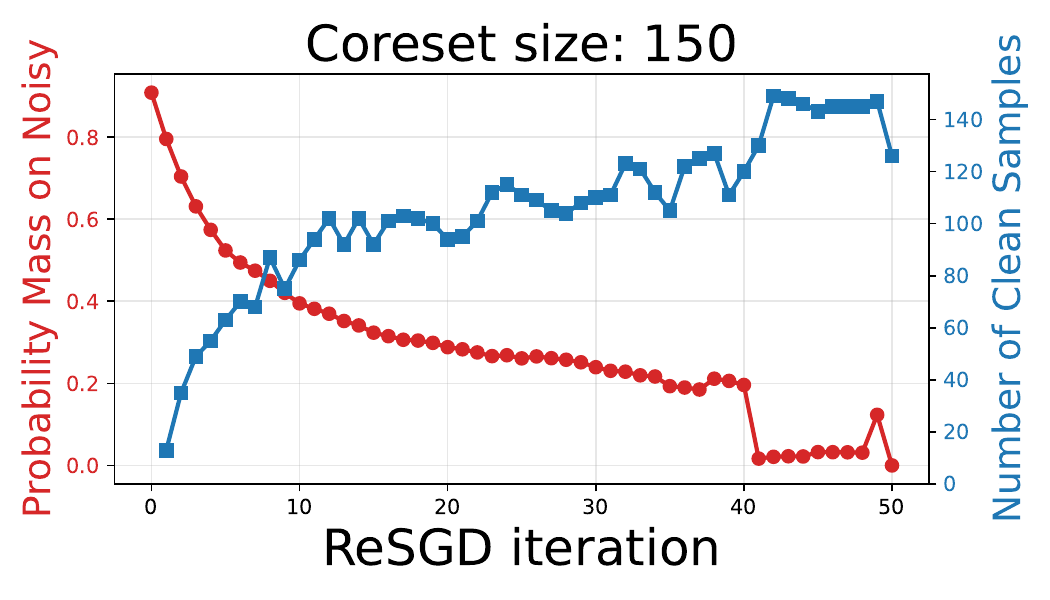} &
        \includegraphics[width=0.32\textwidth]{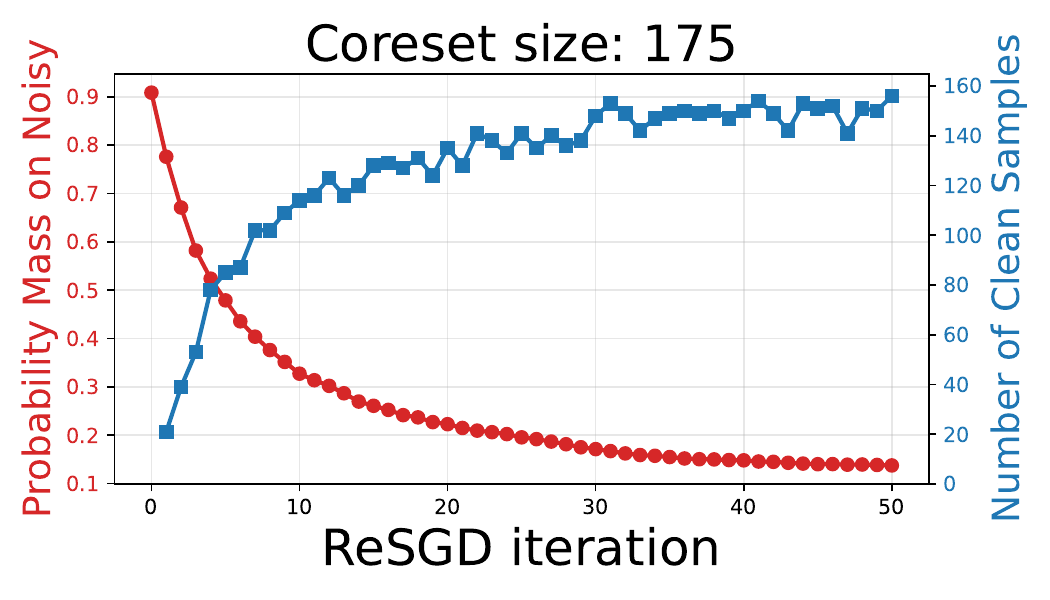} \\
        \includegraphics[width=0.32\textwidth]{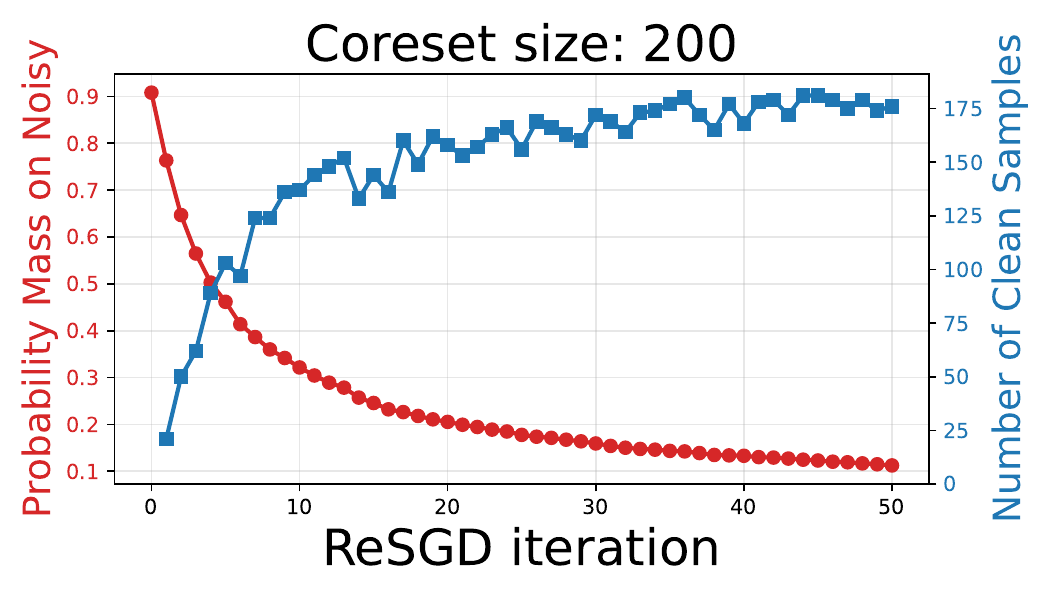} &
        \includegraphics[width=0.32\textwidth]{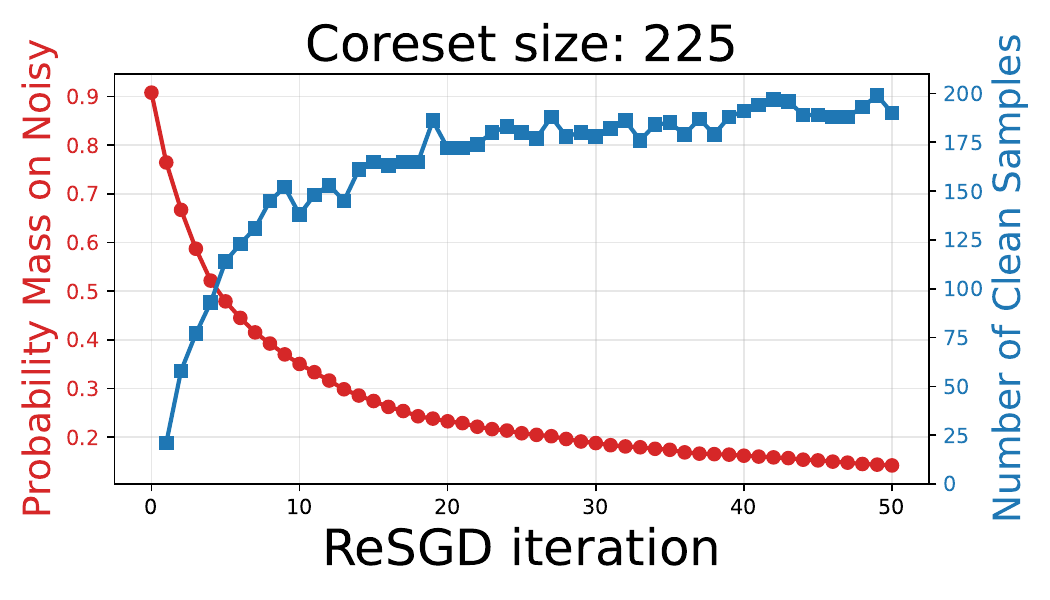} &
        \includegraphics[width=0.32\textwidth]{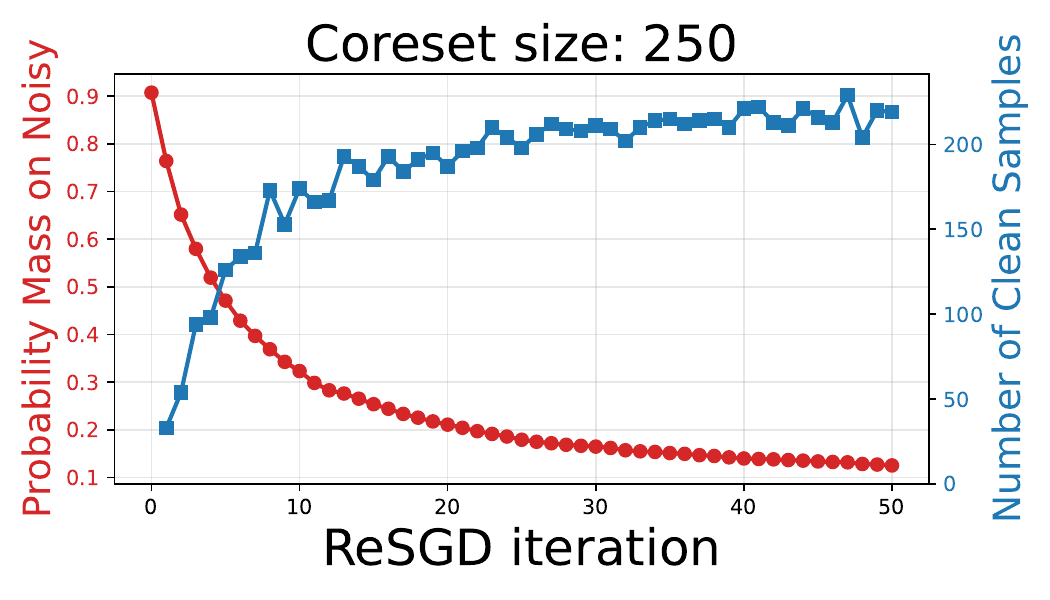}
    \end{tabular}
    \caption{Performance on CIFAR-10 under Asymmetric Noise. Similar to the MNIST experiments, \texttt{ReSGD} demonstrates strong recovery capabilities on the more complex CIFAR-10 dataset. The rapid decay of the noise-associated probability mass (red) and the corresponding climb in the number of clean samples selected (blue) highlight the algorithm's robustness to systematic label distortion across different target coreset budgets.}
    \label{fig:noisy_curves_cifar}
\end{figure}

%% file: appendix_data_corruption.tex
\section{Data Corruption}\label{appendix:data_coruption}
To demonstrate the robustness of selection under noisy conditions, we test data summarization performance under the label noise case on the MNIST, CIFAR-10, and CIFAR-100 datasets. For experiments with label noise and class imbalance, the coreset size is set to 200 for MNIST and CIFAR-10. Here is the detailed protocol:
\begin{itemize}
    \item \textbf{Label Noise}: We experiment with two types of label corruption: Symmetric Noise and Pairwise Noise. For symmetric noise, the true label of a sample is uniformly flipped to any of the other $C-1$ classes with a probability $p$. For pairwise noise (pairflip), the true label is flipped to an adjacent class with a probability $p$. We vary the noise ratio $p$ from 0.2 to 0.8 to test the methods under varying degrees of corruption.
    \item \textbf{Class Imbalance}: We construct a long-tailed (imbalanced) dataset by exponentially decaying the number of samples across classes. Following standard practices, we define the imbalance factor as $\mu = \frac{n_{max}}{n_{min}}$, where $n_{max}$ and $n_{min}$ represent the number of training samples in the largest (head) and smallest (tail) classes, respectively. Specifically, the number of samples retained for the $i$-th class is calculated as $n_i = n_{max} \cdot \sigma^i$, where the decay coefficient is defined as $\sigma = \mu^{-\frac{1}{C-1}}$. In our experiments, we vary the imbalance factor $\mu$ from 50 to 200.
\end{itemize}

We analyze a coreset of 200 samples selected within the noise ratio range of 0.0 to 0.4 under pairwise noise, with the results illustrated in Figure~\ref{fig:appendix_ds_noise}.

\begin{figure}[h]
    \centering
    \includegraphics[width=\linewidth]{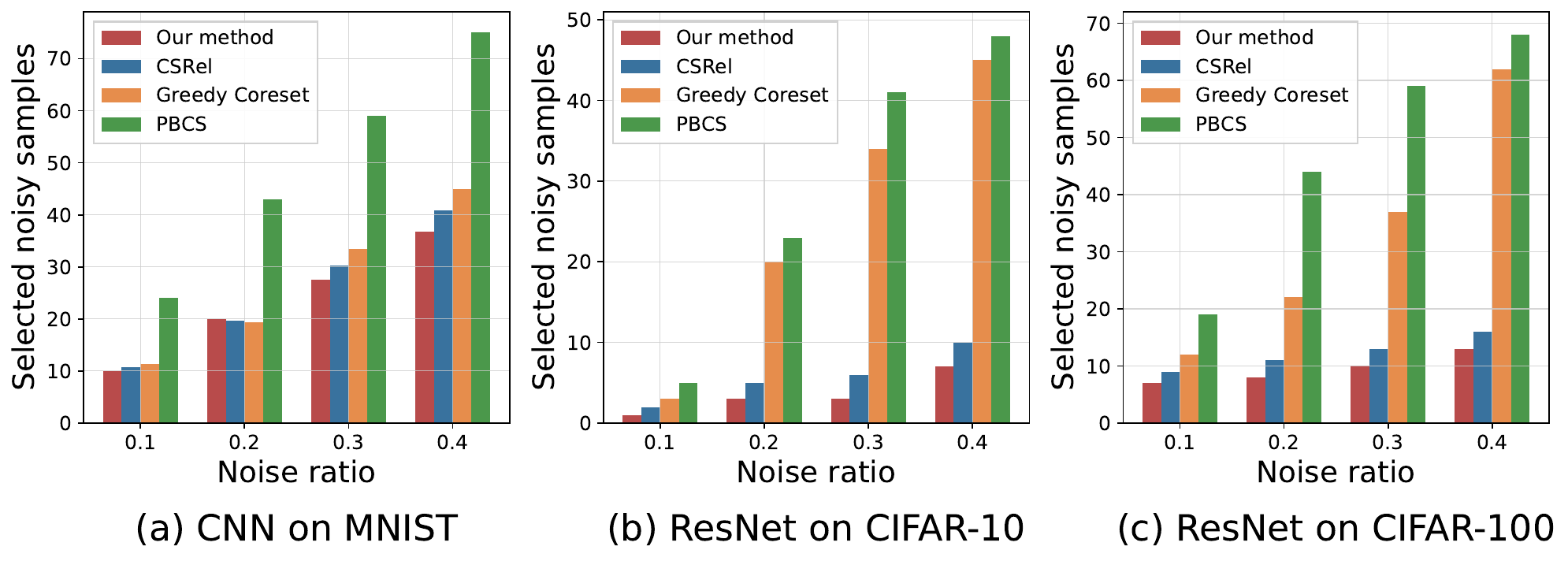}
    \caption{Comparison of the number of selected noisy samples across different noise ratios. We evaluate our method against baseline approaches (CSRel~\cite{tong2025coreset}, Greedy Coreset~\cite{borsos2020coresets}, and PBCS~\cite{zhou2022probabilistic}) on three datasets: (a) MNIST using CNN, (b) CIFAR-10 using ResNet, and (c) CIFAR-100 using ResNet. \texttt{ReSGD} consistently selects the fewest noisy samples across all noise levels, demonstrating its superior robustness and ability to distinguish clean data from pairwise noisy samples.}
    \label{fig:appendix_ds_noise}
\end{figure}

%% file: appendix_experiment_details.tex
\section{Experimental Details}\label{appendix:experimental details}
We note that different baselines in the literature are evaluated under different training and evaluation protocols (e.g., number of training epochs, data repetition strategies, and evaluation schedules). A naive unification of these protocols—for example, adopting the evaluation setting of baseline A for all methods—can lead to suboptimal or misleading performance for other baselines such as B and C, whose effectiveness is tied to their original selection regimes. Therefore, to ensure a fair and faithful comparison, we follow the evaluation protocol used in each original work when reporting results, rather than enforcing a single unified protocol across all methods.

To ensure statistical significance, all experimental results are reported as the mean ± standard deviation across 5 independent runs with different random seeds. All experiments were conducted on a server equipped with four NVIDIA H20 GPUs. The total computational cost for reproducing the complete set of experiments is approximately 800 GPU hours.
\subsection{Data Summarization Experiment Details}\label{appendix:ds_details}
We conduct CNN-based experiments on MNIST and ResNet-based experiments on CIFAR-10 and CIFAR-100.\\\\
\textbf{CNN on MNIST.}
The CNN architecture comprises two convolutional blocks, each consisting of convolution, dropout, max-pooling, and ReLU activation. The numbers of filters in the two convolutional layers are 32 and 64, respectively, both with a kernel size of (5 $\times$ 5). These are followed by two fully connected layers of sizes 128 and 10, with dropout applied. The dropout probability is set to 0.5. We adopt experimental settings consistent with prior work~\cite{borsos2020coresets, zhou2022probabilistic, tong2025coreset}.

\begin{itemize}
\item \textbf{Greedy Coreset and Uniform Sampling}~\cite{borsos2020coresets}:
The CNN is trained on the data summary using the Adam optimizer with a learning rate of $5e^{-4}$ for 4000 epochs.
\item \textbf{PBCS}~\cite{zhou2022probabilistic}:  
The projected stochastic gradient descent update for coreset scores is given by
\begin{equation}
    \boldsymbol{s}^{t+1} \leftarrow \mathcal{P}_{\mathcal{C}}\left(\boldsymbol{s}^{t}-\eta \mathcal{L}_{\mathcal{B}}\left(\boldsymbol{\theta}^{*}(\boldsymbol{m})\right) \nabla_{\boldsymbol{s}} \ln p\left(\boldsymbol{m} \mid \boldsymbol{s}^{t}\right)\right).
\end{equation}
The number of iterations for optimizing the coreset scores \(\boldsymbol{s}\) is \(T = 500\), where \(t \in \{1,2,\dots,T\}\), and the learning rate for score updates is \(\eta = 5e^{-2}\).  
The inner optimization \(\boldsymbol{\theta}^{*}(\boldsymbol{m})\) is performed using Adam with a learning rate of \(5e^{-3}\) for 300 epochs.  
The CNN is then trained on the resulting data summary using Adam with a learning rate of \(2e^{-3}\) for 4000 epochs.

\item \textbf{CSRel}~\cite{tong2025coreset}:  
The reducible loss for coreset selection is defined as
\begin{equation}
    G_{i}=\ell\left(x_{i}, y_{i} ; \boldsymbol{\theta}_{S}\right)-\ell\left(x_{i}, y_{i} ; \boldsymbol{\theta}_{\mathcal{D}}\right).
\end{equation}
The holdout model \(\boldsymbol{\theta}_{\mathcal{D}}\) is trained using SGD with a learning rate of \(1.5e^{-3}\) for 125 epochs.  
The inner model \(\boldsymbol{\theta}_{S}\) is trained using SGD with a learning rate of \(2e^{-2}\) for 100 epochs.  
Finally, the CNN is trained on the selected data summary using SGD with a learning rate of \(2e^{-2}\) for 3000 epochs.

\end{itemize}
For \texttt{ReSGD}, the holdout model \(\boldsymbol{\theta}_{\mathcal{D}}\) is trained using SGD with a learning rate of \(1.5e^{-3}\) for 200 epochs.  
The leave-one-out model $\bs{\theta}^{*}(\mathcal{S}_{-(i,1)})$ is trained using SGD with a learning rate of $5e^{-4}$ for 1000 epochs. Finally, the CNN is trained on the selected data summary using SGD with a learning rate of \(2e^{-2}\) for 3000 epochs. The number of iterations $T$ is 20 and $\eta$ is $1e^{-7}$.\\\\
\textbf{ResNet on Cifar-10 and Cifar-100.}
The experiments of ResNet-18 on Cifar-10 and Cifar-100 follows the setting in~\cite{borsos2020coresets, zhou2022probabilistic, tong2025coreset}. The details of the implementation of baselines are as follows:
\begin{itemize}
    \item \textbf{Greedy Coreset and Uniform Sampling}~\cite{borsos2020coresets}: The ResNet is trained by Adam optimizer with a learning rate $5e^{-5}$ in 6 epochs. In each epoch, the coreset is repeatedly sampled to 60000 and the total update iterations are 937 (6000/batch size 64).
    \item \textbf{PBCS}~\cite{zhou2022probabilistic}: 
    The number of iterations for optimizing the coreset scores \(\boldsymbol{s}\) is \(T = 100\), where \(t \in \{1,2,\dots,T\}\), and the learning rate for score updates is \(\eta = 3e^{-2}\).  The inner optimization \(\boldsymbol{\theta}^{*}(\boldsymbol{m})\) is performed using Adam with a learning rate of \(2e^{-3}\) for 80 epochs.
    The CNN is then trained on the resulting data summary using Adam with a learning rate of \(2e^{-3}\) for 4000 epochs.
    \item \textbf{CSRel}~\cite{tong2025coreset}: The holdout model \(\boldsymbol{\theta}_{\mathcal{D}}\) is trained using SGD with a learning rate of \(3e^{-3}\) for 100 epochs. The inner model \(\boldsymbol{\theta}_{S}\) is trained using SGD with a learning rate of \(1e^{-2}\) for 100 epochs. Finally, the ResNet is trained on the selected data summary using Adam with a learning rate of \(5e^{-4}\) for 1800 epochs. Initial coreset size is set to 0 and the selection step is 200.
\end{itemize}
For \texttt{ReSGD}, the holdout model \(\boldsymbol{\theta}_{\mathcal{D}}\) is trained using SGD with a learning rate of \(3e^{-3}\) for 100 epochs.  
The leave-one-out model $\bs{\theta}^{*}(\mathcal{S}_{-(i,1)})$ is trained using SGD with a learning rate of $5e^{-4}$ for 100 epochs. Finally, the CNN is trained on the selected data summary using SGD with a learning rate of \(5e^{-4}\) for 1800 epochs. The number of iterations $T$ is 20 and $\eta$ is $1e^{-8}$.

\subsection{Continual Learning Experiment Details}\label{appendix:cl_details}
\begin{figure}[!t]
    \centering
    \includegraphics[width=0.5\linewidth]{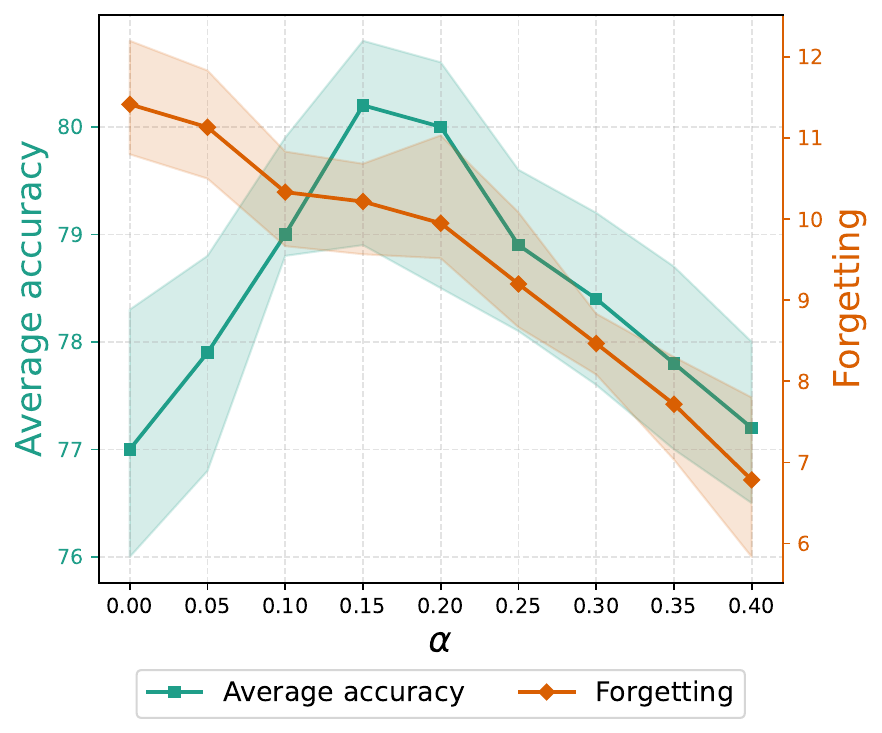}
    \caption{Impact of $\alpha$ on Average Accuracy and Forgetting. Comparison of performance metrics under different loss factors $\alpha$ on the Permuted MNIST dataset. Results show that $\alpha \in [0.15, 0.20]$ provides an optimal trade-off for mitigating forgetting while maintaining overall accuracy.}
    \label{fig:alpha}
\end{figure}

We evaluate our data summarization approach for continual learning across four standard benchmarks: Split MNIST~\cite{zenke2017continual}, Perm MNIST~\cite{goodfellow2013empirical}, Split CIFAR-10, and Split CIFAR-100. \\
\textbf{Datasets and Augmentations.} Following the evaluation protocol established by~\cite{borsos2020coresets, tong2025coreset}, the dataset splits are constructed as follows: Split MNIST and Split CIFAR-10 are both divided into 5 tasks with 2 classes per task. Split CIFAR-100 is divided into 10 tasks, each containing 10 classes. Perm MNIST consists of 10 tasks, where each task corresponds to a randomly permuted version of the original MNIST images. For the training data formulation, we randomly sample 1,000 images per task for Split MNIST, Split CIFAR-10, and Perm MNIST. For Split CIFAR-100, we utilize the full dataset for each task. Regarding data augmentation, we apply standard normalization to both Split MNIST and Perm MNIST. For the more complex CIFAR-10 and CIFAR-100 datasets, our augmentation pipeline includes random cropping, random horizontal flipping, and normalization.\\
\textbf{Network Architectures.} We employ distinct backbone architectures tailored to the complexity of each dataset:
\begin{itemize}
    \item Split MNIST: A standard CNN architecture (as detailed in section~\ref{appendix:ds_details}).
    \item Split CIFAR-10: A standard ResNet-18 architecture (as detailed in section~\ref{appendix:ds_details}).
    \item Split CIFAR-100: A modified ResNet-18 architecture, wherein a Batch Normalization layer is injected into each convolutional block.
    \item Permuted MNIST: A fully connected network consisting of two hidden layers with 100 units each, ReLU activation functions, and dropout regularization with a probability of $0.2$ applied to the hidden layers.\\
\end{itemize}
\textbf{Optimization and Hyperparameters.} Our training procedures follow the methodologies of ~\cite{borsos2020coresets, zhou2022probabilistic, tong2025coreset}. The optimization process involves a specific loss factor, denoted as $\alpha$ (refer to the setting in~\cite{tong2025coreset}). The hyperparameter configurations are specifically tuned for each dataset as follows:
\begin{itemize}
    \item For Split MNIST, Split CIFAR-10, and Perm MNIST, we utilize the Adam optimizer with a batch size of $256$ and a learning rate of $5 \times 10^{-4}$. Models for these three datasets are trained for $400$ epochs. The loss factor $\alpha$ is set to $100.0$ for Split MNIST, $20.0$ for Split CIFAR-10, and $0.15$ for Perm MNIST.
    \item For Split CIFAR-100, we employ the SGD optimizer with a smaller batch size of $32$ and a higher learning rate of $2 \times 10^{-2}$. The model is trained for $100$ epochs, and the loss factor $\alpha$ is set to $4.0$. This SGD-based setup is consistently applied across all compared methods for the CIFAR-100 benchmark.
\end{itemize}

\textbf{The choices of the loss factor}
Following \cite{tong2025coreset}, we adopt the forgetting metric to quantify the performance degradation across subsequent tasks for the experiments in Section 5.2. The forgetting for task $j$ after training on task $k$ is computed as:
\begin{equation}
f_j^k = \max_{l \in {1, \dots, k-1}} a_{l,j} - a_{k,j}, \quad \forall j < k,
\end{equation}
where $a_{k,j}$ denotes the accuracy on task $j$ upon the completion of task $k$. We evaluate both average accuracy and forgetting on the Permuted MNIST dataset. As illustrated in Figure~\ref{fig:alpha}, we analyze these metrics across varying loss factors $\alpha$. Empirically, the optimal balance is achieved at $\alpha=0.15$, though $\alpha=0.2$ also yields robust performance in practice.